\documentclass[journal,twocolumn]{IEEEtran}
\usepackage[colorlinks=true,linkcolor=blue,citecolor=blue,urlcolor=blue]{hyperref}
\usepackage{wrapfig}
\usepackage{booktabs}
\usepackage{graphicx}
\usepackage[compress]{cite}
\usepackage{algorithm}
\usepackage{amsthm}
\usepackage{algpseudocode}
\usepackage{color}
\usepackage{comment}
\usepackage{multirow}
\usepackage[shortlabels]{enumitem}
\usepackage{subcaption} 
\newtheorem{theorem}{Theorem}
\newtheorem{lemma}{Lemma}

\newtheorem{definition}{Definition}

\usepackage{amsmath,amsfonts,bm}

\def\eqref#1{equation~\ref{#1}}

\def\1{\bm{1}}

\def\vmu{{\bm{\mu}}}

\def\vk{{\bm{k}}}

\def\vp{{\bm{p}}}

\def\vr{{\bm{r}}}

\def\vx{{\bm{x}}}

\DeclareMathAlphabet{\mathsfit}{\encodingdefault}{\sfdefault}{m}{sl}
\SetMathAlphabet{\mathsfit}{bold}{\encodingdefault}{\sfdefault}{bx}{n}

\def\data{{\cal{D}}}

\def\Ndata{{{\rm{N}}_{\cal{D}}}}
\def\reference{{\cal{R}}}
\def\Nreference{{\rm{N}}_{\cal{R}}}
\def\Ndata{{\rm{N}}_{\cal{D}}}

\newcommand\Hnull{{\rm{H_{\boldsymbol0}}}}

\usepackage[font=small,labelfont=bf]{caption}

\def\beq{\begin{equation}\displaystyle}
\def\eeq{\end{equation}}
\def\bea{\begin{eqnarray}\displaystyle} 
\def\eea{\end{eqnarray}}

\def\({\left(}
\def\){\right)}
\def\[{\left[}
\def\]{\right]}

\newtheorem{remark}{Remark}
\newtheorem{proposition}{Proposition}
\begin{document}
\title{Sparse, self-organizing ensembles of local kernels detect rare statistical anomalies}

\author{Gaia Grosso$^{a,b,c}$,
Sai Sumedh R. Hindupur$^{c}$,
Thomas Fel$^{c,d}$,\\
Samuel Bright-Thonney$^{a,b}$,
Philip Harris$^{a,b}$,
and Demba Ba$^{c,d}$
\thanks{}
\thanks{$^a$NSF AI Institute for Artificial Intelligence and Fundamental Interactions, Cambridge, 02139, MA}
\thanks{$^b$Laboratory for Nuclear Science, Massachusetts Institute of Technology, Cambridge, 02139, MA}%
\thanks{$^c$School of Engineering and Applied Sciences, Harvard University, Allston, 02134, MA}
\thanks{$^d$Kempner Institute for the Study of Natural and Artificial Intelligence, Harvard University, Allston, 02134, MA}
}
\maketitle

\begin{abstract}
Modern artificial intelligence has revolutionized how we extract representations from scientific data, yet the statistical properties of these representations remain poorly controlled, causing misspecified anomaly detection methods to falter. The hardest anomalies to detect are the rare, weakly separable ones hiding within the nominal distribution---a regime that grows in importance as models mature and easily separable signals are exhausted. We identify structural desiderata for detection in this regime under minimal prior information: \textit{sparsity}, to enforce parsimony; \textit{locality}, to preserve geometric sensitivity; and \textit{competition}, to promote efficient allocation of model capacity. These principles define a class of self-organizing local kernels that adaptively partition the representation space around regions of statistical imbalance. As an instantiation, we introduce \textsc{SparKer}, a sparse ensemble of Gaussian kernels trained in a semi-supervised Neyman--Pearson framework to locally model the likelihood ratio between a sample that may contain anomalies and an anomaly-free reference. We provide theoretical insights into the mechanisms driving detection and self-organization, and demonstrate the approach on realistic high-dimensional problems in scientific discovery, open-world novelty detection, intrusion detection, and generative-model validation. Ensembles of only a handful of kernels identify statistically significant anomalies in representation spaces of thousands of dimensions while remaining sensitive across regimes, underscoring the interpretability, efficiency, and scalability of the approach.
\end{abstract}


\maketitle
\section{Introduction}
Finding data patterns that do not conform to nominal behavior, a practice that goes under the name of anomaly detection (AD)~\cite{10.1145/1541880.1541882}, plays a central role in the scientific method, enabling the discovery of novel phenomena and advancing our understanding of the natural world~\cite{lochner2021astronomaly,eze2023anomaly, belis2024machine,ZINGMAN2024103067,campolongo2025building,moro2025multimodal}. 
Based on theories and controlled experiments, scientists construct generative models of observed phenomena, and hope to reveal new laws of nature as unexpected deviations.  
Similarly, AD is at the core of data quality monitoring, validation and control tools, enabling safe interaction with the myriad of artificial systems that surround us~\cite{patcha2007overview,wu2024physics,wang2025survey}. 
AD also addresses the growing need in computer science for robust methods to test, interpret, and align generative models, a challenge made urgent by the advent of unprecedentedly large models in vision and language~\cite{johnston2025mechanistic,miyai2024generalized,10.1162/coli_a_00429}.

AD is a challenging problem for two key reasons: (1) the variety of ways anomalies manifest in data (e.g. in-distribution, out-of-distribution, point-like, collective, ...) and, more often than not, (2) the dearth of information about the time and location of their occurrence~\cite{10.1145/1541880.1541882}.
Several methods exist to tackle different AD settings~\cite{markou2003novelty}.
Broadly speaking, if the fraction of anomalous points in a dataset is large, key properties of the signal can be learned from the data by means of arbitrarily flexible models, whereas in anomaly-scarce regimes the lack of evidence must be replaced by inductive bias~\cite{yu-pnas}. Feature engineering, geometry, and smoothness are just a few examples of design choices  that implicitly shape models for AD.
The performance of AD tools is thus tightly intertwined with the correctness of the modeling assumptions they rely on~\cite{hindupur2025projecting, Grosso:2024wjt}.
\noindent
Concurrent with the growing need for robust AD tools, advances in representation learning have revolutionized data analysis across scientific and technological domains, enabling scalable statistical inference on massive datasets and complex high-dimensional problems~\cite{bronstein2021geometric,oord2018representation,wang2021desiderata,phlab}.
In spite of this, the statistical properties of neural embeddings remain poorly controlled, exposing downstream AD to failure modes when the embedded anomalies do not retain separability features. As depicted in panel (a) of Figure~\ref{fig:main}, most contemporary AD methods rely on assumptions of separability or dense supervision that break down when anomalies are rare, weakly expressed, or embedded  within the data manifold~\cite{liu2020simple, reiss2023no, fort2021exploring, li2025out}, leaving this corner of the anomaly space unexplored.  
Critically, this is the corner discovery converges toward: as reference models improve, the easily separable signals are absorbed, and what remains is rare and embedded in the nominal distribution. Sustained progress is contingent on reaching these signals.
\begin{figure*}[t]
    \centering
    \includegraphics[width=\linewidth]{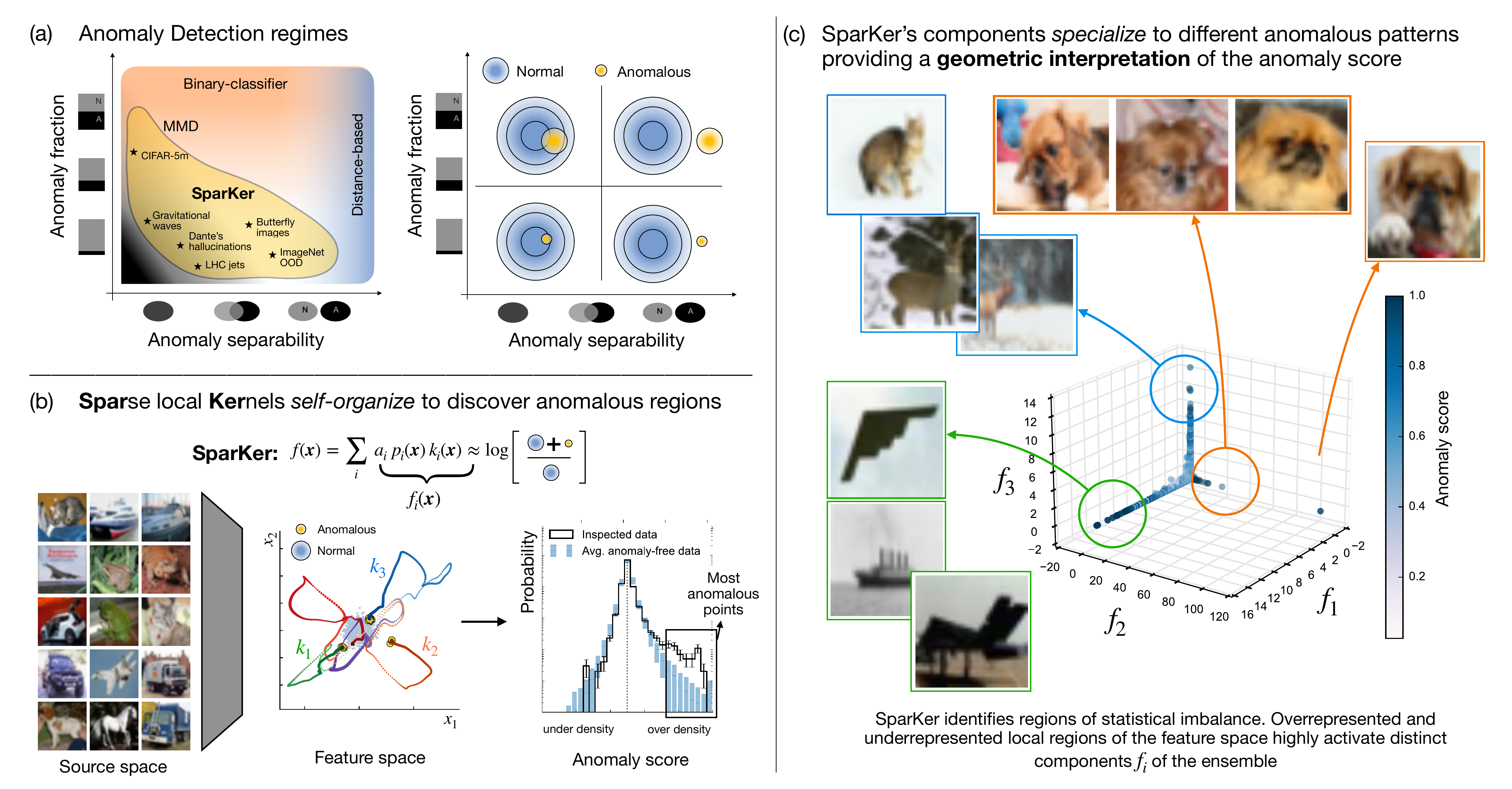}
    \caption{\textbf{Detecting rare in-distribution anomalies with \textsc{SparKer}.} (a) Cartoon diagram illustrating the regimes of anomaly detection; problems are characterized by anomaly separability ($x$-axis) and anomaly fraction ($y$-axis). Detection becomes harder near the origin, where anomalies are rare and overlap with normal data. The right panel pictures the different corners of the space qualitatively. The left panel highlights the region \textsc{SparKer} is designed for, i.e. low anomaly fraction and separability. The stars indicate the suite of applications from natural and computer science that we showcase in Sections~\ref{sec:applications} and~\ref{sec:interp}. 
    (b) \textsc{SparKer}'s anomaly detection pipeline: source data are embedded to extract domain-specific features; \textsc{SparKer} compares the distribution of the inspected data to anomaly-free data in feature space by means of a self-organizing ensemble of local kernels, and extracts an \textit{anomaly score};
    (c) Sparsity, locality and competition, which underlie the \textsc{SparKer} model, make it possible to decompose the anomaly score into distinct, interpretable components, $\{f_i\}$, pointing at different anomalous regions in the feature space, thus enabling a geometry-aware analysis of the anomalies. 
    }
    \label{fig:main}
\end{figure*}
We address this gap by introducing an AD method for \emph{rare} \emph{in-distribution} anomalies. Motivated by applications across natural and computer sciences, we propose three principled desiderata for AD algorithms—\textit{sparsity}, \textit{locality}, and \textit{competition}—as fundamental to discovering and characterizing novel patterns in this regime. These principles motivate the design of efficient ensembles of radial kernels that self-organize around anomalous regions in high-dimensional spaces while retaining interpretability. As a concrete realization, we introduce \textsc{SparKer}, a sparse ensemble of Gaussian kernels equipped with a competition mechanism. Panel~(b) of Figure~\ref{fig:main} shows how \textsc{SparKer} operates in feature space. \textsc{SparKer}'s self-organizing Gaussian kernels model the log-density ratio between the inspected data (a mixture of anomalous and normal occurrences) and an anomaly-free reference set, yielding an interpretable anomaly score. 
We demonstrate the effectiveness of this framework across diverse applications—from scientific discovery to generative model validation, novelty detection, and intrusion detection—spanning multiple data modalities, dimensionalities, and scientific domains. Remarkably, ensembles containing only a handful of kernels can identify statistically significant anomalous regions within representation spaces of thousands of dimensions. Panel~(c) of Figure~\ref{fig:main} illustrates \textsc{SparKer}’s ability to characterize anomalous regions: the sparse kernel components provide low-dimensional representations that isolate the most anomalous regions of the data (dark blue points), revealing distinct underlying anomaly types.

Taken together, our findings demonstrate that \textit{sparsity}, \textit{locality}, and \textit{competition} are key properties for building \textit{scalable} and \textit{interpretable} models for rare in-distribution detection, opening new frontiers in scientific discovery, monitoring, and validation of natural and artificial systems.\\

The rest of our treatment begins in Section~\ref{sec:design}, where we introduce \textsc{SparKer} and the design principles behind it. In Section~\ref{sec:theory}, we analyze the detection properties and training dynamics of the model, providing theoretical and numerical insights in support of its design principles.
In Section~\ref{sec:applications}, we show the performance of \textsc{SparKer} on AD problems spanning scientific discovery, generative AI validation, novelty and intrusion detection. In Section~\ref{sec:interp} we provide guidelines for interpreting the detected anomalies, and discuss domain-specific, actionable findings from the various applications considered in Section~\ref{sec:applications}. We conclude in Section~\ref{sec:conclusion} with future research directions.
\section{Sparse ensembles of local kernels (\textsc{SparKer})}\label{sec:design}
\subsection{Anomaly detection regimes.}
In AD, different problem settings require different assumptions on the anomaly~\cite{blanchard2010semi}. In the absence of a priori knowledge about the signal, two parameters strongly characterize AD problems: (1) the amount of information about the signal hidden in the inspected data (i.e. anomaly fraction), and (2) how difficult it is to distinguish between anomalous and normal data in representation space (i.e. anomaly separability)\footnote{These properties can only be assessed in the presence of ground truth knowledge about the anomaly. The anomaly fraction can be quantified as the ratio between the number of anomalous points and the total number of points in the dataset. The anomaly separability can be quantified as the classification accuracy of a supervised binary classifier between anomalous and normal data points.}. 
The former shapes the balance between expressivity and inductive bias in model design; the latter guides the choice of the statistical method to employ.
Panel (a) of Figure~\ref{fig:main} gives a qualitative view of the most appropriate AD methods given the problem setting: for large anomaly fractions, binary classifiers with standard methods of regularization are the most practical way to learn a density ratio and solve the AD problem in high dimensions~\cite{steinwart2005classification, lopez-paz2017revisiting,cheng2022classification}. On the other hand, if the signal is known to lie in the extreme tails of the distribution of anomaly-free data, then distance-based methods, neglecting distribution shape, are more suitable candidates~\cite{lee2018simple,chalapathy2018anomaly,pmlr-v80-ruff18a}.
We focus on the challenging scenario where the anomaly fraction is low and separability is not guaranteed (bottom left corner of the diagram), and show that a Neyman-Pearson approach to hypothesis testing~\cite{Grosso:2023scl} coupled with a proper set of inductive biases leads to sensitivity in these regimes. 

\subsection{Neyman-Pearson two-sample test.}
The Neyman-Pearson (NP) test addresses the problem of statistical AD by performing a signal-agnostic two-sample test based on the ratio of likelihoods~\cite{blanchard2010semi,Grosso:2023scl, DAgnolo:2019vbw,Letizia:2022xbe}. 
Given a dataset $\data$ comprising $\Ndata$ statistically independent unlabeled observations, and a reference sample $\reference$ of size $\Nreference$ distributed according to the nominal data condition $\Hnull$, we quantify the degree of compatibility between the data and the assumed reference model with a likelihood-ratio test
\begin{equation}\label{eq:NP-test}
    t_{\rm NP}(\data) = 2 \max\limits_\theta \left[
    \log\frac{{\cal L}(\data|{{\rm H}_\theta})}{{\cal L}(\data|\Hnull)}\right] \,,
\end{equation}
where the data distribution under the alternative hypothesis is parametrized by a family of models ${\cal F}=\{ f_\theta ,\, \theta\in \Theta \}$, with unknown parameters $\theta$ to estimate from the data
\begin{equation}\label{eq:n_theta}
    n(\vx|{\rm H}_\theta)=n(\vx|{\rm H_0})e^{f_\theta(\vx)}.
\end{equation}
For extended likelihoods~\cite{barlow1990extended}, the maximum likelihood problem becomes equivalent to minimizing the custom loss function~\cite{DAgnolo:2018cun}
\begin{equation}\label{eq:L_nplm}
L_{\rm NP}[f_{\theta}]= \sum\limits_{\vx\in \reference}w_\reference(e^{f_\theta(\vx)}-1) -\sum\limits_{\vx\in\data} f_\theta(\vx)\,,
\end{equation}
with $w_\reference$ balancing for the different sample sizes.
The test statistic is then readily extracted from the loss function as
\begin{equation}\label{eq:t_nplm}
t_{\rm NP}[f_{\theta}] = -2\cdot L_{\rm NP}[f_{\theta}].
\end{equation}
Recent studies comparing the NP test with alternative two-sample tests have found the NP test more sensitive and universally more robust to various sources of anomalies in the data, especially in challenging detection regimes~\cite{Grosso:2023scl, Grosso:2023ltd,Grossi:2024axb}.
\subsubsection{Related work} 
Statistical anomaly detection addresses the problem of identifying significant distributional deviations from a reference probabilistic model describing the data-generating process under nominal conditions. When the reference model must be inferred from data rather than specified a priori, the problem is naturally formulated as a \textit{two-sample test}.
An extensive body of literature on traditional statistical methods exists to address univariate setups (the best-known tests are Kolmogorov-Smirnov, Anderson-Darling, Cram\'er-von-Mises, $\chi^2$), some of which can be extended to a few dimensions under specific assumptions~\cite{friedman1979multivariate}.
Kernel-based two-sample tests relying on the maximum mean discrepancy metric~\cite{JMLR:v13:gretton12a} provide a powerful solution to high-dimensional problems, despite the computational cost on large samples~\cite{chatalic2026scalable,gerber2023kernel}.
More recent strategies leverage deep neural networks to learn discriminative structures: a binary classifier is trained between samples, and its performance metrics or score distributions are used as test statistics~\cite{lopez-paz2017revisiting,Chakravarti:2021svb,Friedman:2003id}. While effective with abundant anomalies, these approaches lose performance in regimes of low anomaly fraction.

\subsection{Desiderata of detection methods in challenging regimes.}
The power of the NP test depends on the design and optimization of the family of functions $\cal F$ parametrizing the density ratio.
For rare in-distribution anomalies, we argue that sparsity, locality and competition are a powerful set of inductive biases for sensitive and interpretable detection. 

\begin{figure}[t]
    \centering 
    \includegraphics[width=0.95\linewidth]{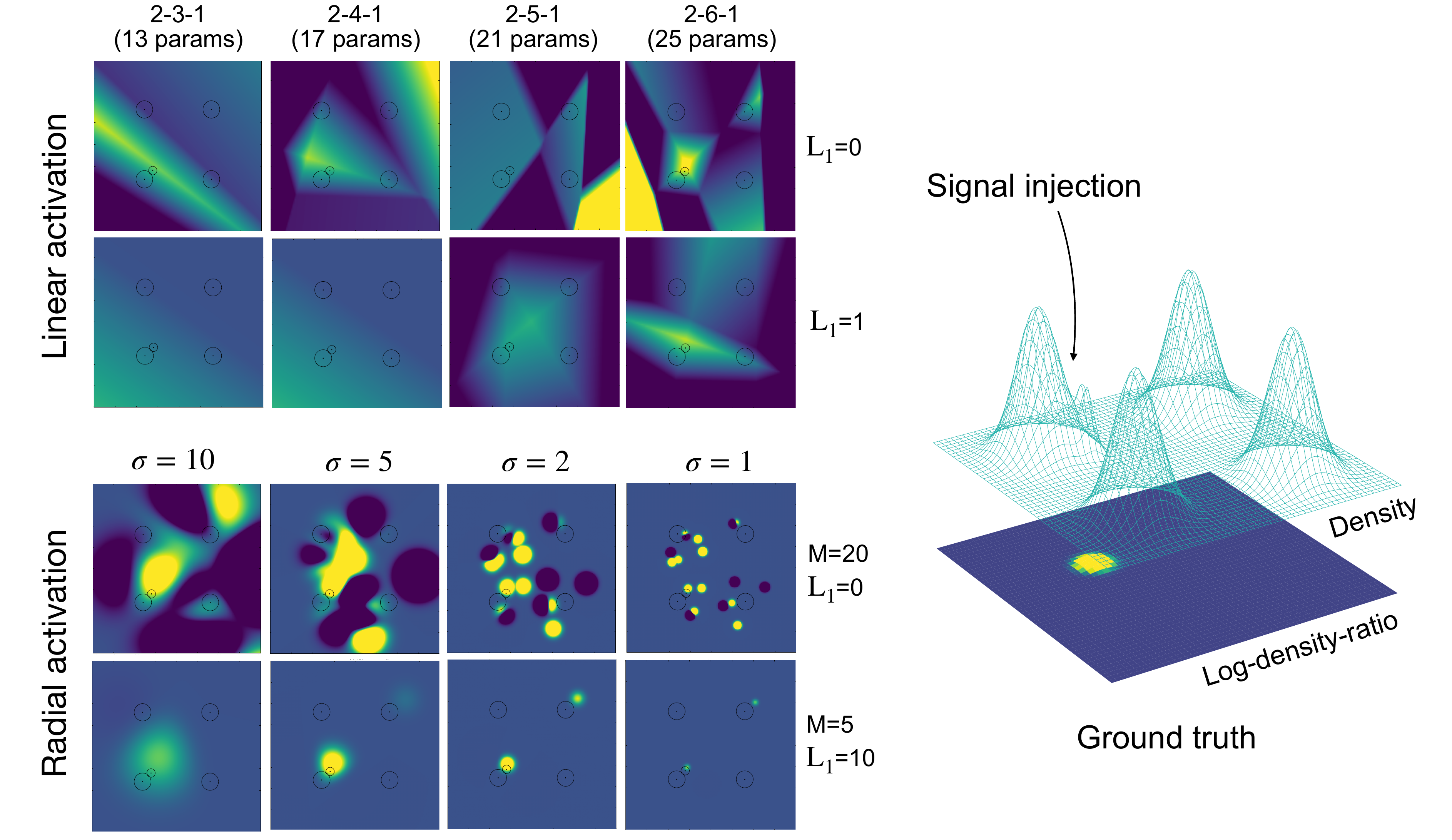}
    \caption{\textbf{Local layers enable geometric interpretation of activation patterns.} Examples of 2D log-density ratio approximated by a 1-layer model with ReLU activation (top left grid) and a mixture of M Gaussian kernels (bottom left grid) subject to different sparsity constraints. The local activation characterizing Gaussian kernels enables efficient selection of convex regions, whereas linear activation requires the interplay of multiple parameters, undermining interpretability. We report on the right side the ground truth of the log-density ratio function targeted by the model, as well as the density of the sample with injected anomalies.}\label{fig:2d_act}
\end{figure}
\subsubsection{The case for sparsity}\label{subsubsec:sparsity} 
Sparsity constraints in machine learning and artificial intelligence play a crucial role in enhancing model interpretability, efficiency, and generalization~\cite{JMLR:v22:21-0366}. By enforcing that only a small subset of features, neurons, or basis functions are active at any given time, sparsity mimics the energy-efficient way biological neural networks operate in the context of sensory processing, where neurons in the brain exhibit localized and selective responses to stimuli, reducing redundancy and enhancing discriminative power~\cite{omohundro1987efficient}. 
Imposing sparsity in the model activation enforces a competing mechanism that acts among the training data points, craving for the model's attention.
Thus, sparsity serves a statistical purpose. In a hypothesis test, model capacity is not free: too flexible a family of alternatives inflates the distribution of the test statistic under the null, so that a genuine signal must overcome a heavier-tailed background of fluctuations to reach significance~\cite{DAgnolo:2018cun,Grosso:2023scl}. 
By allowing only a small number of active components, sparsity concentrates the power of the test on \textit{rare} signatures.
 
\subsubsection{The case for locality}
Under the sparsity budget introduced above, the choice of unit determines which deviations remain representable: a few linearly-activated units, such as ReLU~\cite{agarap2018deep}, capture global, slowly-varying distortions parsimoniously but can synthesize a bounded response only through the coordinated action of many units, entangling the roles of individual parameters~\cite{colin2024local}; conversely, a few radial units with tunable scale cover both smooth and narrow deviations (Figure~\ref{fig:2d_act}), provided they are rare. Beyond performance, locality is the foundation of geometric interpretability: each unit ties its response to an explicit region, so the model decomposes into components with explicit locations and scales, granting mechanistic access to \emph{where} and \emph{how} it reacts---a capability we develop in Section~\ref{sec:interp}.

\subsubsection{The case for competition and self-organization}
The lack of information about the location of the anomaly requires learning dynamics that can efficiently explore the feature space, and ultimately localize the anomalous region.
To this end, a relevant property of a good AD model is the ability to self-organize its degrees of freedom. This can be achieved by instilling \textit{competition} between kernels during the learning stage.
Competition in biological systems induces \textit{specialization} among different agents that results in a more efficient representation of heterogeneous data~\cite{rozell2008sparse}.
Prior studies on likelihood inference with Radial Basis Function Networks (RBFNs)~\cite{NIPS1989_d1c38a09} demonstrated that introducing competition among units, most effectively through ${\rm SoftMax}$-induced soft competition, enhances model performance. The gain goes beyond performance enhancement: specialization means decomposition, which eases the interpretability job downstream.

\subsection{Architecture design.}\label{subsec:architecture}
 To operationalize these desiderata in the NP test we introduce \textsc{SparKer}, a self-organizing ensemble of Gaussian kernels.
We consider a $d$-dimensional Gaussian kernel acting on Euclidean space:

\begin{equation}\label{eq:kernel}
    k^{\sigma}_{\vmu}(\vx) = \exp{\left[-\frac{1}{2}(\vx-\vmu)^{\rm T}\Sigma^{-1} (\vx-\vmu)\right]}
\end{equation}
with location $\vmu=(\mu^1, \dots, \mu^d)\in\mathbb{R}^d $ and uniform diagonal covariance matrix
$\Sigma = \sigma^2 \mathbb{I}^{d\times d}\in \mathbb{R}^{d\times d}$.

Inspired by Radial Basis Function Networks (RBFNs)~\cite{Broomhead1988MultivariableFI,sandberg,poggio,moody1989fast}, we build a sparse ensemble of $M$ Gaussian kernels, with an additional competitive activation promoting the self-organization of the kernels. Collecting the individual kernels' responses into the vector $\boldsymbol{k}_{\vmu}^{\sigma}(\vx)=(k^{\sigma}_{\vmu_1}(\vx), \dots, k^{\sigma}_{\vmu_M}(\vx))\in \mathbb{R}^M$, 
we introduce the normalized activation $\vp: \mathbb{R}^M \rightarrow \mathbb{R}^M$, 
\begin{equation}\label{eq:activation}
    p^i(\vx)= \frac{k^\sigma_{\vmu_i}(\vx)}{\sum_j k^\sigma_{\vmu_j}(\vx)}.
\end{equation}
Evaluated on the kernel vector, $\vp\big(\boldsymbol{k}^{\sigma}_{\vmu}(\vx)\big)$ is exactly a ${\rm SoftMax}$ over the negative squared distances $-\lVert\vx-\vmu_i\rVert^2$ at temperature $2\sigma^2$. 
We define \textsc{SparKer} as
\begin{equation}\label{eq:model}
f^{\sigma}_{{M}}(\vx) 
= \boldsymbol{a}^{T} [\vp(\vx) \odot \vk_{\vmu}^{\sigma} (\vx)]
=\sum_{i=1}^{M} a_{i} \cdot\,\frac{(k^{\sigma_i}_{\vmu_i})^2(\vx)}{\sum_j k^{\sigma_j}_{\vmu_j}(\vx)}
\end{equation}
where $\boldsymbol{a}=(a_{1},\dots,a_M)\in \mathbb{R}^M$ is the vector of the kernels' amplitudes and $\odot$ denotes the Hadamard product between the vectors.
Equivalently, \textsc{SparKer} is a linear combination of $M$ local basis functions each obtained from a squared Gaussian kernel normalized by the ensemble's total response $\sum_j k^{\sigma}_{\vmu_j}(\vx)$.

The vectors of kernels' locations $\boldsymbol{\mu}$ and amplitudes $\boldsymbol{a}$ are the parameters to estimate, while the kernels' width $\sigma$ is treated as a hyper-parameter and is annealed during training, allowing us to scan over multiple assumptions of locality. Sparsity is hard-coded in the model by keeping the number of kernels $M$ much smaller than the training sample size, and further encouraged by the ${\rm SoftMax}$ activation function. Details on model initialization and hyperparameter selection are available in Appendix~\ref{app:hyperp}.

\begin{algorithm*}[h]
\caption{\textsc{SparKer} algorithm}\label{alg:algo}
\footnotesize
\begin{algorithmic}[1]

\State \textbf{Calibration Experiments:}
\State Build a set of anomaly-free bootstrapped datasets: $$S_0 = \{\{\vx_i\}_{i=1}^{\Ndata} \mid \vx \sim p(x|{\rm H_0})\}$$
\ForAll{${\cal D}_0 \in S_0$}
    \State Train $f_M^\sigma$ with linear $\sigma$-annealing on $\reference$ and $\data_0$ to minimize Eq.~\ref{eq:L_nplm}
    \State Store $m$ intermediate model configurations: $(f^{\sigma_1}_M, \dots,f^{\sigma_m}_M) $
    \EndFor
\ForAll{${\cal D}_0 \in S_0$}    
    \ForAll{$\sigma \in [\sigma_1, \dots, \sigma_m]$}
            \State Extract test $t_{\rm NP}[f^\sigma_M](\data_0)$ as in Eq.~\ref{eq:t_nplm}
            \State Compute $p$-value: 
            $ p_{\sigma}(\data_0) = 1 - {\rm EDF}[t_{\rm NP}[f^\sigma_M](\data_0)]$
        \EndFor
    \State Compute the combined test: 
    $$T(\data_0) = -\frac{1}{2}\min\limits_{\sigma}[\log p_{\sigma}(\data_0)]-\frac{1}{2m}\sum\limits_\sigma \log p_{\sigma}(\data_0)$$
\EndFor
\State Collect all calibration tests: $T_0=\{T(\data_0) \mid \data_0 \in S_0\}$

\Statex

\State \textbf{Anomaly Detection:}
\State Consider the dataset of interest $\data$
\State Train $f_M^\sigma$ with linear $\sigma$-annealing on $\reference$ and $\data$ to minimize Eq.~\ref{eq:L_nplm}
\State Store $m$ intermediate model configurations: $(f^{\sigma_1}_M, \dots,f^{\sigma_m}_M) $
\ForAll{$\sigma \in [\sigma_1, \dots, \sigma_m]$}
        \State Extract test $t_{\rm NP}[f^\sigma_M](\data)$ as in Eq.~\ref{eq:t_nplm}
        \State Compute $p$-value: 
        $p_{\sigma}(\data) = 1 - {\rm EDF}[t_{\rm NP}[f^\sigma_M](\data)]$
    \EndFor
\State Compute the combined test: 
    $$T(\data) = -\frac{1}{2}\min\limits_{\sigma}[\log p_{\sigma}(\data)]-\frac{1}{2m}\sum\limits_\sigma \log p_{\sigma}(\data)$$

\State Compute global $p$-value:
$P = \frac{\sum_{t \in T_0} \mathbb{I}[t > T(\data)]}{|T_0|}$
\end{algorithmic}
\end{algorithm*}
\subsection{A semi-supervised analogue to Self-Organizing Maps.}
Competition and locality, the principles behind \textsc{SparKer}, are also at the core of Self-Organizing Maps (SOMs)~\cite{kohonen1982self,kohonen}, the canonical unsupervised model in which prototypes compete for inputs and self-organize into a representation of a data distribution. \textsc{SparKer} reframes competition and locality toward a different goal: whereas a SOM performs vector quantization of a \emph{single} distribution, allocating its prototypes to the high-density bulk, \textsc{SparKer} uses a two-sample likelihood-ratio objective to drive its units toward regions of statistical imbalance between the inspected data and the reference. 
The two methods also differ in how their units interact: a SOM couples nearby prototypes \emph{cooperatively}, so that neighbors move together to preserve the topology of the data manifold, whereas \textsc{SparKer}'s ${\rm SoftMax}$ activation (Eq.~\ref{eq:activation}) makes kernels compete \emph{repulsively}, each specializing to a distinct region rather than tiling the manifold. Together, these choices redirect the shared self-organizing machinery from \emph{representing} a distribution to \emph{testing} for local departures from it, and, unlike a SOM, yield a calibrated test statistic and $p$-value (see next paragraph). We compare \textsc{SparKer} to SOM-based detectors in Section~\ref{sec:applications}.
\subsection{Testing pipeline.}\label{subsec:algo-steps}
\textsc{SparKer} follows a frequentist approach to hypothesis testing: it returns a $p$-value by comparing the observed value of the test statistic to its distribution under the null (i.e., anomaly-free) hypothesis $H_0$. As with any frequentist procedure, the distribution of the test statistic under $H_0$ must be known or approximated before a $p$-value can be assigned. Since it is not available in closed form, we estimate it empirically in a dedicated calibration stage exploiting the reference sample $\mathcal{R}$, drawn from $H_0$. In this initial ``calibration" phase, the algorithm is run multiple times with bootstrapped data that follow the anomaly-free distribution. For each trial, we collect the outcome of the test and use the resulting empirical distribution to calibrate the $p$-value, e.g. to assess the rate of false positives.
This phase is blind to the nature of the anomaly and is kept as generic as possible. The hyper-parameters of the model can be tuned at this stage to improve the shape of the test statistic distribution (details on this aspect of model design can be found in Appendix~\ref{app:hyperp}). The second phase is the actual anomaly detection, in which we run \textsc{SparKer} on the data of interest and we compute an observed $p$-value exploiting the calibration distribution obtained in the previous stage.
The main steps of the algorithm are summarized in Algorithm~\ref{alg:algo}.
\section{Theoretical analysis of \textsc{SparKer}}\label{sec:theory}
In this section, we analyze the properties and training dynamics of \textsc{SparKer}, providing theoretical insights and illustrative experiments that clarify its working mechanisms (additional details are reported in Appendix~\ref{app:th}). 
The analysis moves from the general to the specific: first we address how an \textit{individual} data point affects a kernel, and then which data points \textit{significantly} influence its dynamics. Both steps follow from the gradient descent on the NP loss, so we write the update rule for a generic trainable parameter $\theta$:
\begin{equation}
\label{eq:parameter-update}
    \theta_{t+1} -\theta_{t} =
    \Delta\theta_t
    \propto - \left.\frac{\partial_{\theta} L_{\rm NP}[f_{\theta}]}{\partial\theta}\right|_{\theta = \theta_t} ,
\end{equation}
where $L_{\rm NP}$ is the NP objective, whose gradient with respect to a generic parameter $\theta$ reads as
\begin{equation}\label{eq:nablaL}
      \partial_{\theta} L_{\rm NP}[f_{\theta}] = w_{\reference} \sum\limits_{\vx\in \reference} e^{f_{\theta}(x)} \partial_{\theta} f_{\theta}(\vx) - \sum\limits_{\vx\in \data}  \partial_{\theta} f_{\theta}(\vx).
\end{equation}
\subsection{How do data points affect kernel dynamics?}  We first focus on how individual data points drive kernel dynamics.
For the amplitude of the $i^{th}$ kernel, the contribution of each data point $\vx$ to the gradient of the model takes the form of a scalar field locally active around the kernel's mean and vanishing far from it:
\begin{equation}\label{eq:grad_a}
    \partial_{a_i} f(\vx) = p^i(\vx)\cdot k^\sigma_{\vmu_i}(\vx).
\end{equation}

On the other hand, the gradient of the location of the kernel $\vmu_i$ has the form of a radial vector field
\begin{equation}\label{eq:grad_mu}
    \partial_{\vmu_i} f(\vx) = 
    A_i(\vx) \cdot \vr_i,
\end{equation}
where $\vr_i = \vx -\vmu_i$ is the displacement vector between the kernel location and the data point in Euclidean space, and $A_i(\vx)=\frac{1}{\sigma^2}\cdot \left[2\,a_i k_{\vmu_i}^\sigma(\vx) - f_M^\sigma(\vx)\right] \cdot p^i(\vx)$ is the magnitude. 

The radial nature of the gradient in Eq.~\ref{eq:grad_mu} lets us identify ``push-pull'' dynamics between pairs of kernels and training points that, in turn, induce the kernels to self-organize. We formalize this behavior in the following lemma:
\begin{lemma}[Push-pull dynamics of kernels' locations] The parameter update of the $i^{th}$ kernel location $\Delta \vmu_i$ (Eq. \ref{eq:parameter-update}) is the result of radial forces arising from the training points. Let $y$ be the class label associated with each data point ($y=0$ for $\vx\in \reference$ and  $y=1$ for $\vx \in \data$), and let $m_i(\vx)=2a_i k^{\sigma_i}_{\vmu_i}(\vx) - f^\sigma_M(\vx)$ be the $i^{th}$ kernel ``mass charge'' in $\vx$. Each update $\Delta \vmu_i$ is:
\begin{align}
    \text{``radially attractive''} \quad \text{if}\quad m_i(\vx)\cdot(2y-1)>0 \notag \\
    \text{``radially repulsive''} \quad \text{if}\quad m_i(\vx)\cdot(2y-1)<0 \notag
\end{align}  
\end{lemma}
\noindent \textit{Interpretation}: The direction and strength of the force on $\mu_i$ are set by the sign of $m_i(\vx)\cdot(2y-1)$, which in turn depends on the class of the point ($\reference$ or $\data$), the sign of the amplitude $a_i$, the kernel's proximity to the point (through $k_i(\vx)$), and the competition between kernels (which enters $f(\vx)$ and hence $m_i(\vx)$).
As a result, a kernel with positive amplitude that \textit{wins} the local competition is attracted toward nearby data points $\vx\in\data$ and repelled from reference points $\vx\in\reference$. Flipping either the dataset or the sign of $a_i$ reverses the force.

\subsection{Which data points significantly affect kernel dynamics?} Having examined how individual points act on the kernels, we now show that they do not act equally: two effects single out which points drive the dynamics.
\subsubsection{An energy-based interpretation of \textsc{SparKer}} \label{subsec:loss-gradient}
The first is the asymmetry of the action of the two training sets $\data$ and $\reference$ on the gradient of the loss.
To understand how this asymmetry shapes the dynamics, we study the continuum limit of Eq.~\ref{eq:nablaL}\footnote{Assuming the integral is finite, the approximation in Eq.~\ref{eq:generative} follows from the strong law of large numbers}:
\begin{align}
      \partial_{\theta} L_{\rm NP}[f_{\theta}] 
      \rightarrow
      &\int d\vx\, [ n(\vx|\Hnull) e^{f_{\theta}(\vx)} - n(\vx|\data)]\partial_{\theta} f(\vx)  \nonumber\\
      =& \int d\vx\,  [n(\vx|{\rm H}_{\theta}) - n(\vx|\data)] \partial_{\theta} f_{\theta}(\vx)\nonumber\\
      \approx& 
      \sum\limits_{\vx\in \hat\data} \partial_{\theta} f_{\theta}(\vx) 
      -\sum\limits_{\vx\in \data} \partial_{\theta} f_{\theta}(\vx).\label{eq:generative}
\end{align}
Due to the parametrization in Eq.~\ref{eq:n_theta}, summing the gradients over $\reference$ weighted by $e^{f_{\theta}(\vx)}$ is equivalent, in the limit of large statistics, to summing them over a dataset $\hat \data$ generated by the learned model $\rm H_{\theta}$.
The dynamics reaches equilibrium when the force from the generated points $\hat \data$ balances that from the real data $\data$: the optimization goal is a function $f_{\theta}$ whose generated distribution matches that of $\data$.
This binary regression dynamic thus implicitly hides an \textit{energy-based model}~\cite{lecun2006tutorial,grathwohl2019your} describing a system of interacting particles, in which the kernels' locations are hidden units and the datasets $\data$ and $\reference$ are the positive and negative visible units they adapt to.
\begin{proposition}[Balanced regions have zero target]\label{prop:bulk-zero}
In the continuum limit of Eq.~\ref{eq:generative}, the functional derivative of the NP loss vanishes at
$$
f^{*}(\vx)=\log\!\big[n(\vx|\data)/n(\vx|H_0)\big],
$$
which equals $0$ wherever $\data$ and $\reference$ are locally balanced, $n(\vx|\data)=n(\vx|H_0)$.
Gradient descent therefore drives $|f_\theta|\to 0$ on statistically balanced regions.
\end{proposition}
\noindent
The consequence is that kernels do not spend capacity on balanced regions, but instead seek out regions of imbalance between $\data$ and $\reference$.
Locality promotes this behavior: a region with no nearby kernel already outputs $f \approx 0$, which is the target optimum of balanced regions in Proposition~\ref{prop:bulk-zero}, so the null consumes no capacity and the small kernel budget is free to concentrate where $\data$ and $\reference$ disagree.
\begin{figure}[t]
    \centering
    \includegraphics[width=0.95\linewidth]{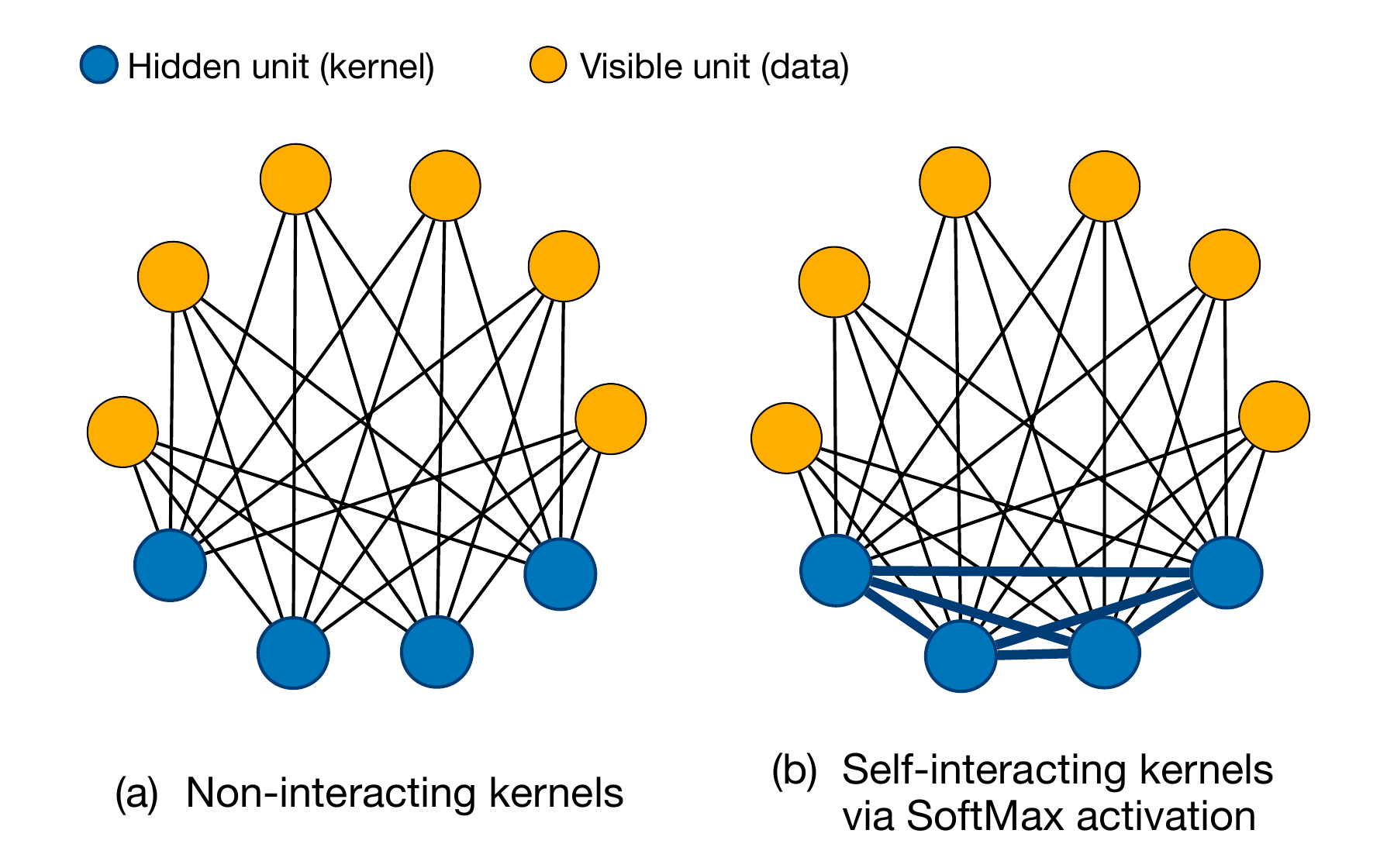}
    \caption{\textbf{Interpreting \textsc{SparKer} as a system of interacting particles.} Training data (visible units, orange) form the environment; kernels (hidden units) are latent variables capturing higher-order correlations. Without \textsc{SoftMax} (left), hidden units interact with the environment but not with each other; \textsc{SoftMax} (right) adds an interaction term between kernels, promoting joint optimization and specialization.}
    \label{fig:system}
\end{figure}
\subsubsection{A scale-dependent, local dynamic} \label{subsec:sphere-of-influence} The second effect is that the gradients in Eq.~\ref{eq:grad_mu} and~\ref{eq:grad_a} are \textit{local} and \textit{scale dependent}, and only data points that are in the neighborhood of a kernel have a leading role in them. 
We formalize this phenomenon by defining an active region for a kernel's dynamics that has radial geometry:
\begin{definition}[Kernel radial $\alpha$-level]
Let $k:\mathbb{R}_{+}\to\mathbb{R}_{+}$ be a continuous, radial kernel profile that is nonincreasing in $r\ge0$ (i.e. $k(r_1)\ge k(r_2)$ for $r_1\le r_2$).  
For a given level $\alpha\in(0,k(0))$ define the \emph{radial $\alpha$-level} $r_\alpha$ as the radius such that $k(r) < \alpha \quad \forall r>r_\alpha$, 
or equivalently $r_\alpha \;=\; \sup\{ r\ge0 \;:\; k(r)\ge\alpha\}$.
\end{definition}
\noindent The radial $\alpha$-level delineates each kernel’s \textit{sphere of influence}, the region where data points meaningfully influence the location $\vmu$ and amplitude $a$ dynamics:
\begin{lemma}[Kernel's Sphere of Influence] 
For each kernel $i$ that makes up the \textsc{SparKer} model, there exists a value $\alpha\in (0,1)$ such that the associated ball $S_{i}^\sigma (\alpha) = \{\vx \in \mathbb{R}^d: \|\vx-\vmu_i\|\leq r^\sigma_i(\alpha)\}$ contains all the training points that contribute at least $\epsilon_\alpha$ to the kernel's dynamics, i.e., points outside this ball have no more than $\epsilon_\alpha$ effect on the kernel's dynamics.\\ We name this ball the kernel's ``Sphere of Influence'':
\begin{equation}\notag
    \exists\, \epsilon_\alpha>0 \quad s.t. \quad |\partial_{\theta_i}f(\vx)| < \epsilon_\alpha, \quad \forall \vx \notin S_i^{\sigma}(\alpha) 
\end{equation}
with $\theta_i \in \{\vmu_i,\, a_i\}$.
If $\alpha\ll1$ then $\epsilon_\alpha\ll1$ as well.
Namely, the dynamics is local on a scale that is determined by the kernels' width, $\sigma$. 
\end{lemma}
\noindent
\textit{Interpretation}: Each kernel's dynamics is only affected by data points $\vx$ that are close to the kernel mean $\vmu$.\\

\noindent
Combining this observation with those in Section~\ref{subsec:loss-gradient}, we reach the central point: each kernel is driven only by the \emph{imbalanced} regions inside its sphere of influence, the balanced ones neither attract nor retain it.
\subsubsection{Explore and exploit via scale annealing}\label{subsec:anneal}
An AD algorithm exhibits strong detection power when it assigns high scores to anomalous regions while maintaining low scores across the rest of the data distribution. This requires that at least one kernel of \textsc{SparKer} identifies the anomalous region, and that its resolution (i.e., the kernel width $\sigma$) is sufficiently fine to distinguish it from most of the background.
Hence, discovering anomalous regions requires the kernels to simultaneously (1) explore the data space, and (2) resolve the anomaly.
However, due to the local nature of the Sphere of Influence, exploration is only possible with broad kernels, while resolution requires narrow ones. 
We address this challenge using deterministic annealing over $\sigma$~\cite{moody1988fast}, analogous to lowering the temperature of the system of particles. At high temperatures (large $\sigma$), kernels behave like diffusive particles that explore the space broadly, avoiding poor local minima. As the temperature decreases (smaller $\sigma$), their interactions become more localized, allowing the ensemble to condense and specialize around regions of anomalous density.
\begin{figure}[t]
    \centering
    \includegraphics[width=0.8\linewidth]{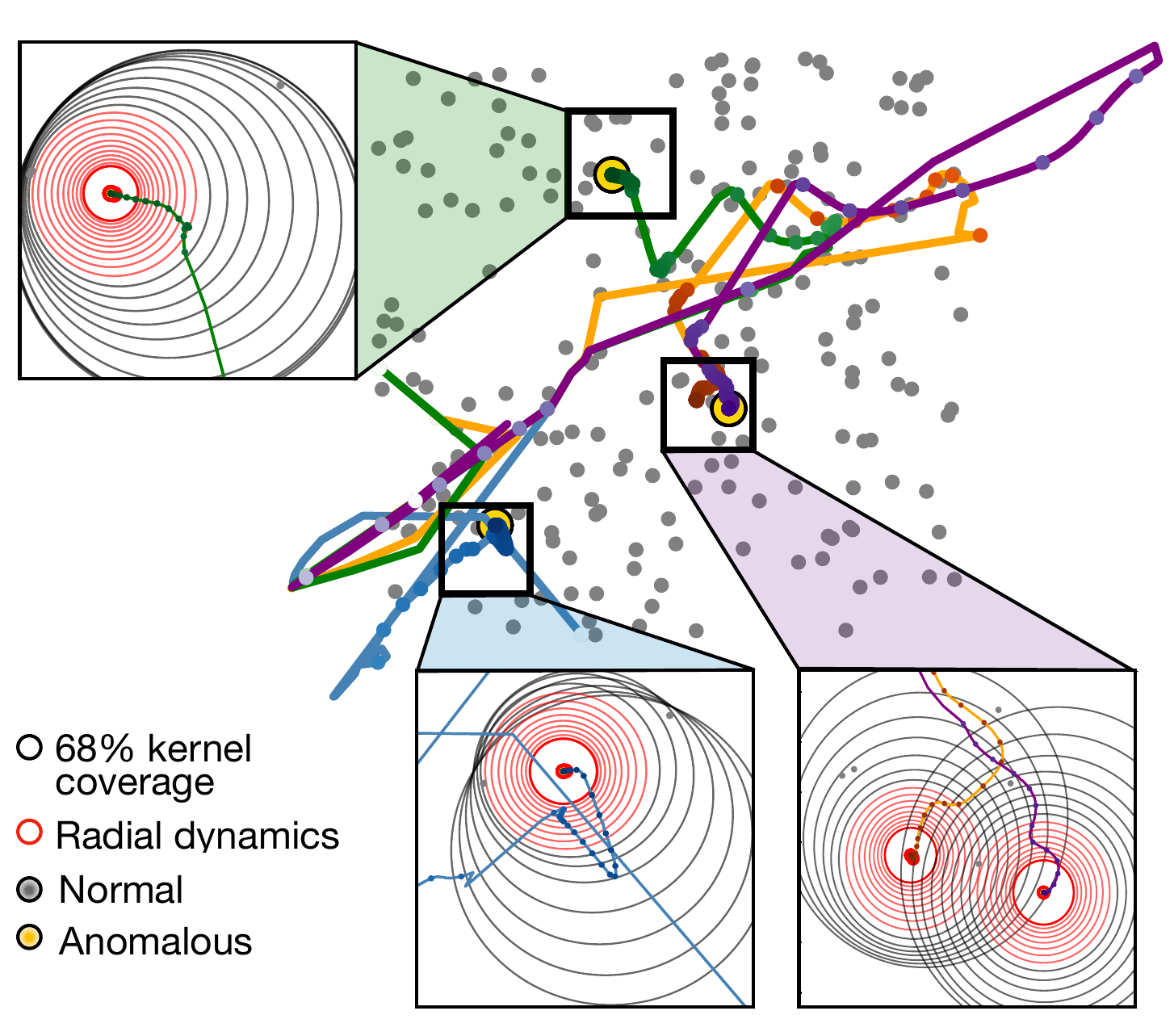}
    \caption{\textbf{Attraction and repulsion govern kernels' learning dynamics.} Kernels' trajectories (colored lines) over training in a 2D setting, where normal data are uniformly distributed in a square (gray points) and three sources of anomalies are injected (yellow). Running a 4-kernel \textsc{SparKer} model, we identify two stages in the self-organization dynamics: kernels are first broad and pushed to explore the space, then progressively narrow and radially attracted toward the anomalous regions (zoomed panels).}
    \label{fig:push-pull}
\end{figure}
\begin{lemma}[Scale annealing leads to localization] Monotonically annealing the kernel scale $\sigma$ shrinks the size of the Sphere of Influence: 
$$\sigma_t < \sigma_{t-1} \implies \text{Vol}\; S^{\sigma_t}_i(\alpha) < \text{Vol}\; S^{\sigma_{t-1}}_i (\alpha),$$
resulting in a decrease in the number of data points that influence the dynamics of each kernel $i$.
\end{lemma}
\noindent Gradually annealing $\sigma$ leads to regimes where, eventually, only one data point remains in the sphere of influence. In this setting, \textsc{SparKer} provably converges to that point within the Sphere of Influence:
\begin{theorem}[Local convergence]
    If $S_i^{\sigma_t}$ contains only one data point $\vx^*$ and kernel $i$ dominates the local competition, then $\vmu_i$ converges radially toward $\vx^*$, getting $\epsilon$-close to $\vx^*$ in $O(\log 1/\epsilon)$ further steps.
\end{theorem}
\noindent \textit{Interpretation}: Annealing turns the static setting of Sections~\ref{subsec:loss-gradient} and~\ref{subsec:sphere-of-influence} into a dynamic, coarse-to-fine one: the Sphere of Influence starts large and shrinks, so only an anomaly persistent across scales is resolved at progressively finer scale, and the kernel converges onto it.\\

\noindent
Figure~\ref{fig:push-pull} corroborates this mechanism on a two-dimensional toy problem, where a 4-kernel \textsc{SparKer} model fits three anomalous regions (yellow). The colored trajectories are the four kernels' locations during training.
Broad kernels first explore by spreading across the space. As training progresses, narrower kernels specialize on distinct regions, until each enters a phase of radial convergence (red in the zoomed panels) onto the only active point left in its Sphere of Influence.

To demonstrate how annealing promotes detection across scales, we perform a one-dimensional teacher-student experiment~\cite{freeman1997dynamics}.
The reference is a mixture of two unit-scale Gaussians at $c_{\pm}=\pm3$, and the signal is a single anomalous point at $x^*=-10$. In the infinite-statistics limit ($N\gg 1$) we train a 1-kernel model either at fixed width ($\sigma=5,1,0.5,0.05$) or with linear annealing, all from the same initialization.
Figure~\ref{fig:annealing} tracks the kernel's distance to the anomaly over training.
At fixed width, localization succeeds only for a well-matched scale (e.g.\ $\sigma=1$): broad kernels fail to localize, narrow ones evolve too slowly as their gradient vanishes far from the anomaly.
Annealing both locates the anomaly and refines resolution by gradually reducing the width.
Crucially, annealing also enables robust testing: the contributions of $f_\theta$ across time steps all enter the test statistic (see Algorithm~\ref{alg:algo}).
\begin{figure}[h]
\centering
\includegraphics[width=0.99\linewidth]{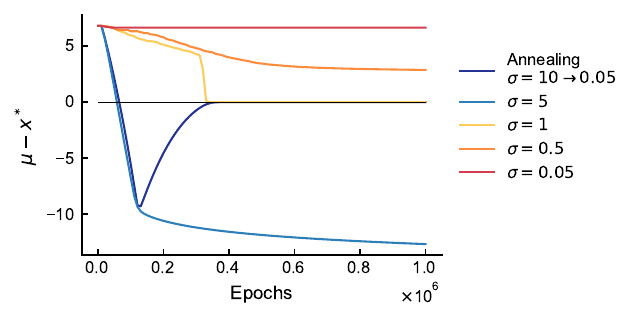}
\caption{\textbf{Scale annealing overcomes the vanishing gradient, leading to convergence.} Kernel distance from the anomalous point over training, using linear annealing (blue) or a fixed kernel width (other colors). Kernels that are too wide or too narrow either fail to localize the signal or do not converge in reasonable time. A fixed $\sigma=1$ localizes the signal, while annealing localizes it with progressively higher resolution.}
\label{fig:annealing}
\end{figure}
\subsection{Competition between kernels, interaction and self-organization.}
Annealing enables efficient exploration but does not exploit the collective properties of the ensemble. Specialization is further enhanced by the competitive nature of the \textsc{SoftMax} activation (Eq.~\ref{eq:activation}).
Because its output must sum to one, \textsc{SoftMax} couples the kernels: \textsc{SparKer}'s gradient with respect to a location $\mu_i$ now depends on nearby kernels, and among same-sign kernels active in the same region the one with the largest activation repels the others and wins.
This drives a spontaneous self-organization toward low-entropy configurations, which promotes specialization beyond exploration: it makes the model more interpretable and yields two concrete detection benefits---the ability to resolve multiple signals, and robustness to statistical fluctuations---which we demonstrate in turn.\\
\noindent
\emph{Multiple signals.} To demonstrate the first benefit, we study the one-dimensional model of Section~\ref{subsec:anneal} with two anomalous points at different distances from the bulk ($x_1^*=-5$, $x_2^*=-10$), hence of different strength.
Figure~\ref{fig:app_2k-2a} reports a 2-kernel model trained with and without \textsc{SoftMax}: without it, both kernels collapse onto the leading anomaly; with it, their mutual repulsion prevents collapse and the secondary anomaly is discovered.
\begin{figure}[h]
    \centering
    \includegraphics[width=\linewidth]{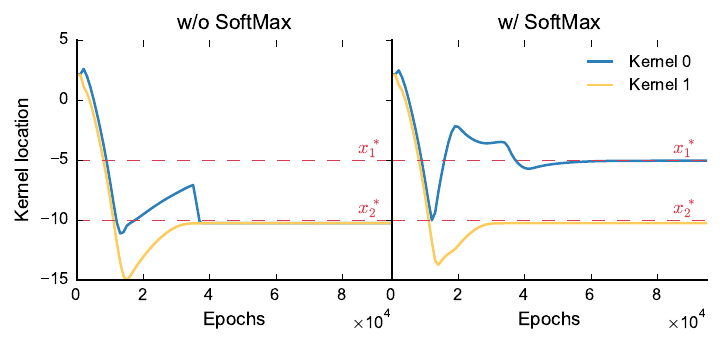}
    \caption{\textbf{The \textsc{SoftMax} activation helps exploration.} Kernel locations over training in a 2-kernel \textsc{SparKer} model with (right) and without (left) \textsc{SoftMax}. By making the kernels repel each other, \textsc{SoftMax} lets the model discover both the dominant anomaly at $x=-10$ and the secondary one at $-5$; without it, only the dominant anomaly is found.}
    \label{fig:app_2k-2a}
\end{figure}
\noindent
\emph{Robustness to fluctuations.} The second benefit is escaping local minima caused by statistical fluctuations, which can produce over- and under-densities of local significance comparable to a real anomaly.
By discouraging collapse onto any single solution, \textsc{SoftMax} sustains exploration and improves the odds of escaping such traps.
On the same one-dimensional model, we add a sharp signal (a Gaussian at $x^*=-5$, scale $0.1$) and, to expose fluctuations, use a limited training set.
Figure~\ref{fig:app_7k-1a} compares two 7-kernel models differing only in \textsc{SoftMax}: without it, kernels collapse into a few fluctuation-dominated locations and miss the anomaly at $-5$; with it, they explore more broadly and recover it.
We stress that fluctuations cost power, not validity: fluctuations occur identically in calibration, so the false-positive rate stays controlled.
\begin{figure}[h]
    \centering
    \includegraphics[width=\linewidth]{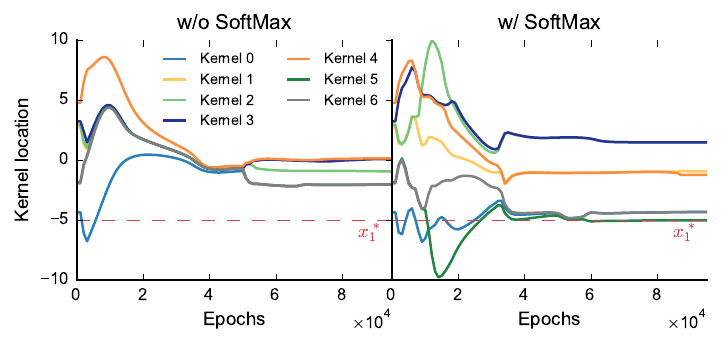}
    \caption{\textbf{The \textsc{SoftMax} activation helps escape local traps.} Evolution over training of the location of the kernels in a 7-kernel \textsc{SparKer} model with (right panel) and without (left panel) \textsc{SoftMax} activation. \textsc{SoftMax} induces a repulsion between kernels that allows the model to discover the real anomaly at $x=-5$, despite statistical fluctuations. }
    \label{fig:app_7k-1a}
\end{figure}

The theoretical analysis and controlled experiments presented in this Section highlight the strong detection performance of \textsc{SparKer} in regimes of low anomaly separability and high rarity.  
Additional experiments performed over multiple data realizations and initial conditions (reported in Appendix~\ref{app:abl}) show that  in this same regime (1) the NP loss systematically leads to higher performance than alternative binary classification losses, like binary cross-entropy or mean square error; (2) \textsc{SoftMax} systematically promotes a uniform kernel distribution over the space, enhancing sensitivity. 
\section{Applications}\label{sec:applications} 
\begin{figure*}[t]
    \centering
\includegraphics[width=\linewidth]{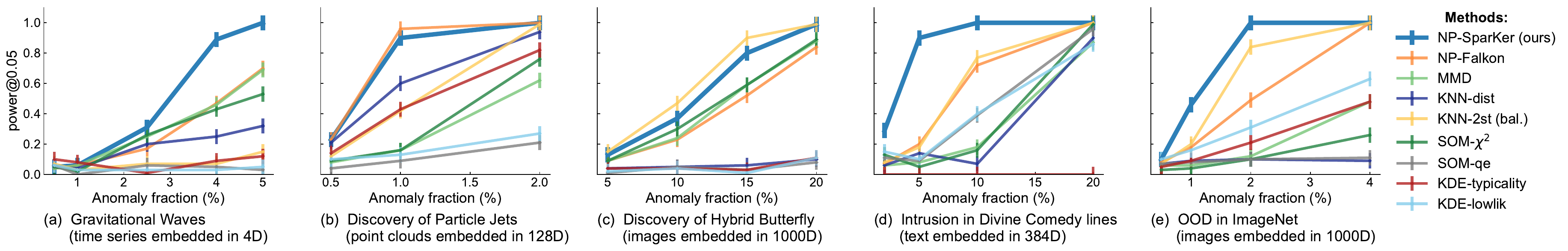}
\caption{\textbf{Sensitivity study: detection power as a function of anomaly fraction.} We report the performance of \textsc{SparKer} and the baselines in terms of detection power at 5\% false positive rate as a function of the anomaly fraction. Error bars represent the Clopper-Pearson 68\% confidence intervals. We consider various applications, differing in raw-data modality, anomaly fraction and the dimension of the feature space in which we embed the data: we embed gravitational-wave time series into a 4-dimensional representation (panel (a)), point clouds representing particle jets in a 128-dimensional space (panel (b)); butterfly images and ImageNet data in a 1000-dimensional space (panels (c) and (e)), and lines from Dante's Divine Comedy in a 384-dimensional space (panel (d)). We report additional details about the experiments in Appendix~\ref{app:appl}.} \label{fig:novelty}
\end{figure*}
We apply \textsc{SparKer} to a broad set of AD benchmarks encompassing scientific discovery, open-world novelty detection, intrusion detection, and validation of generative AI models. Despite differences in application domain, data modality, and feature-space dimensionality, these problems exhibit similar characteristics—in particular, anomalies are rare and in-distribution, the bottom left corner of Fig.~\ref{fig:main}(a). We outline each of them in the following sections and defer implementation details and the description of datasets to Appendix~\ref{app:appl}.

\subsection{Scientific discovery.}\label{subsec:science}
Novelty discovery in mature scientific domains is often characterized by extremely rare signals of unknown location and nature, making the desiderata we operationalize in \textsc{SparKer} particularly suitable. 
We present three examples of signal-agnostic novelty detection for scientific discovery, spanning the fields of astrophysics, particle physics and genomics. We encode data from the different modalities, namely time series of gravitational waves in astrophysics, point clouds from particle jets, and butterfly images from genomics, in rich latent representations and compare them to a nominal dataset using \textsc{SparKer}. For the gravitational waves and particle jets, the nominal sample consists of synthetic data coming from well controlled simulations of a theoretical model. For the application to genomics, the ground truth comprises experimental data labeled as nominal by experts in the field. 

\subsection{Open-world novelty detection.}\label{subsec:ood}
Out-of-domain (OOD) detection assesses the robustness and safety of machine learning systems trained under the \textit{closed-world} assumption~\cite{markou2003novelty}: inputs from the \textit{open world} differ from the training distribution and yield unreliable predictions, so detecting them is essential for trustworthy deployment~\cite{amodei2016concrete}.
Numerous methods have been benchmarked in computer vision~\cite{yang2022openood}, scoring the extent to which data lie out of distribution (``outness'') via classification, density, or distance.
A recent study~\cite{li2025out}, however, shows that these strategies are often misaligned with the detection objective: they assess feature proximity or prediction uncertainty rather than whether a sample originates from the training distribution, proxies that do not reliably capture distributional novelty.
To illustrate this, we test \textsc{SparKer} on the 1000-dimensional latent space of a ResNet50 classifier trained on ImageNet-1000~\cite{ILSVRC15}, injecting ``Striped'' samples from the Textures dataset~\cite{cimpoi2014describing} as OOD anomalies---a class the network classifies with high confidence, exposing logit-based detection to failure~\cite{li2025out}.

\subsection{Intrusion detection.}\label{subsec:intrusion}
We prompted Meta-Llama-3-70B-Instruct~\cite{llama3modelcard} to randomly generate lines from Dante's Divine Comedy and  embedded these into 384-dimensional representations using multilingual Sentence-BERT~\cite{reimers-2019-sentence-bert}. We tested samples comprising 1000 lines. To simulate a problem of intrusion detection, we injected 2 to 20\% AI-generated lines into the dataset and tested \textsc{SparKer}'s detection power.

\subsection{Validation of Generative AI.}\label{subsec:genAI}
Another relevant application of \textsc{SparKer} is the validation of generative models. AI techniques to generate artificial samples are rapidly improving, producing synthetic data that are often hard to distinguish from real data. Detecting distributional shifts can highlight subtle differences between real and artificial generation processes, and identify representational bias, helping refine generative AI and develop more reliable tools. 
We compared the original CIFAR-10 dataset~\cite{krizhevsky2009learning} with the CIFAR-5m dataset introduced in~\cite{nakkiran2020deep}. This dataset was artificially generated by sampling from the DDPM unconditional generative model developed in~\cite{ho2020denoising} to imitate CIFAR-10.
We apply \textsc{SparKer} to CIFAR-10 and CIFAR-5m samples embedded in the 1000-dimensional feature space output by a ResNet50 model pretrained on ImageNet. This representation provides a rich basis of features that is not tailored to CIFAR, making the problem harder but realistic. 
The CIFAR-5m generated data takes the role of the reference hypothesis, while the CIFAR-10 data that of the test set. 

\subsection{Results.}\label{subsec:baselines}
Figure~\ref{fig:novelty} presents the performance of \textsc{SparKer} across the scientific discovery, open-world novelty detection and intrusion detection applications. We report the test power at a 5\% false-positive rate as a function of the anomaly fraction. The power is calculated over 100 repetitions with Clopper-Pearson 68\% confidence intervals.
To put the performance of \textsc{SparKer} in context, we compare against two-sample tests based on static kernels randomly initialized from data (Falkon~\cite{Letizia:2022xbe,Grosso:2024wjt}, MMD~\cite{chatalic2026scalable}), SOM-based metrics (quantization error~\cite{vettigli2018minisom} and binned $\chi^2$), distance-based methods (KNN distance~\cite{ramaswamy2000,angiulli2002} and KNN two-sample test~\cite{schilling1986,henze1988}), and density-estimation methods (low-likelihood score~\cite{bishop1994novelty} and typicality~\cite{nalisnick2019typicality}). Metrics designed for point-wise detection are promoted to two-sample tests by averaging over the data points in $\data$. Where sensitivity depends on a single hyper-parameter (bandwidth for kernels, number of prototypes for SOMs, number of neighbors for KNN), we evaluate the test over a range of values and combine them into a single calibrated test, as in Algorithm~\ref{alg:algo}, which reduces model-selection bias.
For a fair comparison, all methods, including \textsc{SparKer}, are calibrated on 300 anomaly-free bootstrap replicas (Section~\ref{subsec:algo-steps}). Definitions and implementations are detailed in Appendix~\ref{app:baselines}. 
Overall, \textsc{SparKer} consistently performs similarly to or outperforms the baselines.
In the open-world novelty detection application (panel (e)), we obtain 46\% power with an anomaly fraction as low as 1\%, whereas baselines targeting out-of-distribution anomalies are below 20\%. 
This exposes a general principle: with a sufficiently-sampled reference, out-of-distribution novelties reappear as in-distribution excesses, which is the regime \textsc{SparKer} is designed for, and where detectors based on ``outness'' are misaligned and thus suboptimal.
In the intrusion detection application (panel (d)), \textsc{SparKer} retains almost full power for an anomaly fraction as low as 5\%, a power at least five times that of the considered baselines. This enables the detection of subtle anomalous patterns that could otherwise not be discovered.
For the application on generative AI validation, using a reference set of $\Nreference=10^5$ and a test set of $\Ndata=10^4$ points, we find an upper limit on the empirical $p$-value of $0.005$. This corresponds to a disagreement between the synthetic and real data distributions at the $2.6\sigma$ level, revealing for the first time a non-negligible distribution bias.
\section{Extracting intelligible knowledge}\label{sec:interp}
Interpretability refers to the ability of a model to yield explanations that are both faithful to its internal mechanisms and meaningful to human understanding~\cite{doshi2017towards}.
In AD, interpretability has evolved from \textit{post-hoc} explanations—such as feature attribution, saliency maps, and surrogate models—toward intrinsic and mechanistic approaches that embed human-understandable structure within the model itself~\cite{li2023survey}.
Recent studies highlight how superposition, modularity, and representational geometry shape interpretability~\cite{higgins2018towards,elhage2022toy,klindt2025superposition,cunningham2023sparse,colin2024local}, motivating architectures that are transparent by design~\cite{sharkey2025open}.
In this context, \textsc{SparKer} provides a geometric framework in which sparsity, locality, and competition yield features that directly map to regions of distributional deviations in the data. Below we explain how to extract geometric insights and interpretations from \textsc{SparKer}.
 
\begin{figure*}[t]
    \centering
\includegraphics[width=0.979\linewidth]{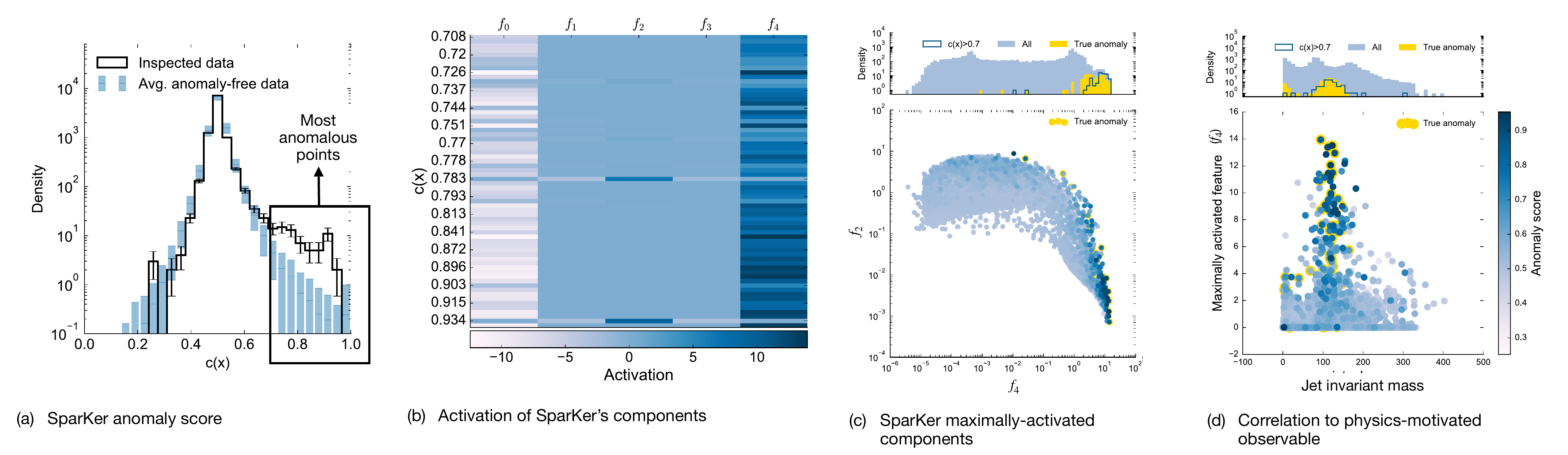}
    \caption{\textbf{Inspecting anomalous particle jets.} (a) Distribution of the anomaly score $c(x)$ over the inspected data (black histogram), compared with the average expected distribution in the absence of anomalies (error bars: standard deviation over 100 realizations). 
    (b) Activation pattern over \textsc{SparKer}'s components for the most anomalous points; $f_4$ and $f_2$ are the most positively activated components at high anomaly score.
    (c) Correlation of the two maximally-activated components, color coded by anomaly score: anomalous points are pushed to the edges of the distribution.
    (d) Correlation of $f_4$ with the jet invariant mass, a physics-motivated observable. Without prior physics knowledge, the model highlights jet masses around 125 GeV, matching the Higgs boson mass used as the anomalous signal here. 
    }\label{fig:LHCjet_interp}
\end{figure*}
\begin{figure*}[h]
    \centering
\includegraphics[width=0.979\linewidth]{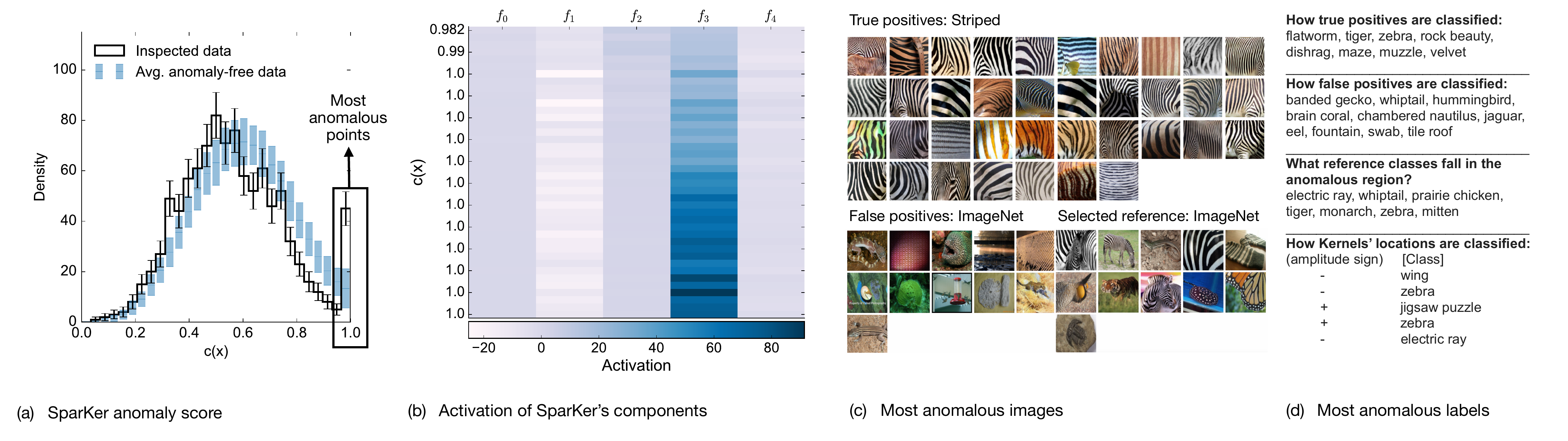}
    \caption{\textbf{Inspecting detected open-world novelties.} (a) Distribution of the anomaly score $c(x)$ over the inspected data (black histogram), compared with the average expected distribution in the absence of anomalies (error bars: standard deviation over 100 realizations). 
    (b) Activation pattern over \textsc{SparKer}'s components for the most anomalous points; $f_3$ is the most positively active.
    (c) Images of the most anomalous points, distinguishing true positives (``Striped'' class in Textures), false positives, and reference events (ImageNet test set).
    (d) Predicted labels of the most anomalous points: though localized around $f_3$, they span different ImageNet classes, indicating a non-trivial geometric signature.}\label{fig:imagenet_interp}
\end{figure*}
\subsection{Anomaly score.}
The most informative output of the algorithm is the log-density ratio, namely the model $f_M^\sigma$, which we can transform into an \textit{anomaly score} via the Sigmoid activation:
\begin{equation}
    c(\vx) = {\rm Sigmoid}(f_M^\sigma(\vx))
\end{equation}
The anomaly score lets us identify and select the most anomalous points in the dataset. Figures~\ref{fig:LHCjet_interp}(a) and~\ref{fig:imagenet_interp}(a) show anomaly-score distributions for samples containing anomalous points (black histogram), alongside the average distributions for anomaly-free datasets as a calibration baseline (light blue bars) with error bars that give the standard deviation over 100 replicas. As with standard classifiers, the anomaly score reflects a confidence level on the anomaly label, indicating the expected density of anomalies in a given region of the input space. A threshold on the right tail of the distribution identifies regions of highest over-density for further inspection, while the left tail highlights regions of under-density.
\subsection{Geometric characterization via feature specialization.}
In contrast to existing models, \textsc{SparKer} provides a geometric perspective on the anomaly score  arising from the interplay of locality, sparsity, and competition. The model naturally decomposes into local components:
\begin{equation}
    f_M^\sigma(\vx) = \sum_{i=1}^M f^\sigma_i(\vx) \, .
\end{equation}
Due to the $\rm SoftMax$ activation in the model's design, different kernels specialize to different subsets of the anomalous data and activate in response to different attributes, allowing us to disentangle the role of individual kernels and their interpretation. 
Pairing the anomaly score with the activation pattern helps to characterize anomalies and their properties (Figures~\ref{fig:LHCjet_interp}(b) and \ref{fig:imagenet_interp}(b)). 

\noindent In the following, we focus on three of the applications discussed in Section~\ref{sec:applications}: anomalous jet discovery in particle physics, novel class discovery in open-world settings, and intrusion detection in text streams. We present complementary analyses addressing generative-model validation and butterfly-species discovery in Appendix~\ref{app:interp}.
\subsection{Interpreting novel particle jets.}
Figure~\ref{fig:LHCjet_interp} illustrates the discovery of anomalous particle jets introduced in Section~\ref{subsec:science}, focusing on the 1\% signal injection case. Panel (b) shows that anomalous events exhibit a strong positive correlation with \textsc{SparKer} component $f_4$. To further characterize the signals detected, panel (d) displays the correlation between $f_4$ and a physics-motivated observable related to the jet’s invariant mass. The true anomalous points (in yellow) are strongly activated by $f_4$ and cluster around the Higgs boson mass region (125 GeV), which corresponds to the dominant attribute of the injected signal. This demonstrates \textsc{SparKer}’s ability to isolate and explain physically-meaningful structure within complex data.

\subsection{Extracting insights on open-world novelties.} Figure~\ref{fig:imagenet_interp} summarizes the interpretability analysis of the open-world novelty detection task introduced in Section~\ref{subsec:ood}. The anomaly score distribution in panel (a) highlights the most anomalous samples, while their correlation with feature activations in panel (b) reveals that the anomalies are primarily associated with features $f_3$ and $f_1$. Applying a threshold of $c=0.95$ yields a selection efficiency of 92.5\% and a false positive rate of 1.1\%. Panel (c) shows the corresponding images, distinguishing true positives, i.e. out-of-domain samples from the ``Striped'' class in the Textures dataset, from false positives and reference samples from ImageNet. Notably, the reference images selected above the threshold exhibit striped patterns, suggesting that this concept was implicitly learned by ResNet50 despite the absence of explicit supervision. Interestingly, ResNet50 classifies the centroids of components $f_1$ and $f_3$ both as ``zebra'', although their locations and the sign of their correlations with anomalies in logit space differ, indicating that the novel ``Striped'' concept overlaps only partially with the known ``zebra'' representation, and relates to other known classes such as ``tiger'', ``velvet'' and ``dishrag''.  

\begin{figure*}[t]
    \centering
    \includegraphics[width=0.979\linewidth]{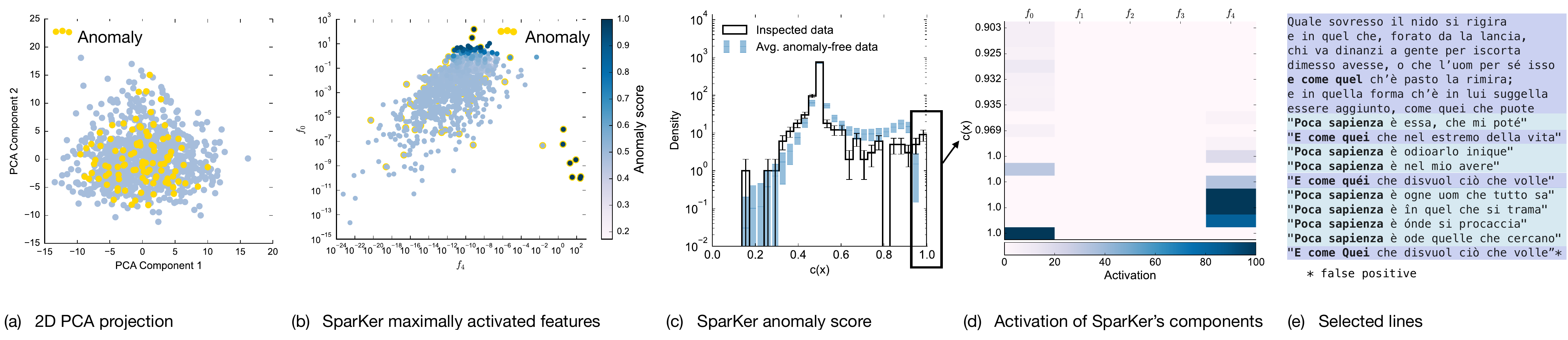}
    \caption{\textbf{Inspecting detected intrusions in text lines.} 
    (a) 2D PCA projection of the data, where the anomalous points, highlighted in yellow, lie on the support of the normal data. (b) 2D projection obtained using the two maximally-activated features found by \textsc{SparKer}; in contrast to PCA, \textsc{SparKer} finds low-dimensional representations that single out the most anomalous regions of the data. (c) Distribution of \textsc{SparKer}'s anomaly score over the inspected data (black histogram), compared to the average expected distribution in the absence of anomalies (the light blue bars measure the standard deviation over 100 anomaly-free realizations). (d) Activation pattern of \textsc{SparKer}'s components for the most anomalous data points. Two components, $f_0$ and $f_4$, exhibit the most correlation with the anomaly score. (e) Most anomalous lines colored by their activation pattern.}\label{fig:intro_interp}
\end{figure*}
\subsection{Extracting insights on intruded text.}Figure~\ref{fig:intro_interp} summarizes the interpretability study performed on the language application we presented in Section~\ref{subsec:intrusion}. 
Panel (a) displays a simple 2D PCA projection of the data, where anomalous samples (yellow dots) overlap with nominal points, indicating that simple linear projections cannot separate them. In contrast, panel (b) shows a 2D projection derived from the most activated \textsc{SparKer} components, in which the anomalous points are distinctly isolated. Panel (c) presents the distribution of the anomaly score over the inspected data (black histogram) as opposed to the typical behavior of anomaly-free realizations (light blue bars). We can identify a group of highly anomalous points (in the black box) for further analysis. 
The activation patterns of the most anomalous data points are displayed in panel (d), and their corresponding text lines in panel (e). We recognize two dominant patterns, highlighted in different colors in panel (e), corresponding to the activation of components $f_0$ and $f_4$. Both patterns reveal over-represented linguistic expressions. In the light blue group, associated with $f_4$, we recognize the repeated use of the phrase \textit{``Poca sapienza \`e''} (``little wisdom is'') at the beginning of each line, a construction entirely absent in the Divine Comedy. In the dark blue group, we identify the recurrent phrase \textit{``E come quei che''} (``and like those who'') at the line openings. This expression does appear in the Divine Comedy in the form marked by an asterisk at the bottom of the list, but the LLM generates multiple variations with hallucinated continuations.   
\section{Final remarks}\label{sec:conclusion}
As our models of the world sharpen and the obvious signals are exhausted, what remains unexplained retreats into the rare and the subtle. Building instruments sensitive to that residue — faint, in-distribution, easily mistaken for noise — is how discovery continues once the easy answers are spent. We used sparsity, locality and competition as the bases of an anomaly detector, \textsc{SparKer}, whose advantage over standard two-sample tests is largest precisely where detection is hardest. For rare, weakly separable anomalies hidden within a well-understood reference, \textsc{SparKer} retains high power where state-of-the-art alternatives lose most of theirs, and, unlike methods specialized to a single anomaly regime, it remains sensitive across all of them simultaneously. We demonstrated this across applications ranging from scientific discovery and generative-model validation to intrusion and open-world novelty detection, spanning diverse data modalities, feature-space dimensions, and signal fractions.
With \textsc{SparKer}, we demonstrate the importance of careful design choices and desiderata to control the failure modes of AD methods and to characterize the conditions under which we can expect them to succeed. 
Beyond its detection power, \textsc{SparKer} naturally lends itself to interpretability: its local components specialize to distinct regions of the data space, yielding a geometric decomposition of the anomaly score assigned to each input. This property paves the way for trustworthy, rule-based anomaly triggers, scientific insight extraction, and continual learning. 
Pairing sensitivity in the rare, in-distribution regime with a geometric characterization of flagged anomalies can turn SparKer into a powerful and broadly useful exploratory tool to advance science.

The design of \textsc{SparKer} leaves open several exciting directions of future advances, from adaptive kernel covariances to trainable log-concave alternatives to Gaussian kernels~\cite{lin2023train}. 
While we focused on semi-supervised anomaly detection, the model’s capacity to capture local geometric structure suggests broader potential~\cite{yeh2018representer}. Future work will explore local activation mechanisms as a foundation for robust and interpretable model design beyond anomaly detection, advancing the integration of locality, interpretability, and adaptability in modern machine learning systems. 

Finally, we expose the challenges introduced by AI-based data representations that, while rich and informative, do not provide guarantees on OOD generalization properties. We showed that \textsc{SparKer} enables detection even when signal separability is not guaranteed, while helping scientists to characterize these new spaces geometrically. A question we have left unanswered is what constitutes ``good'' feature extractors for anomaly detection. We conjecture that ``good'' feature extractors make it simple to relate anomalies in representation space back to attributes of inputs for interpretability, suggesting some form of invertibility as important.

\section*{Acknowledgment}
G.G., S.B.T., and P.H acknowledge the financial support of the National Science Foundation under Cooperative Agreement PHY-2019786 (The NSF AI Institute for Artificial Intelligence and Fundamental Interactions, http://iaifi.org/).
D.B. and T.F. acknowledge the gift from the Chan Zuckerberg Initiative Foundation to establish the Kempner Institute for the Study of Natural and Artificial Intelligence at Harvard University. Computations in this paper were run on the FASRC Cannon cluster supported by the FAS Division of Science Research Computing Group at Harvard University.

\clearpage
\onecolumn 
\newpage
\appendices
\section{Model design and implementation}\label{app:hyperp}
In this Section we provide additional details on the model design and implementation.

\subsection{Deterministic annealing}
The vanishing gradient represents the main engineering challenge to train local kernels. We mitigate this issue by imposing deterministic annealing to the kernels' width.
Starting from a wide width, $\sigma_0$, we linearly reduce the kernels' width up to a final narrow width, $\sigma_T$, taking equally large steps at each epoch $t$:
\begin{equation}
    \sigma_t = \sigma_0 + (\sigma_T - \sigma_0)\cdot \frac{t }{T} 
\end{equation}

\subsection{Amplitude regularization for robust testing}
When statistical power is low, density estimates are sensitive to fluctuations, leading to model misspecification and overfitting in sparse regions. Huber’s seminal work on robust hypothesis testing~\cite{huber1965robust} showed that minimax-optimal tests under distributional uncertainty clip the likelihood ratio to limit sensitivity to outliers. Following this principle, we clip \textsc{SparKer}’s kernel amplitudes as a regularization against noisy regions, with the clipping coefficient tuned in a signal-agnostic way using anomaly-free data. We further apply standard L2 regularization.

\subsection{Model initialization}
The kernels' amplitudes are initialized at zero, making the starting model identically null over the entire data domain. This initialization reduces initial bias towards specific solutions and allows a slow evolution towards meaningful solutions as the widths are annealed.
The width range for annealing is defined by estimating the typical range of pair-wise distance between points. We compute the 1\% and 99\% percentiles of the empirical distribution of the pair-wise distance between data points (a subset of the data is used for computational reasons). The initial width is set at twice the 99\% percentile, while the final width is set at half the 1\% percentile. 
The kernels' locations are randomly initialized among the locations of the data points. This is an inexpensive way to ensure that the kernels are initially located in-distribution. This initialization, combined with the range chosen for the kernels' widths, ensures with high probability that at least one model gradient is significantly larger than zero for each data point.

\subsection{Multiple resolution test}
Multiple solutions based on different values of the kernels' width can be retrieved from intermediate stages of the annealing process happening during training. 
A priori there is no way to tell which value of the width encountered during annealing is optimal for testing, as the latter depends on the specifics of the anomalous signal. To obtain robustness against scale variations, we perform multiple testing across values of the width and account for the trial factor by combining the resulting p-values into a single test. Based on previous studies on multiple testing with the NP test~\cite{Grosso:2024nho}, the aggregation rule we adopt is the average of the minimum $p$-value and product of $p$-values~\cite{Grosso:2024wjt}. For an observed value of the test, $t_{\rm obs}$, the  $p$-value is estimated empirically, as
\begin{equation}\label{eq:emp_pvalue}
   \hat{p}_{\rm obs} = \frac{1}{N_{\rm toys}^{({\rm H_0})}+1} \left[\sum_{i=1}^{N_{\rm toys}^{({\rm H_0})}} \mathbb{I}(t_i>t_{\rm obs})+1\right], 
\end{equation}
where $\{t_i\}_{i=1}^{N_{\rm toys}^{(H_0)}}$ represents a collection of anomaly-free realizations used for calibration.

\subsection{Choosing the number of kernels}
The number of kernels $M$ is the model's capacity budget, and can be read as the maximum number of distinct anomalous regions the ensemble can resolve simultaneously.
Two considerations bound it from opposite sides.
From above, capacity is not free in a hypothesis test: each additional kernel offers more ways to accommodate statistical fluctuations of anomaly-free data, inflating the null distribution of the test statistic and raising the bar a genuine signal must clear---the post-selection bias internalized in the model design (Section~1.3.1 of the main text).
From below, $M$ must be large enough to explore the space efficiently, resolve the expected multiplicity of anomalous regions and to support the geometric decomposition of Section~4 of the main text. 
When prior guidance on the anomaly multiplicity is unavailable and the computational budget allows, the residual dependence on $M$ can be removed by aggregating tests over a small grid of $M$ values with the same $p$-value combination used for the kernel width (Eq.~\ref{eq:aggregation}), paying the associated trial factor explicitly; this mirrors the treatment we apply to the kernels' width (see Algorithm~1 in the main article).
In our baseline comparisons (Section~\ref{subsec:toys}), which are all characterized by unimodal signals, we present $M=2$ and $M=5$ as well as their aggregation. 
We adopt $M{=}5$ as the default in the high-dimensional applications presented in Section~3 of the main article as a good signal-agnostic compromise.
\section{Theoretical analysis}\label{app:th}
In this section, we develop the theoretical underpinnings of \textsc{SparKer}'s anomaly detection properties. We restate and prove the results included in Section~2 of the main article, following the same order and numbering.
\noindent The Neyman-Pearson (NP) loss which is optimized to train the \textsc{SparKer} model is defined as:
\begin{align}
        L_{\rm NP}[f]= \sum\limits_{x\in \reference}w_\reference(e^{f(x)}-1) -\sum\limits_{x\in\data} f(x)\,.
        \label{eq:nploss}
    \end{align}
The dynamics of mean $\vmu_i$ and amplitude $a_i$ are given by gradient descent on the NP Loss:
\begin{align}
    \vmu_{i;t+1} &= \vmu_{i;t}-\eta \left.\frac{\partial L_{\rm NP}[f_{\theta}]}{\partial \vmu_{i}}\right|_{\theta = \theta_t},
    \label{eq:grad-descent-mu}\\
    a_{i;t+1} &= a_{i;t} - \eta \left. \frac{\partial L_{\rm NP}[f_{\theta}]}{\partial a_{i}}\right|_{\theta = \theta_t},
    \label{eq:grad-descent-a}
\end{align}
where $\theta_t =\{\vmu_{1;t}, \dots \vmu_{M;t}, a_{1;t}, \dots a_{M,t}\}$ represents all trainable parameters at training iteration $t$. The loss function $L_{\rm NP}$ includes evaluating the model $f$ at multiple data points $\vx \in \reference \cup \data$. We focus our analysis on the dynamics of the location $\vmu_i$ and amplitude $a_i$ of a generic kernel $i$ in the \textsc{SparKer} ensemble.

\subsection{Effect of individual points on kernel means and amplitudes}
We first compute the contribution of an individual data point $\vx$ to the gradients of the model with respect to the kernel parameters.
Since $f$ is parameterized using Gaussian kernels, i.e.\ $f(\vx) \equiv f(\|\vr_1\|^2,\dots \|\vr_M\|^2)$ with $\vr_i = \vx-\vmu_i$, the chain rule gives $\partial_{\vmu_i}f = (\partial_{\|\vr_i\|^2}f)\cdot (-2\vr_i)$; computing the prefactor directly via the product rule yields
\begin{align}
    \partial_{a_i} f(\vx) &= p^i(\vx)\cdot k^\sigma_{\vmu_i}(\vx),
    \label{eq:app_grad_a}\\
    \partial_{\vmu_i} f(\vx) &= A_i(\vx) \cdot \vr_i,
    \qquad
    A_i(\vx)=\frac{1}{\sigma^2}\cdot \underbrace{\left[2\,a_i k_{\vmu_i}^\sigma(\vx) - f_M^\sigma(\vx)\right]}_{m_i(\vx)} \cdot\, p^i(\vx).
    \label{eq:app_grad_mu}
\end{align}
The gradient with respect to the amplitude (Eq.~\ref{eq:app_grad_a}) is a scalar field, locally active around the kernel's mean and vanishing far from it. The gradient with respect to the location (Eq.~\ref{eq:app_grad_mu}) is a radial vector field: each data point imposes a radial influence on the kernel mean $\vmu_i$, whose direction along $\vr_i$ is set by the sign of the magnitude $A_i(\vx)$, and hence by the sign of the kernel's ``mass charge'' $m_i(\vx)=2a_i k_{\vmu_i}^\sigma(\vx) - f_M^\sigma(\vx)$, since $\frac{1}{\sigma^2}\,p^i(\vx)>0$.
The radial influence of each data point can either be attractive (radially toward the point) or repulsive (radially away from the point) on the kernel mean $\vmu_i$. The conditions which determine the direction of influence are specified in the following lemma.
\begin{lemma}[Push-pull dynamic of kernels' locations]\label{lemma:pushpull}
The parameter update of the $i^{th}$ kernel location $\Delta\vmu_i$ (Eq.~\ref{eq:grad-descent-mu}) is the result of radial forces arising from the training points. Let $y$ be the class label associated with each data point ($y=0$ for $\vx\in \reference$ and  $y=1$ for $\vx \in \data$), and let $m_i(\vx)=2a_i k^{\sigma}_{\vmu_i}(\vx) - f^\sigma_M(\vx)$ be the $i^{th}$ kernel ``mass charge'' in $\vx$. Each update $\Delta \vmu_i$ is:
\begin{align}
    \text{``radially attractive''} \quad \text{if}\quad m_i(\vx)\cdot(2y-1)>0 \notag \\
    \text{``radially repulsive''} \quad \text{if}\quad m_i(\vx)\cdot(2y-1)<0 \notag
\end{align}
\end{lemma}
\begin{proof}
    The dynamics of kernel means is given by Eq.~\ref{eq:grad-descent-mu}. For a generic parameter $\theta = \{\vmu_i, a_i \}$, the gradient of NP Loss (Eq.~\ref{eq:nploss}) is given as:
    \begin{equation}\label{eq:app_nablaL}
      \partial_{\theta} L_{\rm NP}[f_{\theta}] = w_{\reference} \sum\limits_{\vx\in \reference} e^{f_{\theta}(\vx)} \partial_{\theta} f_{\theta}(\vx) - \sum\limits_{\vx\in \data}  \partial_{\theta} f_{\theta}(\vx),
\end{equation}
while the parameter update through gradient descent is along the negative gradient of the loss (Eq.~\ref{eq:grad-descent-mu}):
\begin{align}
    \Delta \theta = \theta_{t+1}-\theta_t \propto -\partial_\theta L_{NP}.
\end{align}
Therefore, for points $\vx \in \reference$, $\partial_{\theta} L_{\rm NP}[f_{\theta}]$ and $\partial_{\theta} f_{\theta}(\vx)$ have the same sign, while for $\vx \in \data$, $\partial_{\theta} L_{\rm NP}[f_{\theta}]$ and $\partial_{\theta} f_{\theta}(\vx)$ have the opposite sign. The overall effect of $\partial_\theta f$ on $\Delta \theta$ can be captured by multiplying $\partial_\theta f$ with the term $(2y-1)$ where $y=\begin{cases}
    0 & \vx\in \reference,\\
    1 & \vx\in \data
\end{cases}$, since this induces an overall sign flip for points $\vx \in \reference$. Additionally, the magnitude of the gradient, $A_i(\vx)$ in Eq.~\ref{eq:app_grad_mu}, can itself be positive or negative, with its sign determined by the mass charge $m_i(\vx)$, which depends on the kernel amplitudes, values and result of ${\rm SoftMax}$ weighting. Radially attractive behavior results when $\Delta \vmu_i \propto \vr_i = \vx-\vmu_i$, while repulsive behavior results when $\Delta \vmu_i \propto -\vr_i = \vmu_i - \vx$. The result for kernel location follows from the above using the gradient expression in Eq.~\ref{eq:app_grad_mu}.
\end{proof}

\subsection{The bulk optimum is zero}
Before turning to the question of which points significantly affect the dynamics, we restate and prove the fixed-point result of the main text (Proposition~1), which establishes that statistically balanced regions offer no reward to the model. In the continuum limit, the gradient of the loss (Eq.~\ref{eq:app_nablaL}) reads
\begin{align}
      \partial_{\theta} L_{\rm NP}[f_{\theta}]
      \rightarrow
      \int d\vx\, \left[ n(\vx|H_0)\, e^{f_{\theta}(\vx)} - n(\vx|\data)\right]\partial_{\theta} f_{\theta}(\vx),
      \label{eq:continuum}
\end{align}
where $n(\vx|H_0)$ and $n(\vx|\data)$ denote the number densities of the reference model and of the inspected data.
\begin{proposition}[The bulk optimum is zero]\label{prop:bulk-zero-SI}
In the continuum limit of Eq.~\ref{eq:continuum}, the functional derivative of the NP loss vanishes at
$$
f^{*}(\vx)=\log\!\big[n(\vx|\data)/n(\vx|H_0)\big],
$$
which equals $0$ wherever $\data$ and $\reference$ are locally balanced, $n(\vx|\data)=n(\vx|H_0)$.
Gradient descent therefore drives $|f_\theta|\to 0$ on statistically balanced regions.
\end{proposition}
\begin{proof}
In the continuum limit, the NP loss (Eq.~\ref{eq:nploss}) becomes the functional
\begin{align}
    L_{\rm NP}[f] \rightarrow \int d\vx\, \Big[ n(\vx|H_0)\big(e^{f(\vx)}-1\big) - n(\vx|\data)\, f(\vx) \Big],
\end{align}
whose functional derivative with respect to $f$ is
\begin{align}
    \frac{\delta L_{\rm NP}}{\delta f(\vx)} = n(\vx|H_0)\, e^{f(\vx)} - n(\vx|\data).
    \label{eq:funcder}
\end{align}
Setting Eq.~\ref{eq:funcder} to zero yields $e^{f^*(\vx)} = n(\vx|\data)/n(\vx|H_0)$, i.e.\ $f^{*}(\vx)=\log[n(\vx|\data)/n(\vx|H_0)]$. Since the integrand is strictly convex in $f(\vx)$ (its second derivative, $n(\vx|H_0)e^{f(\vx)}$, is positive), $f^*$ is the unique minimizer. Wherever the two samples are locally balanced, $n(\vx|\data)=n(\vx|H_0)$, the optimum is $f^*(\vx)=0$; there, the functional derivative reduces to $n(\vx|H_0)\big(e^{f(\vx)}-1\big)$, which shares the sign of $f(\vx)$, so gradient descent decreases $|f_\theta|$ on balanced regions---either by sending the amplitudes of the kernels active in the region to zero, or by repelling their locations out of it (Lemma~\ref{lemma:pushpull}).
\end{proof}
\noindent The consequence is that kernels do not spend capacity on balanced regions, but instead seek out regions of imbalance between $\data$ and $\reference$. Locality makes this cheap: a region with no nearby kernel already outputs $f\approx0$, precisely the bulk optimum, so the null consumes no capacity and the small kernel budget is free to concentrate where $\data$ and $\reference$ disagree.

\subsection{Kernel dynamics are affected only by nearby data points} Having studied the effect of individual data points on parameter updates, we now turn to the question of which points from the data significantly affect the dynamics of kernel means. We show that only those data points that are within a neighborhood of the kernel mean $\vmu_i$ can contribute to the dynamics of $\vmu_i$. Furthermore, this neighborhood is dependent on the scale $\sigma$. We first set up useful definitions:
\begin{definition}[Kernel radial $\alpha$-level]
Let $k:\mathbb{R}_{+}\to\mathbb{R}_{+}$ be a continuous, radial kernel profile that is nonincreasing in $r\ge0$ (i.e. $k(r_1)\ge k(r_2)$ for $r_1\le r_2$).
For a given level $\alpha\in(0,k(0))$ define the \emph{radial $\alpha$-level} $r_\alpha$ as the radius such that $k(r) < \alpha \quad \forall r>r_\alpha$,
or equivalently $r_\alpha \;=\; \sup\{ r\ge0 \;:\; k(r)\ge\alpha\}$.
\label{def:alphalevel}
\end{definition}
By monotonicity of $k$, the set $\{r\ge0: k(r)\ge\alpha\}$ is an interval $[0,r_\alpha]$ (possibly closed or half-open), so $r_\alpha$ is well defined.
For the Gaussian profile $k(r)=\exp\!\big(-r^2/(2\sigma^2)\big),$
the level-set quantile admits the closed form
$r_\alpha = \sigma\sqrt{-2\ln\alpha}$,
valid for $0<\alpha<1$.
Using a kernel's radial $\alpha$-level, we show that each kernel has a \textit{sphere of influence} which characterizes the region in space where data points can significantly affect the dynamics of kernel mean $\vmu_i$ and amplitude $a_i$.
\begin{lemma}[Kernel's Sphere of Influence]\label{lemma:soi}
For each kernel $i$ that makes up the \textsc{SparKer} model, there exists a value $\alpha\in (0,1)$ such that the associated ball $S_{i}^\sigma (\alpha) = \{\vx \in \mathbb{R}^d: \|\vx-\vmu_i\|\leq r^\sigma_i(\alpha)\}$ contains all the training points that contribute at least $\epsilon_\alpha$ to the kernel's dynamics, i.e., points outside this ball have no more than $\epsilon_\alpha$ effect on the kernel's dynamics.\\ We name this ball the kernel's ``Sphere of Influence'':
\begin{equation}\notag
    \exists\, \epsilon_\alpha>0 \quad s.t. \quad |\partial_{\theta_i}f(\vx)| < \epsilon_\alpha, \quad \forall \vx \notin S_i^{\sigma}(\alpha) \quad {\rm with}\quad \theta_i \in \{\vmu_i,\, a_i\}.
\end{equation}
If $\alpha\ll1$ then $\epsilon_\alpha\ll1$ as well.
Namely, the dynamics is local on a scale that is determined by the kernels' width, $\sigma$. 
\end{lemma}
\begin{proof}
(Kernel amplitude)
The gradient of $f$ with respect to $a_i$ is given by Eq.~\ref{eq:app_grad_a}.
Since $p^i(\vx) \in [0,1]$ and for $\vx \notin S_i(\alpha)$ $k^\sigma_{\vmu_i}(\vx)\in [0, \alpha]$, we can write the following bound:
\begin{equation}
    0 \le \partial_{a_i} f(\vx) \le \alpha \;\; \forall \vx \notin S_i(\alpha).
\end{equation}
Choosing $\epsilon_\alpha = \alpha$ gives us the desired result for kernel amplitudes.
\newline(Kernel mean). Consider now the gradient of $f$ with respect to kernel mean $\vmu_i$ (Eq.~\ref{eq:app_grad_mu}):
    \begin{align}
    \partial_{\vmu_i} f(\vx) &=
    A_i(\vx) \cdot \vr_i, \\
        A_i(\vx) &=\frac{1}{\sigma^2}\cdot \left[2\,a_i k_{\vmu_i}^\sigma(\vx) - f_M^\sigma(\vx)\right] \cdot p^i(\vx).
    \end{align}
    The coefficient $A_i(\vx)$ determines the contribution of the point $\vx$ to the gradient $\partial_{\vmu_i}L_{\rm NP}$ driving the dynamics of $\Delta \vmu_i$. Since $A_i(\vx)$ is the result of the product of three terms, we bound it by obtaining upper bounds on each term separately.\\
    Consider a point $\vx \notin S_i^{\sigma}(\alpha)$. \\
    The factor $\frac{1}{\sigma^2}$ does not depend on $\vx$ and is positive and bounded (design choice).
    The factor $\left[2\,a_i k_{\vmu_i}^\sigma(\vx) - f_M^\sigma(\vx)\right]$ depends on $k_{\vmu_i}^\sigma(\vx)$ and $f_M^\sigma(\vx)$.
    From the definition of the set $S_i^{\sigma}(\alpha)$,
    \begin{align}
        \vx \notin S_i^\sigma (\alpha) \implies \|\vx -\vmu_i\|> r_{\alpha,i}^{\sigma}.
    \end{align}
    Since $r_{\alpha,i}^{\sigma}$ is the $\alpha$-level (Def. \ref{def:alphalevel}), most of the volume of $k^\sigma_{\vmu_i}$ is contained within the $\alpha$-level
    \begin{align}
        \|\vx -\vmu_i\|> r_{\alpha,i}^{\sigma} \implies k^\sigma_{\vmu_i}(\vx) \leq \alpha.
    \end{align}
    To estimate the contribution of the terms
    $f_M^\sigma$ and $p^i(\vx)$ which include all kernels, we consider whether the chosen point $\vx$ belongs to the sphere of influence of other kernels. Consider two mutually exhaustive cases:
    \begin{itemize}
        \item \textit{Case 1}. The point $\vx$ is outside the sphere of influence of all kernels:
        \begin{equation}\label{eq:soi1}
            \vx \notin S_j^{\sigma}(\alpha) \quad \forall j
            \implies k_{\vmu_j}^{\sigma}(\vx) < \alpha
            \quad\forall j
        \end{equation}
        We can write the following decomposition of $f$:
        \begin{equation}
            f(x) = \sum_{j \in J^+} |a_j| \frac{k_j^2(\vx)}{\sum_{i \in J}  k_i(\vx)} - \sum_{j \in J^-} |a_j| \frac{k_j^2(\vx)}{\sum_{i \in J} k_i(\vx)}
        \end{equation}
        where $J^+$ and $J^-$ represent the subsets of indices in $J=(1, \dots, M)$  corresponding to positive and negative amplitudes respectively:
        \begin{align}
            J^+ &= \{j \in J \,|\, a_j>0\}\\
            J^- &= \{j \in J \,|\, a_j<0\}.
        \end{align}
        We can write the following bound on $f$:
        \begin{equation}
            - \sum_{j \in J^-} |a_j| \frac{k_j^2(\vx)}{\sum_{i \in J} k_i(\vx)} \le f \le \sum_{j \in J^+} |a_j| \frac{k_j^2(\vx)}{\sum_{i \in J}  k_i(\vx)}.
        \end{equation}
        Using Eq.~\ref{eq:soi1} and the fact that $p_j\le1 \quad \forall j \in J$ this bound can be further developed in
        \begin{equation}\label{eq:soi2}
            - \alpha\sum_{j \in J^-} |a_j| \le f(\vx) \le \alpha\sum_{j \in J^+} |a_j|.
        \end{equation}
        Eq.~\ref{eq:soi2} can be used to bound the expression $[2a_ik_i(\vx)-f(\vx)]$:
        \begin{equation}
            2a_ik_i(\vx) - \alpha\sum_{j \in J^+} |a_j| \le 2a_ik_i(\vx) - f(\vx) \le 2a_ik_i(\vx) + \alpha\sum_{j \in J^-} |a_j|.
        \end{equation}
        and because $k_i(\vx)<\alpha$ we can further develop into two cases:
        \begin{align}
            - \alpha\sum_{j \in J^+} |a_j| \le 2a_ik_i(\vx) - f(\vx) \le \alpha(2|a_i| + \sum_{j \in J^-} |a_j| )& \quad\text{if}\quad a_i>0\\
            - \alpha(-2|a_i|+ \sum_{j \in J^+} |a_j|)\le 2a_ik_i(\vx) - f(\vx) \le \alpha\sum_{j \in J^-} |a_j| & \quad\text{if}\quad a_i<0
        \end{align}
        Since $p^i\leq 1$ is bounded, this result implies that,
        \begin{align}
            | A_i |  = \left|\frac{1}{\sigma^2}\cdot \left[2\,a_i k_{i}(\vx) - f(\vx)\right] \cdot p^i\right| \leq \eta \alpha
        \end{align}
        where $\eta = \frac{1}{\sigma^2} \max \Big( \left|\sum_{j \in J^-} |a_j| \right|, \left| -2|a_i|+ \sum_{j \in J^+} |a_j| \right| \Big)$. Choosing $\epsilon_{\alpha}=\eta \alpha$ gives us the desired result for Case 1.
        \item \textit{Case 2}. The point $\vx$ lies in the sphere of influence of at least one other kernel $j \neq i$.\\
        In this case we can decompose $f$ as follows:
        \begin{equation}
            f(\vx)= a_j \frac{k_j^2(\vx)}{\sum_{i \in J}  k_i(\vx)} + \sum_{k\neq j,\,k \in J^+,} |a_k| \frac{k_k^2(\vx)}{\sum_{i \in J}  k_i(\vx)} - \sum_{k\neq j,\,k \in J^-} |a_k| \frac{k_k^2(\vx)}{\sum_{i \in J} k_i(\vx)}
        \end{equation}
        Using the arguments in Case 1 for all $k\neq j$ we can write the following bound:
        \begin{equation}
            2a_ik_i(\vx) - \alpha\sum_{k\neq j,\,k \in J^+} |a_k| \le 2a_ik_i(\vx) - f(\vx) + a_j \frac{k_j^2}{\sum_{i \in J}  k_i} \le 2a_ik_i(\vx) + \alpha\sum_{k\neq j,\,k \in J^-} |a_k|.
        \end{equation}
        In this case, we can distinguish between four cases depending on $a_i$ and $a_j$ signs:
        \begin{align}
            -|a_j| \frac{k_j^2(\vx)}{\sum_{i \in J}  k_i(\vx)} - \alpha\sum_{k\neq j,\,k \in J^+} |a_k| \le 2a_ik_i(\vx) - f(\vx) \le \alpha\left(2|a_i| + \sum_{k\neq j,\,k \in J^-} |a_k| \right)& \quad\text{if}\quad a_i, a_j>0\\
            - \alpha\sum_{k\neq j,\,k \in J^+} |a_k| \le 2a_ik_i(\vx) - f(\vx) \le \alpha\left(2|a_i| + \sum_{k\neq j,\,k \in J^-} |a_k| \right)+|a_j| \frac{k_j^2(\vx)}{\sum_{i \in J}  k_i(\vx)} & \quad\text{if}\quad a_i>0, a_j<0 \\
            -|a_j| \frac{k_j^2(\vx)}{\sum_{i \in J}  k_i(\vx)} - \alpha\left(-2|a_i|+ \sum_{k\neq j,\,k \in J^+} |a_k|\right)\le 2a_ik_i(\vx) - f(\vx) \le \alpha\sum_{k\neq j,\,k \in J^-} |a_k| & \quad\text{if}\quad a_i<0, a_j>0 \\
            - \alpha\left(\sum_{k\neq j,\,k \in J^+} |a_k|-2|a_i|\right)\le 2a_ik_i(\vx) - f(\vx) \le \alpha\sum_{k\neq j,\,k \in J^-} |a_k|+|a_j| \frac{k_j^2(\vx)}{\sum_{i \in J}  k_i(\vx)}  & \quad\text{if}\quad a_i, a_j<0.
        \end{align}
        Due to the term $|a_j| \frac{k_j^2}{\sum_{i \in J}  k_i}$, the quantity $2a_ik_i(\vx) - f(\vx)$ remains finite as $\alpha$ tends to zero:
        \begin{equation}
            -|a_j| \frac{k_j^2(\vx)}{\sum_{i \in J}  k_i(\vx)}\le 2a_ik_i(\vx) - f(\vx) \le |a_j| \frac{k_j^2(\vx)}{\sum_{i \in J}  k_i(\vx)} \quad \text{for} \quad \alpha\ll1 .
        \end{equation}
        We now study the ${\rm SoftMax}$ activation $p^i(\vx)=\frac{k_i(\vx)}{\sum_{k\in J}k_k(\vx)}$, contributing to $A_i(\vx)$.
        The numerator, $k_i(\vx)$, is less than $\alpha$ since $\vx\notin S_i(\alpha)$.
        The denominator can be decomposed as follows:
        \begin{equation}
            \sum_{k\in J}k_k(\vx) = \sum_{k\neq j,\, k\in J}k_k(\vx)+k_j(\vx).
        \end{equation}
        Since $k_k(\vx)>0$ by definition and $k_j(\vx)>\alpha$ for $\vx \in S_j(\alpha)$, we can write the following lower bound for the denominator:
                \begin{equation}
            \sum_{k\in J}k_k(\vx) = \sum_{k\neq j,\, k\in J}k_k(\vx)+k_j(\vx)> k_j(\vx).
        \end{equation}
        which translates into the following upper bound for $p^i(\vx)$:
        \begin{equation}
            0<p^i(\vx) < \frac{\alpha}{k_j(\vx)}.
        \end{equation}
        All together:
        \begin{align}
            &-|a_j| \frac{k_j^2(\vx)}{\sum_{i \in J}  k_i(\vx)} \cdot \frac{\alpha}{k_j(\vx)} \le \big(2a_ik_i(\vx) - f(\vx)\big)\cdot p^i(\vx) \le |a_j| \frac{k_j^2(\vx)}{\sum_{i \in J}  k_i(\vx)}\cdot \frac{\alpha}{k_j(\vx)}\\
            \implies&-|a_j| \frac{k_j(\vx)}{\sum_{i \in J}  k_i(\vx)} \cdot \alpha \le \big(2a_ik_i(\vx) - f(\vx)\big)\cdot p^i(\vx) \le |a_j| \frac{k_j(\vx)}{\sum_{i \in J}  k_i(\vx)}\cdot \alpha
        \end{align}
        which again implies that,
        \begin{align}
            | A_i |  = \left|\frac{1}{\sigma^2}\cdot \left[2\,a_i k_{i}(\vx) - f(\vx)\right] \cdot p^i\right| \leq \eta \alpha
        \end{align}
        where $\eta =|a_j| \frac{k_j(\vx)}{\sum_{i \in J}  k_i(\vx)} $. Using $\epsilon_\alpha = \eta \alpha$ again gives us the desired result for Case 2.
    \end{itemize}
    Thus, in either case, the gradient $\partial_{\vmu_i}f$ is small if $\vx \notin S_i^{\sigma}(\alpha)$ and $\alpha\ll1$, which completes the proof.
\end{proof}
Therefore, only training points within the sphere of influence of a kernel significantly affect the dynamics of its mean. We induce localization of kernels by scale annealing, i.e., we gradually decrease the value of the scale $\sigma$ so that kernels become localized, which is formalized as follows:
\begin{lemma}[Scale annealing leads to localization] Monotonically annealing the kernels' scale $\sigma$, shrinks the size of the sphere of influence:
$$\sigma_t < \sigma_{t-1} \implies \text{Vol}\; S^{\sigma_t}_i(\alpha) < \text{Vol}\; S^{\sigma_{t-1}}_i (\alpha),$$
resulting in a decrease in the number of data points that influence the dynamics of each kernel $i$.
\end{lemma}\label{lemma:annealing}
\begin{proof}
    The proof follows trivially from the sphere of influence being restricted to Euclidean balls $S_{i}^\sigma (\alpha)$ around $\vmu_i$ with a radius that is proportional to $\sigma$. Decreasing $\sigma$ leads to a shrinkage of the Euclidean balls, and therefore a shrinkage of the sphere of influence.
\end{proof}
\begin{remark}[Why fluctuations rarely capture the dynamics]\label{rem:fluctuations}
In finite samples, balanced regions are balanced only in expectation, and Proposition~\ref{prop:bulk-zero-SI} leaves open whether a kernel can be captured by a local fluctuation rather than a genuine anomaly.
A scaling argument makes this unlikely.
The force on a kernel (Lemma~\ref{lemma:pushpull}) aggregates the contributions of the $n_\sigma$ training points inside its Sphere of Influence: incoherent contributions cancel to $O(\sqrt{n_\sigma})$, so structure on scales below $\sigma$ is smoothed away, and the only fluctuations that can steer the kernel are those coherent at scale $\sigma$---whose typical amplitude corresponds to a local significance of order one under $H_0$.
A genuine anomaly of $N_S$ points instead contributes coherently, with local significance $N_S/\sqrt{n_\sigma}$ at the current resolution.
Annealing scans $\sigma$ across scales, so the anomaly is encountered near the resolution at which this significance is maximal, and the ${\rm SoftMax}$-induced repulsion prevents the ensemble from collapsing onto a single, possibly spurious, region (Figure~7 of the main text).
This is a scaling heuristic rather than a theorem; the experiments of Section~2 of the main text establish the resulting global behavior, and we note that capture by a fluctuation would in any case affect power only, since the same dynamics occurs in the anomaly-free calibration.
\end{remark}
As a consequence of annealing-induced localization, and consistently with Remark~\ref{rem:fluctuations}, we empirically observe that kernel means in \textsc{SparKer} enter a regime where a single data point---typically an anomalous one---lies within the sphere of influence.
In this regime, kernel mean $\vmu_i$ converges to the data point:

\begin{theorem}[Local convergence]
    If $S_i^{\sigma_t}$ contains only one data point $\vx^*\in \data$, and kernel $i$ dominates other kernels such that the ${\rm SoftMax}$ output $p^i \approx 1$ (i.e., the vector is nearly one-hot), then $\vmu_i$ converges radially towards $\vx^*$, getting $\epsilon$-close to $\vx^*$ in $O(\log 1/\epsilon)$ further steps.
\end{theorem}
\begin{proof}
    From Eq.~\ref{eq:app_grad_mu}, we know that the dynamics of $\vmu_i$ will be radially toward the point $\vx^*$. Therefore, the dynamics becomes one-dimensional, along the line joining $\vmu_i$ (initial value) and $\vx^*$. If the point $\vx^*\in \data$, and the probability vector is one-hot, the gradient of NP loss can be simplified as:
    \begin{align}
        \partial_\mu L_{NP} &\approx -f_t(\vx^*)\left(\frac{\vx^*-\vmu_t}{\sigma^2} \right).
    \end{align}
Plugging this into the gradient descent equation gives us,
\begin{align}
    \vmu_{t+1} &\approx \vmu_t + \eta \; f_t(\vx^*)\left( \frac{\vx^*-\vmu_t}{\sigma^2} \right), \\
    \text{where } f_t (\vx^*) &= \exp \left( - \frac{\|\vx^*-\vmu_t\|^2}{2\sigma^2} \right),
\end{align}
absorbing the (positive) kernel amplitude into the learning rate $\eta$.
Let $\delta_t = \|\vmu_t-\vx^* \|$. Subtracting $\vx^*$ from both sides of the above equation and writing everything in terms of $\delta_t$ results in:
\begin{align}
    \delta_{t+1} &\approx \delta_t - \eta \frac{\delta_t}{\sigma^2}  \exp \left( -\frac{\delta_t^2}{2 \sigma^2} \right) .
\end{align}
If $\delta_t/ \sigma \ll 1$ (which can be achieved by initializing $\sigma$ to a large value), the exponent can be approximated using the leading term in the Taylor series expansion as $\exp \left( -\frac{\delta_t^2}{2\sigma^2} \right) \approx 1 - \frac{\delta_t^2}{2\sigma^2} $, which gives us:
\begin{align}
    \delta_{t+1} &\approx \delta_t - \eta \frac{\delta_t}{\sigma^2} \left( 1 - \frac{\delta_t^2}{2\sigma^2} \right), \\
    &\approx \delta_t (1- \eta / \sigma^2),
\end{align}
by retaining only the first order term since $\delta_t/\sigma \ll1$. Iterating the above over time yields,
\begin{align}
    \delta_t \approx \delta_0 \left(1-\eta /\sigma^2 \right)^t.
\end{align}
This implies the following convergence result for $\delta_t$:
\begin{align}
        \text{for } \epsilon >0, \; \delta_t \leq \epsilon\; \forall t \geq \frac{\log (\delta_0/\epsilon)}{\log \left( \frac{1}{1-\eta /\sigma^2} \right)}.
\end{align}
Since $\delta_t= \|\vmu_t - \vx^*\|$, the above result indicates that $\vmu_t$ converges $\epsilon-$close to $\vx^*$ in $O(\log 1/\epsilon)$ steps of gradient descent (starting from when $\vx^*$ is the only point in the sphere of influence of kernel $i$). This completes the proof.
\end{proof}
\section{Numerical experiments and benchmark comparisons}\label{subsec:toys}
In-distribution detection can pose different levels of difficulty based on the degree of separability between the anomalies and the reference data, and based on the fraction of signal present in the dataset.
In regions of high density of the reference, small signal injections are hard to detect because the log-density-ratio tends to be very close to zero, whereas in the tails of the reference distribution the anomalous pattern has to compete with statistical fluctuations.
We explore the sensitivity of \textsc{SparKer} as a function of the anomaly separability with a series of one- and two-dimensional toy experiments. In the experiments, a localized anomalous signal is injected at different levels of overlap with the reference, either in the bulk, the tail, or out of the distribution. The amount of signal injected is tuned to keep the detection hardness approximately constant in the various scenarios.
In addition, we test \textsc{SparKer} on delocalized signals of two kinds: rare tail excesses, which remain within the model's target regime despite having no compact structure, and global scale distortions mimicking domain shifts, for which the sparsity assumption is by construction not optimal.
Finally, we consider a realistic, five-dimensional physics benchmark that contains both a localized (resonant) and a delocalized (non-resonant) signal, allowing us to probe the two extremes within a single dataset.
On these benchmarks we compare \textsc{SparKer} against a broad set of anomaly-detection paradigms, described in Section~\ref{app:baselines}.
Before presenting the comparison, a methodological caveat is in order. Anomaly detection is not a single task but a family of tasks (Fig.~1(a) of the main text represents this variety along the separability and rarity axes), and most methods are implicitly optimized for a particular corner of it: distance-based detectors target out-of-support or extreme-tail outliers, density-based scorers target low-likelihood points. A method tuned for one corner can appear ineffective in another purely by design, so a raw head-to-head across paradigms must be interpreted with this in mind: the comparison below is designed to delineate the domain of applicability of each paradigm, not to claim a single universal winner.
Since several established anomaly-detection methods are point-wise scorers rather than collective tests, we structure the comparison in two rungs, placing every method on a common footing in each: (1) \emph{two-sample test} comparisons, addressing the problem of collective anomaly detection as framed in the main text (Section~\ref{subapp:regimes}); and (2) \emph{point-wise} comparisons, in which the per-point anomaly score of \textsc{SparKer} is used to tag individual anomalous points, so that \textsc{SparKer} can be evaluated against the baselines that are point-wise detectors by design (Section~\ref{subapp:pointwise-comparison}).
In both rungs the operating point is fixed at a common $5\%$ false-positive rate applied uniformly to every method; the figure of merit is the detection power at $5\%$ false-positive rate for the two-sample rung, and the per-point true-positive rate at $5\%$ per-point false-positive rate for the point-wise rung.
The benchmarks are detailed in the following.
\subsection{Datasets}
\subsubsection{1D exponential}\label{subapp:1d}
The first benchmark we consider is a one-dimensional toy model introduced in Ref.~\cite{DAgnolo:2018cun}.
The data distribution under the reference model is a unit scale exponential.
The data under the alternative hypotheses are composed of a mixture of reference-like data and signal-like data.
To test the model's ability under different separability regimes, we consider localized signals of Gaussian form with scale $0.16$ and different locations, as well as a signal manifesting as an excess in the tail proportional to $x^2e^{-x}$.
The number of expected events in the dataset is ${\rm N(R)}= 2000$.
The reference sample $\reference$ is taken $10$ times larger than the expected number of events in the dataset ${\rm N(R)}$, $\Nreference=20\,000$.
The details on the signal injected are given in Table~\ref{tab:1d-dataset}; Top row of Figure~\ref{fig:1d-groundtruth} reports the data distribution under the reference hypothesis (dashed line) and different alternatives with signal (solid lines) in log scale; bottom row reports the corresponding log-density ratios. The two rows expose complementary aspects of the detection difficulty: the bulk bump carries the largest injected fraction yet produces the smallest log-ratio, as the signal sits where the reference density is highest, while the tail bumps produce large log-ratios from only a handful of events, which must be distinguished from statistical fluctuations in sparsely populated regions. The signal fractions are tuned to keep the detection hardness approximately constant across benchmarks.
\begin{table}[h]
    \centering
    \begin{tabular}{llccc}
         &label&\textbf{$\mu$}&  \textbf{$\sigma$}& N(S)/N(R) \\
         \toprule
         \textbf{Gaussian bumps}
         &bulk   &  1.6  &  0.16& $4.5\cdot 10^{-2}$\\
         &tail   &  6.4  &  0.16& $5\cdot 10^{-3}$\\
         &extreme tail   &  9  &  0.16& $2.5\cdot 10^{-3}$\\
         \midrule
         \textbf{Tail excess}&excess &       &      & $4.5\cdot 10^{-2}$\\
         \bottomrule
    \end{tabular}
    \caption{Summary of the 1D signal benchmarks. }
    \label{tab:1d-dataset}
\end{table}
\begin{figure}[h]
    \centering
    \includegraphics[width=1\linewidth]{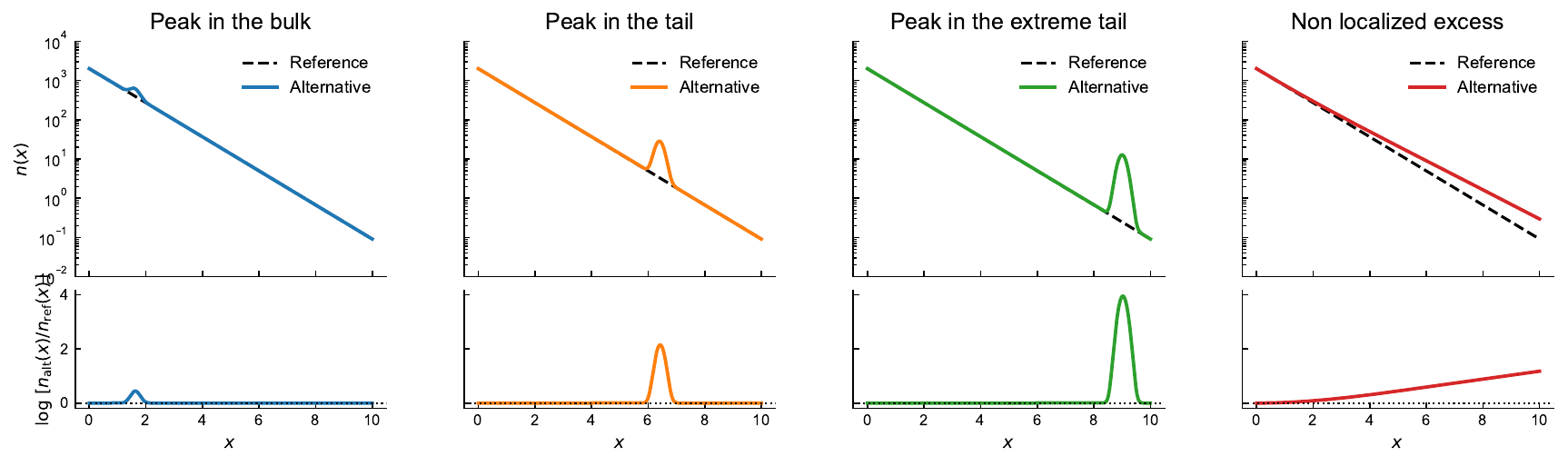}
    \caption{\textbf{1D benchmark ground truth.} Top row: number densities of the reference hypothesis (unit-scale exponential, dashed line) and of the four alternative hypotheses (solid lines) on a logarithmic scale, for the signal benchmarks of Table~\ref{tab:1d-dataset}: a Gaussian bump in the bulk, in the tail, and in the extreme tail, and a non-localized excess $\propto x^2 e^{-x}$. Bottom row: log-ratio between the alternative and the reference densities, on a common scale across panels. The two rows expose complementary aspects of the detection difficulty: the bulk bump carries the largest injected fraction yet produces the smallest log-ratio, as the signal sits where the reference density is highest, while the tail bumps produce large log-ratios from only a handful of events, which must be distinguished from statistical fluctuations in sparsely populated regions. The signal fractions are tuned to keep the detection hardness approximately constant across benchmarks.}
    \label{fig:1d-groundtruth}
\end{figure}
\subsubsection{2D mixture of Gaussians}\label{subapp:2d}
The second benchmark that we consider is inspired by~\cite{10.5555/3666122.3669406}. The ground truth is represented by a mixture of four 2D Gaussian models with covariance matrix $\Sigma = 4 \cdot\mathbb{I}^{2\times 2}$, located at the corners of a square of length 16.
We consider four anomalous scenarios. Two of them consist of a local new cluster appearing either in the bulk of the bottom left Gaussian or in the extreme tail of all of them (i.e. in the middle of the square identified by the four Gaussian models' locations). The others, inspired by~\cite{10.5555/3666122.3669406}, are non-local distortions obtained by varying the scale of the bottom left Gaussian. While the latter are shape anomalies that do not modify the expected number of data points, the new anomalous clusters are an injection on top of a set of background-like data. The number of expected events in the dataset ${\rm N(R)}= 4000$.
The reference sample $\reference$ is taken $5$ times larger than the expected number of events in the dataset ${\rm N(R)}$, $\Nreference=20\,000$. To make the problem hard, we consider very low signal injections (1\% or lower). Details on the signal injection are reported in Table~\ref{tab:2d-dataset}; Figure~\ref{fig:2d-groundtruth} shows the log-density-ratio between the reference distribution and the four alternatives. The out-of-distribution cluster (second panel) produces a large log-ratio in the empty central region, and the two scale distortions (third and fourth panels) deform the whole dominance region of the modified component; in contrast, the over-density in the bulk (first panel) is barely visible on this scale, as the signal sits where the reference density is highest and the log-ratio remains close to zero---making it the hardest benchmark of the four, and the closest to the rare, in-distribution regime targeted in this work.
\begin{table}[h]
    \centering
    \begin{tabular}{llccc}
         &label&\textbf{$\mu$}&  \textbf{$\sigma$}& N(S)/N(R)\\
         \toprule
         \textbf{Gaussian bumps}
         &bulk   &  (-6, -6)  &  1& $1\cdot 10^{-2}$\\
         &ood   &  (0, 0)  &  1& $1\cdot 10^{-3}$\\
         \midrule
         \textbf{Scale distortion}  &smearing &    &   2.2   & \\
                                    &squeezing &   &   1.8   & \\
         \bottomrule
    \end{tabular}
    \caption{Summary of the 2D signal benchmarks. }
    \label{tab:2d-dataset}
\end{table}
\begin{figure}[h]
    \centering
    \includegraphics[width=1\linewidth]{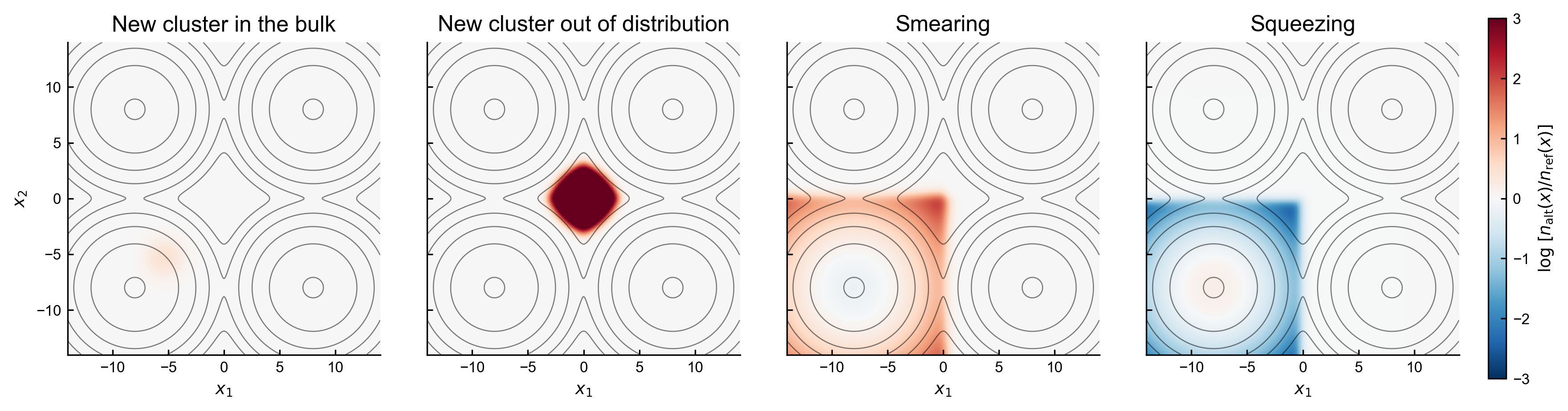}
    \caption{\textbf{2D benchmark ground truth.} Log-likelihood ratio between the alternative and the reference hypotheses in the four 2D benchmarks considered in this work, shown on a common color scale (clipped at $\pm5$); the black contours are iso-density level curves of the reference mixture, delimiting its bulk and tails. The out-of-distribution cluster (second panel) produces a large log-ratio in the empty central region, and the two scale distortions (third and fourth panels) deform the whole dominance region of the modified component; in contrast, the over-density in the bulk (first panel) is barely visible on this scale, as the signal sits where the reference density is highest and the log-ratio remains close to zero---making it the hardest benchmark of the four, and the closest to the rare, in-distribution regime targeted in this work.}
    \label{fig:2d-groundtruth}
\end{figure}
\subsubsection{5D di-muon benchmark: resonant and non-resonant signals}\label{subapp:mumu}
To assess \textsc{SparKer} on a realistic, higher-dimensional problem containing signals of opposite geometry, we consider the \textsc{Dimuon-5D} benchmark introduced in Refs.~\cite{DAgnolo:2019vbw,grosso_2021_4442665} and further studied in Refs.~\cite{Letizia:2022xbe,Grosso:2023scl}.
The reference model describes Standard-Model Drell--Yan production, $pp\to\mu^+\mu^-$, simulated at $\sqrt{s}=13$~TeV.
Each event is described by five kinematic features: the transverse momenta $p_{T}$ and pseudorapidities $\eta$ of the two muons, and their azimuthal separation $\Delta\phi$.
The di-muon invariant mass $m_{\ell\ell}$, which is the most discriminating observable for the signals considered, is \emph{deliberately excluded} from the inputs, reproducing the realistic situation in which the relevant discriminant is not known a priori.
We inject two families of signals on top of this background.
The first is a resonant $Z'\to\mu^+\mu^-$ excess, generated for three masses $m_{Z'}=200,\,300,\,600$~GeV, which manifests as a localized bump in $m_{\ell\ell}$.
The second is a non-resonant deviation induced by a dimension-six four-fermion contact interaction (``EFT'', $c_W/\Lambda^2$ with $\Lambda=1$~TeV), whose effect is spread over the whole spectrum and grows smoothly towards high $m_{\ell\ell}$, with no peak at any mass.
Together, these two signals span the localized and delocalized extremes of the in-distribution regime within a single, physically-motivated dataset.
According to Ref.~\cite{DAgnolo:2019vbw}, we fix the expected number of events to $N(\mathcal{R})=20000$ and the reference sample size to $\Nreference=100000$. Additional physics cuts applied according to Ref.~\cite{DAgnolo:2019vbw} reduce the final sample sizes by a factor of 2.5. Details on the fraction of signal injected are given in Table~\ref{tab:5d-dataset}.
\begin{table}[h]
    \centering
    \begin{tabular}{llc}
         &mass& N(S)/N(R)*\\
         \toprule
         \textbf{resonant signal}
         &Z' ($m=200$ GeV)   &  $7.5\cdot 10^{-3}$\\
         &Z' ($m=300$ GeV)   &  $3.75\cdot 10^{-3}$\\
         &Z' ($m=600$ GeV)   &  $1.25\cdot 10^{-3}$\\
         \midrule
         \textbf{non-resonant signal}
         &$c_W=1.5 \,\text{TeV}^{-2}$ & $2.65\cdot 10^{-3}$\\
         \midrule
         {\footnotesize *after physics cuts applied}\\
    \end{tabular}
    
    \caption{Summary of the 5D di-muon signal benchmarks. }
    \label{tab:5d-dataset}
\end{table}
\subsection{Baseline methods}\label{app:baselines}
We consider the following test statistics, grouped by paradigm. Statistics are written so that a larger value indicates a more anomalous sample;
$n=|\mathcal{D}|$, $m=|\mathcal{R}|$, $N=n+m$, and $w_\mathcal{R}=n/m$.
Wherever a test depends on a hyperparameter (the neighborhood size $k$, the number of prototypes $M$), we run it at four values and additionally report an aggregated, calibrated test combining the per-value $p$-values with the same rule used in Algorithm~1 of the main text for the kernel bandwidths:
\begin{equation}\label{eq:aggregation}
    T({\cal D}) = -\frac{1}{2}\min_{i\in[1,\dots,K]}\log p_{i}({\cal D})
    -\frac{1}{2K}\sum_{i=1}^{K} \log p_{i}({\cal D}),
\end{equation}
where $i$ runs over the $K$ hyperparameter values.
\\
\paragraph{Distance-based tests.}
The tests in this category are functions of pairwise Euclidean distances alone. The first four share a common skeleton: they contrast within-sample with cross-sample interactions, and differ in the interaction attached to a pair of points. The energy distance uses the metric itself; the MMD a positive-definite kernel of it (the two are formally instances of a single family, see below); the $k$-nearest-neighbor two-sample and Friedman--Rafsky tests replace the interaction magnitude with the adjacency of a geometric graph built from the distances---the $k$-NN graph and the minimum spanning tree respectively---thus retaining only the local ordering of distances. The first pair accumulates evidence over all pairs and is sensitive to weak, diffuse imbalances; the graph-based pair is unchanged by any monotone transformation of the metric, but for the same reason it reacts mostly when the two samples separate in space, and loses power when they overlap and differ only by a mild local imbalance, as in the rare, in-distribution regime targeted here. The last member of the category, the $k$-NN distance, is a one-sample point-wise scorer rather than a two-sample test: it scores each point of $\mathcal{D}$ by its distance to the reference sample.
\begin{itemize}
  \item \emph{Energy distance}~\cite{szekely2013energy,baringhaus2004,sejdinovic2013}.
  The energy distance is one member of a general family of distance-based discrepancies. For a semimetric $\rho$ of negative type on the data space and two distributions $\mu$ and $\nu$, the generalized energy distance is
  \begin{equation}
    \mathcal{E}_{\rho}(\mu,\nu)=2\,\mathbb{E}\,\rho(X,Y)
      -\mathbb{E}\,\rho(X,X')-\mathbb{E}\,\rho(Y,Y'),
  \end{equation}
  with $X,X'\sim\mu$ and $Y,Y'\sim\nu$ mutually independent~\cite{lyons2013,sejdinovic2013}. The name is borrowed from physics: reading $\rho$ as a pairwise interaction, the three terms are the average cross-sample interaction and the two internal self-energies, so $\mathcal{E}_{\rho}$ measures the \emph{excess} interaction between the samples---how much their separation exceeds what their own internal variability would produce. For $\rho$ of strong negative type one has $\mathcal{E}_{\rho}\geq 0$, with equality if and only if $\mu=\nu$, so the statistic is sensitive to arbitrary distributional differences without any density estimation.
  Taking the Euclidean distance $\rho(x,y)=\|x-y\|$ recovers the classical energy distance~\cite{szekely2013energy,baringhaus2004}, which we adopt here, estimated by
  \begin{equation}
    E=\frac{2}{nm}\sum_{x\in\mathcal{D}}\sum_{y\in\mathcal{R}}\|x-y\|
     -\frac{1}{n^{2}}\sum_{x,x'\in\mathcal{D}}\|x-x'\|
     -\frac{1}{m^{2}}\sum_{y,y'\in\mathcal{R}}\|y-y'\|.
  \end{equation}
  Other choices of $\rho$ yield other members of the family; in particular, the semimetric induced by a positive-definite kernel yields the maximum mean discrepancy, described next.
\item \emph{Maximum mean discrepancy (MMD)}~\cite{gretton2012mmd}.
  The MMD measures the discrepancy between two distributions $\mu$ and $\nu$ as the distance between their mean embeddings in the reproducing-kernel Hilbert space of a kernel $k$,
  \begin{equation}
  \mathrm{MMD}^2_k(\mu,\nu)=\mathbb{E}[k(X,X')]+\mathbb{E}[k(Y,Y')]-2\,\mathbb{E}[k(X,Y)],
  \end{equation}
  with $X,X'\sim\mu$ and $Y,Y'\sim\nu$ mutually independent; it vanishes if and only if $\mu=\nu$ for a characteristic kernel. Substituting the kernel-induced semimetric $\rho_k(x,y)=k(x,x)+k(y,y)-2k(x,y)$ into the generalized energy distance above recovers $\mathrm{MMD}^2_k$ up to normalization, and conversely every kernel MMD is an energy distance for some negative-type semimetric~\cite{sejdinovic2013}: the two tests differ only in the choice of $\rho$. Given the samples, $\mathrm{MMD}^2_k$ is estimated by its empirical counterpart
  \begin{equation}
  \widehat{\mathrm{MMD}}^2_k=\frac{1}{n^2}\sum_{x,x'\in\mathcal{D}}k(x,x')+\frac{1}{m^2}\sum_{y,y'\in\mathcal{R}}k(y,y')-\frac{2}{nm}\sum_{x\in\mathcal{D}}\sum_{y\in\mathcal{R}}k(x,y).
  \end{equation}
  For a Gaussian kernel the test power depends strongly on the bandwidth, and no single scale is optimal across signals; we therefore aggregate the statistic over a range of bandwidths using the MMD-Fuse combination~\cite{10.5555/3666122.3669406}, a soft-maximum (log-sum-exp) over the per-bandwidth statistics closely related to the aggregated MMD test~\cite{schrab2023mmdagg}, which preserves power without splitting the data or committing to a scale a priori. We choose the list of bandwidths by computing the distribution of pair-wise distances between reference points that we empirically estimate from 2000 points taken from the reference set. The values of the bandwidths that we choose correspond to the 1\%, 25\%, 50\%, 75\% and 99\% quantiles to which we add twice the 99\% quantile to ensure sensitivity to non-localized effects. Because the exact estimator costs $\mathcal{O}\!\left((n+m)^2 d\right)$ time and memory---prohibitive at the sample sizes considered here---we use the Nystr\"om-approximated MMD~\cite{chatalic2026scalable}, which replaces the full kernel matrix by a low-rank projection onto a small set of landmark points, reducing the cost to near-linear while retaining minimax-optimal detection power. The MMD is intrinsically a collective two-sample statistic and requires no per-point lifting.
  \item \emph{$k$-nearest-neighbor two-sample test}~\cite{schilling1986,henze1988}.
  The two samples are pooled into $\mathcal{Z}=\mathcal{D}\cup\mathcal{R}$ of size $N=n+m$, each point $z_i$ retaining its sample label $y_i$ ($y_i=1$ if $z_i\in\mathcal{D}$ and $0$ otherwise).
  For every point one finds the set $\mathcal{N}_k(i)$ of its $k$ nearest neighbors in the pooled sample and counts the \emph{coincidences}, the neighbors that carry the same label; the standard (unweighted) statistic is the overall proportion of same-sample neighbor pairs,
  \begin{equation}
    s_{\rm KNN}^{k}=\frac{1}{Nk}\sum_{i=1}^{N}\sum_{j\in\mathcal{N}_k(i)}\mathbb{I}\!\left[\,y_{j}=y_i\,\right].
  \end{equation}
  The rationale is geometric and local: if $\mathcal{D}$ and $\mathcal{R}$ follow the same distribution the two labels are effectively interleaved in space, so the coincidence rate fluctuates around the chance level $p_0=[n(n-1)+m(m-1)]/[N(N-1)]$ fixed by the sample sizes; any distributional difference causes points of the same sample to cluster together and drives the coincidence rate above $p_0$, so large values of $s_{\rm KNN}^{k}$ are evidence against $H_0:\mu=\nu$.
  In our setting the reference is significantly larger than the data by design ($n/m<1$), and this imbalance degrades the performance of the test, which is dominated by reference--reference coincidences, carrying no information about whether $\mathcal{D}$ differs from $\mathcal{R}$ and merely diluting the informative data. We therefore consider a \emph{sample-size-balanced} variant that uses only the known ratio $n/m$ (never the anomaly fraction) to equalize the two coincidence types.
  Separating the per-label counts $S_{\mathcal{D}\mathcal{D}}=\sum_{i:\,y_i=1}\sum_{j\in\mathcal{N}_k(i)}\mathbb{I}[y_j=1]$ and $S_{\mathcal{R}\mathcal{R}}=\sum_{i:\,y_i=0}\sum_{j\in\mathcal{N}_k(i)}\mathbb{I}[y_j=0]$, with null means $\mathbb{E}_0[S_{\mathcal{D}\mathcal{D}}]=nk(n-1)/(N-1)$ and $\mathbb{E}_0[S_{\mathcal{R}\mathcal{R}}]=mk(m-1)/(N-1)$, we standardize each by its null fluctuations and sum them,
  \begin{equation} s_{\rm KNN,\mathrm{bal}}^{k}=\frac{S_{\mathcal{D}\mathcal{D}}-\mathbb{E}_0[S_{\mathcal{D}\mathcal{D}}]}{\mathrm{sd}_0[S_{\mathcal{D}\mathcal{D}}]}+\frac{S_{\mathcal{R}\mathcal{R}}-\mathbb{E}_0[S_{\mathcal{R}\mathcal{R}}]}{\mathrm{sd}_0[S_{\mathcal{R}\mathcal{R}}]},
  \end{equation}
  so that the data-side coincidences are no longer swamped by the far more numerous reference-side ones.
  This correction uses only the design ratio $n/m$ and leaves the test signal-agnostic. 
  Note that the nearest-neighbor tests depend on the neighborhood size $k$. We run the balanced test at four different values of $k$ ($k=5, 20, 100, 200$) and we aggregate them by combining the p-values according to Eq.~\ref{eq:aggregation}.
    \item \emph{Friedman--Rafsky test}~\cite{friedman1979rafsky}. A multivariate generalization of the Wald--Wolfowitz runs test~\cite{waldwolfowitz1940}: one builds the Euclidean minimum spanning tree (MST) on the pooled sample and counts the edges joining points from different samples. Let $(u,v)$ be a generic edge of the MST that connects points $u$ and $v$, the statistic is defined as 
    \begin{equation}
    R=\sum_{(u,v)\in\mathrm{MST}}\mathbb{I}[\mathrm{label}(u)\neq\mathrm{label}(v)].
    \end{equation}
    Few cross-sample edges indicate that the samples segregate in space, so we take
    \begin{equation}
        s_{\rm FR}=-R
    \end{equation}
     as the statistic. In one dimension the MST reduces to the sorted chain and the test recovers the classical runs test. Computing the MST from the dense pairwise-distance matrix costs $\mathcal{O}(N^2 d)$ time and $\mathcal{O}(N^2)$ memory for $N=n+m$ points in $d$ dimensions, which becomes the bottleneck at large $N$; the memory can be reduced to $\mathcal{O}(N)$ by building the MST on a sparse $k$-nearest-neighbor graph instead, while in one dimension the sorted-chain construction requires only $\mathcal{O}(N\log N)$ time and $\mathcal{O}(N)$ memory.
    \item \emph{$k$-nearest-neighbor distance}~\cite{ramaswamy2000,angiulli2002}.
  Each point of $\mathcal{D}$ is scored by its distance to the $k$-th nearest
  neighbor in the reference sample,
  \begin{equation}
      d_k(x) = \min_{\substack{S\subset\mathcal{R},\,|S|=k}}\;\max_{y\in S}\,\|x-y\|,
  \end{equation}
  i.e.\ the radius of the smallest ball centered at $x$ containing $k$ reference
  points; a large $d_k(x)$ indicates a point poorly supported by the
  reference. This is the low-dimensional counterpart of the deep-feature
  $k$-NN detector used for out-of-distribution detection~\cite{sun2022knnood}.
  We lift the per-point score to a sample statistic by averaging over the
  data,
  \begin{equation}
      s^{k}_{\rm KNNdist}=\frac{1}{n}\sum_{x\in\mathcal{D}} d_k(x).
  \end{equation}
  As for the $k$-NN two-sample test, the statistic depends on the choice of
  $k$; we run it at $k=5,20,100,200$ and report an aggregated version
  combining the four $p$-values via Eq.~\ref{eq:aggregation}.
    \end{itemize}
\paragraph{Neyman--Pearson two-sample tests.}
Tests in this category share the construction of \textsc{SparKer}: a model
$f$ of the log-density-ratio is learned from the two samples, and the test
statistic is the in-sample Neyman--Pearson statistic $t=-2L_{\rm NP}[\hat f]$
(Eq.~4 of the main text).
\begin{itemize}
    \item \emph{Kernel-based Neyman--Pearson test (Falkon)}.
    We follow the implementation of Ref.~\cite{Letizia:2022xbe}, which exploits the Falkon library~\cite{falkonlibrary2020}: fast iterative solvers combined with a Nystr\"om approximation of the kernel matrix, projected onto a low-rank subspace defined by a small set of centers chosen randomly from the data and kept fixed during training. The model is trained by weighted kernel logistic regression (cross-entropy loss), a classification surrogate whose population minimizer is the log-density-ratio itself; the NP statistic is then computed from the trained model. \textsc{SparKer} instead minimizes the NP objective directly, so that training objective and test statistic coincide. For a fair comparison, we ensemble Falkon over multiple kernel bandwidths and combine the tests using the minimum $p$-value, as proposed in Ref.~\cite{Grosso:2024wjt}.
\end{itemize}

\paragraph{Density-ratio-based methods.}
Like the Neyman--Pearson tests above, the two methods in this category
estimate the density ratio $p_\mathcal{D}/p_\mathcal{R}$ directly, rather than
modeling the two densities separately. They differ in how the estimate is
turned into a test: the fitted divergence (Pearson-type for RuLSIF,
Kullback--Leibler for KLIEP) is used as a plug-in statistic, rather than the
likelihood ratio itself, and hyperparameters are selected by cross-validation rather than by aggregation over values.
\begin{itemize}
    \item \emph{RuLSIF (Relative unconstrained Least-Squares Importance Fitting)}~\cite{yamada2013rulsif,kanamori2009ulsif,sugiyama2012book}.
  RuLSIF models the \emph{$\alpha$-relative} density ratio
  \begin{equation}
  r_\alpha(x)=\frac{p_\mathcal{D}(x)}{\alpha\,p_\mathcal{D}(x)+(1-\alpha)\,p_\mathcal{R}(x)},
  \end{equation}
  which is bounded by $1/\alpha$ for $\alpha\in(0,1]$---avoiding the tail
  instability of the plain ratio where $p_\mathcal{R}$ is small---and reduces
  to $p_\mathcal{D}/p_\mathcal{R}$ at $\alpha=0$. The model is a linear
  expansion over Gaussian kernels centered on data points,
  $r_\theta(x)=\sum_\ell \theta_\ell\,k_\sigma(x,c_\ell)$, fitted by
  minimizing the $\ell_2$-regularized squared error to $r_\alpha$
  (equivalently, the $\alpha$-relative Pearson divergence). Dropping
  non-negativity constraints gives the closed-form solution
  $\hat\theta=(\hat H+\lambda I)^{-1}\hat h$, with
  $\hat H = \frac{\alpha}{n}\sum_{x\in\mathcal{D}}\vk(x)\vk(x)^{\top}
  +\frac{1-\alpha}{m}\sum_{x\in\mathcal{R}}\vk(x)\vk(x)^{\top}$ and
  $\hat h=\frac{1}{n}\sum_{x\in\mathcal{D}}\vk(x)$, where
  $\vk(x)=\big(k_\sigma(x,c_1),\dots\big)^{\top}$.
  The test statistic is the plug-in estimate of the $\alpha$-relative Pearson
  divergence,
  \begin{equation}
  s_{\rm RuLSIF} \;=\; \frac{1}{n}\sum_{x\in\mathcal{D}}\hat r(x)
  \;-\;\frac{\alpha}{2n}\sum_{x\in\mathcal{D}}\hat r(x)^2
  \;-\;\frac{1-\alpha}{2m}\sum_{x\in\mathcal{R}}\hat r(x)^2
  \;-\;\frac12 ,
  \end{equation}
  with the bandwidth $\sigma$ and regularization $\lambda$ selected by
  cross-validation of the fitting objective. We compute the RuLSIF test using the standard $\alpha=0.1$.
\item \emph{KLIEP (Kullback--Leibler importance estimation procedure)}~\cite{sugiyama2008kliep}.
  KLIEP uses the same kernel expansion
  $r_\theta(x)=\sum_\ell \theta_\ell\,k_\sigma(x,c_\ell)$ with $\theta\ge0$,
  but fits it under the Kullback--Leibler divergence between $p_\mathcal{D}$
  and the induced model $r_\theta\,p_\mathcal{R}$, i.e.\ solves the convex
  program
  \begin{equation}
  \max_{\theta\geq 0}\ \sum_{x\in\mathcal{D}}\log r_\theta(x)
  \quad\text{s.t.}\quad \frac{1}{m}\sum_{x\in\mathcal{R}} r_\theta(x)=1,
  \end{equation}
  by iterative gradient/projection steps, with $\sigma$ set by likelihood
  cross-validation. The test statistic is the plug-in KL estimate
  \begin{equation}
  s_{\rm KLIEP}=\frac{1}{n}\sum_{x\in\mathcal{D}}\log \hat r(x).
  \end{equation}
  Because the objective weighs $\log r_\theta$ on the data, KLIEP emphasizes
  regions of large ratio---the tails, where $p_\mathcal{R}$ is small---and its
  estimates are correspondingly noisier there than RuLSIF's, whose
  $\alpha$-relative target is bounded by $1/\alpha$.
\end{itemize}

\paragraph{Density-based methods (unsupervised).}
These methods share a common backbone: a flexible density $q$ is fit on the reference $\mathcal{R}$ and each point is scored by its negative log-likelihood $-\log q(x)$ under that model. For $q$ we use a kernel density estimator~\cite{parzen1962}; $q$ is an explicit generative model of the anomaly-free data, and its quality (and cost) degrades as the dimension grows, which is the main practical limitation of this family. The two methods differ in how the per-point scores are turned into a decision.
    \begin{itemize}
    \item \emph{Likelihood novelty score} [point-wise statistics]. The classical density-based approach flags low-likelihood points as anomalous~\cite{bishop1994novelty}: each $x\in\mathcal{D}$ is scored by $-\log q(x)$. We estimate the density from the reference set using \texttt{scipy.stats.gaussian\_kde} with Scott's-rule bandwidth. We aggregate the point-wise score into a collective statistic averaging the score over $\mathcal{D}$, obtaining a sample-level test:
    \begin{equation}
        s_{\rm lowLik} = - \frac{1}{n}\sum_{x\in \cal{D}}\log q(x)
    \end{equation}
      This score is designed based on the assumption that anomalies sit in low-density regions of $q$.
    \item \emph{Typicality test}~\cite{nalisnick2019typicality}.
      The intuition behind the likelihood novelty score can fail in high dimension: samples from $q$ do not concentrate where $q$ is largest but on a \emph{typical set} whose average negative log-likelihood equals the entropy $H[q]=\mathbb{E}_{x\sim q}[-\log q(x)]$, so genuinely in-distribution inputs need not have high likelihood and likelihood alone can misrank out-of-distribution inputs~\cite{nalisnick2019know}.
      The typicality test therefore compares the \emph{batch} average negative log-likelihood of $\mathcal{D}$ to this typical value rather than asking for high likelihood,
      \begin{equation}
        s_{\rm typ}=\Bigl|\,{-\tfrac{1}{n}\sum_{x\in\mathcal{D}}\log q(x)}\;-\;H[q]\,\Bigr|,
      \end{equation}
      flagging $\mathcal{D}$ as anomalous when its mean log-likelihood departs---in \emph{either} direction---from what a genuine reference batch would produce.
      It is thus a collective, batch-level test by construction. To estimate $q(x)$ and $H[q]$ independently we split the reference sample into two same-size subsets; we use the first part to estimate $q(x)$ with \texttt{scipy.stats.gaussian\_kde} with Scott's-rule bandwidth, and the second part to estimate $H[q]$ as the sample average of $-\log q$.\\
      By construction the test monitors a single scalar---the batch-average log-likelihood compared to the entropy $H[q]$---and is therefore expected to be powerful only against distortions that move this average, such as broad distributional shifts or anomalies producing anomalously high likelihood, which the plain novelty score of the previous item would instead miss.
      Conversely, it is expected to be insensitive to rare, localized, in-distribution over-densities: a small fraction of anomalous points lying where $q$ is already moderate leaves the mean log-likelihood essentially unchanged, and any excess that raises the likelihood of some points while lowering it for others can cancel in the average.
      The test thus assesses whether $\mathcal{D}$ is a typical sample of $q$ rather than whether $\mathcal{D}=\mathcal{R}$, and is not consistent against all alternatives; its power also degrades when the batch is small or when $q$ (and hence $H[q]$) is poorly estimated in high dimension.
\end{itemize}

\paragraph{Self-Organizing-Map (SOM) based methods (unsupervised).}
\begin{itemize}
    \item \emph{SOM quantization error} [point-wise statistics]. A self-organizing map~\cite{kohonen1982,kohonen1990,kohonenbook} is trained on
      $\mathcal{R}$, learning a set of prototypes $\{w_i\}_{i=1}^{M}$ through competitive
      learning under an annealed neighborhood. Each point of $\mathcal{D}$ is then
      scored by its quantization error ${\rm QE}(x)=\min_i\|x-w_i\|$, the distance to its
      best-matching prototype, which measures how poorly the reference-trained map
      represents it~\cite{munoz1998som}. We use the MiniSom
      implementation~\cite{vettigli2018minisom}.
      We lift the per-point score to a sample statistic by averaging over the data:
      \begin{equation}
          s_{\rm SOM-qe}^M = \frac{1}{n}\sum_{x \in \cal{D}}{\rm QE}(x)
      \end{equation}
    \item \emph{SOM binned test}. The SOM prototypes induce a Voronoi partition of
      the input space, providing a data-adaptive binning. Over these cells we perform a binned two-sample comparison of the counts $n_i$ in $\mathcal{D}$ against the reference expectation $\mu_i=w_\mathcal{R}\,m_i$, using the Poisson
      likelihood-ratio statistic~\cite{bakercousins1984},
      \begin{equation}
      s_{\rm SOM-\chi^2}^M=2\sum_{i=1}^M\big[n_i\log(n_i/\mu_i)-(n_i-\mu_i)\big]
      \end{equation}
      (or, equivalently, a Pearson $\chi^2$). Unlike the quantization-error variant, this is a genuine collective test that exploits both samples over the learned bins.
\end{itemize}
Note that SOM-based metrics depend on the number of prototypes $M$ used to train the SOM. This is a hyperparameter common to both tests that significantly impacts sensitivity, similarly to the kernel bandwidth for kernel-based methods like \textsc{SparKer}. To mitigate the impact of hyper-parameter choice, we proceed as we did for the bandwidth in kernel-based methods and compute an aggregated test combining p-values over $N_M$ values of $M$ ($M=5, 20, 100, 200$) using Eq.~\ref{eq:aggregation}. 
\paragraph{Other point-wise statistics (unsupervised).}
These methods are unsupervised scorers, fit on the reference $\mathcal{R}$ alone and then applied to score individual points of $\mathcal{D}$. As they yield a score per point rather than a collective two-sample statistic, we ``lift'' each to a single sample-level statistic by averaging.
  \begin{itemize}
    \item \emph{Isolation Forest}~\cite{liu2008iforest,liu2012iforest}.
      An ensemble of $t$ random binary trees (\emph{isolation trees}) recursively partitions the feature space by selecting a random coordinate and a random split value between its observed extremes; points in sparse regions are separated from the rest in few splits and therefore terminate close to the root.
      Each tree is grown on a sub-sample of $\psi$ reference points and branches until every point is isolated or the height limit $\lceil\log_2\psi\rceil$ is reached.
      The path length $h(x)$ is the number of edges from the root to the leaf in which $x$ lands, augmented by $c(\ell)$ when that leaf still contains $\ell>1$ training points, to account for the sub-tree that was not grown, where
        \begin{equation}
        c(\ell)=2H(\ell-1)-\frac{2(\ell-1)}{\ell},\qquad H(i)=\ln i+\gamma,
        \end{equation}
        is the expected path length of an unsuccessful search in a binary search tree of $\ell$ nodes ($\gamma\approx0.5772$ the Euler--Mascheroni constant). Averaging the path length over the ensemble, $E[h(x)]$, gives the per-point anomaly score
        \begin{equation}
        s_{\rm IF}(x)=2^{-\,E[h(x)]/c(\psi)},
        \end{equation}
      normalized by $c(\psi)$ so that a short path yields $s\!\to\!1$ (anomalous) and a typical point $s\!\to\!\tfrac12$.
      To compute the test, we use the \texttt{sklearn.ensemble.IsolationForest} library with \texttt{n\_estimators=200} and \texttt{contamination=`auto'}.
      The model is fit on the reference $\mathcal{R}$ alone, and the per-point scores are reduced to a single collective statistic by averaging $s(x)$ over the inspected sample $\mathcal{D}$, which is then calibrated empirically as for every other baseline.
      Isolation Forest rests on the \emph{few-and-different} premise: anomalies are assumed to be both rare and easily separated, hence isolated near the root.
      It is expected to perform best when the signal falls in a low-density part of the reference support. 
      Its power is expected to be weakest on an in-distribution over-density lying in the high-density bulk. 
      We stress that the sensitivity is only \emph{attenuated}, not eliminated, in this case: because each tree is built from a finite sub-sample of $\psi$ points and splits on the marginal coordinate, a local over-density does perturb the split geometry and leaves a residual signal, so Isolation Forest retains some, though reduced, power on bulk anomalies.
      The essential limitation is that isolation depth is primarily governed by the \emph{support and sparsity} of the data rather than by local density-ratio excesses \emph{within} the support---precisely the rare, in-distribution regime that \textsc{SparKer} is designed to target.
    \item \emph{ECOD}~\cite{li2022ecod}, as implemented in PyOD~\cite{zhao2019pyod}.
      ECOD (empirical-cumulative-distribution outlier detection) is a parameter-free detector that scores a point by how deep it lies in the tails of the reference distribution, estimated non-parametrically per feature.
      From the reference $\mathcal{R}$ it builds, for each dimension $j$, the empirical left-tail and right-tail distributions
      \begin{equation}
        \hat F^{L}_j(v)=\frac{1}{m}\sum_{r}\mathbb{I}[R_{r,j}\le v],
        \qquad
        \hat F^{R}_j(v)=\frac{1}{m}\sum_{r}\mathbb{I}[R_{r,j}\ge v],
      \end{equation}
      and converts them to per-dimension ``outlierness'' via the negative log tail probability, $O^{L}_j(x)=-\log\hat F^{L}_j(x_j)$ and $O^{R}_j(x)=-\log\hat F^{R}_j(x_j)$, so that a point in a rare tail (small tail probability) receives a large score.
      The per-dimension scores are aggregated across features by summation---equivalent, under an assumption of feature independence, to taking the negative log of the product of tail probabilities, i.e.\ the joint tail probability---yielding a left, a right, and a skewness-guided ``auto'' aggregate,
      \begin{equation}
          O^{L}(x)=\sum_j O^{L}_j(x),\quad
        O^{R}(x)=\sum_j O^{R}_j(x),\quad
        O^{\mathrm{auto}}(x)=\sum_j
          \begin{cases} O^{L}_j(x), & \mathrm{skew}_j<0\\[2pt] O^{R}_j(x), & \mathrm{skew}_j\ge 0,\end{cases}
      \end{equation}
      and the final per-point score is the most extreme of the three,
      \begin{equation}
          s_{\mathrm{ECOD}}(x)=\max\!\big(O^{L}(x),\,O^{R}(x),\,O^{\mathrm{auto}}(x)\big).
      \end{equation}
      Being non-parametric and free of hyperparameters or training, ECOD is attractive as an assumption-light baseline.
      As for the other per-point scorers, the detector is fit on the reference $\mathcal{R}$ alone and its per-point scores on $\mathcal{D}$ are reduced to a single collective statistic, by averaging
      \begin{equation}
          s_{\mathrm{ECOD}}=\frac{1}{n}\sum_{x\in\mathcal{D}} s_{\mathrm{ECOD}}(x).
      \end{equation}
      Because the score is entirely a measure of tail depth in the reference ECDF, ECOD is a sparsity/tail detector of the same character as Isolation Forest: it responds when $\mathcal{D}$ places excess mass in a low-probability tail of $\mathcal{R}$, and is essentially blind to an in-distribution over-density sitting near the mode, where the tail probabilities are of order one-half and the score barely changes---again the rare, in-distribution regime that \textsc{SparKer} targets.
\end{itemize}
\paragraph{Calibration and fairness of the comparison.}
For the two-sample test comparisons in Section~\ref{subapp:regimes}, every method described above (including those not calibrated in their standard usage) is run through the identical empirical calibration procedure used for \textsc{SparKer} (Algorithm~1 of the main text), on the same anomaly-free bootstrap experiments, and all reported detection powers are computed at a common $5\%$ false-positive rate set by that procedure.
This uniformity is not optional: calibration is a generic wrapper that turns any scalar score $s(\mathcal{D},\mathcal{R})$ into a frequentist test with a controlled false-positive rate, and it is a prerequisite for reporting a $p$-value at all. A score on its own cannot be converted into a $p$-value without comparing it to the distribution of that score under the null; the only way to bypass an empirical calibration is an analytic (e.g.\ asymptotic) null distribution, which some baselines admit in principle but which is generally unreliable in the low-sample, imbalanced regime studied here.
Calibrating all methods empirically therefore serves both validity and comparability: no method is advantaged by the calibration step relative to another, and the observed sensitivity differences are attributable to the choice of statistic alone.
We verified that, under $H_0$, the calibrated $p$-values of every method are uniform and the empirical false-positive rate matches the nominal $5\%$.
The same symmetry governs the hyperparameters: for every method whose sensitivity depends on a scale-like hyperparameter (the neighborhood size $k$, the SOM prototype count, the kernel bandwidths) the per-value tests are aggregated into a single calibrated result via Eq.~\ref{eq:aggregation}, mirroring \textsc{SparKer}'s aggregation over the annealed width $\sigma$; capacity parameters (the number of centers or kernels) are reported explicitly in the method labels, with \textsc{SparKer} shown at two small ensemble sizes spanning its a priori range together with the $M$-aggregated variant (see Section~\ref{subapp:regimes}).
The point-wise experiments of Section~\ref{subapp:pointwise-comparison} follow an analogous logic at the level of individual points. Each per-point score can itself be calibrated against its distribution on anomaly-free data, so that thresholding the score at a $5\%$ per-point false-positive rate is equivalent to cutting on a per-point $p$-value at $0.05$. The difference with respect to the collective rung lies in the figure of merit: rather than the calibrated significance of the whole sample, we care about the power of the tagger on individually labeled anomalous points, and thus report the true-positive rate at the fixed per-point operating point. For the unsupervised detectors, which are fit on the reference $\mathcal{R}$ alone and never see the inspected data $\mathcal{D}$, the threshold $t^\star$ is read directly off the (frozen) score distribution on $\mathcal{R}$. For \textsc{SparKer}, which is semi-supervised and fit on $(\mathcal{R},\mathcal{D})$, thresholding on the same data used in the fit would incur an in-sample (post-selection) bias; its threshold is therefore calibrated on held-out anomaly-free data via the cross-fitting protocol detailed in Section~\ref{subapp:pointwise-comparison}, placing the semi-supervised score on the same out-of-sample footing that the unsupervised detectors enjoy by construction. In both cases the anomaly labels are used only to evaluate the true-positive rate, never to set $t^\star$, and we verified that the realized per-point false-positive rate matches the nominal $5\%$ for every method.
\begin{figure}[h]
    \centering
    \begin{minipage}[c]{0.5\linewidth}
        \centering
        \includegraphics[width=\linewidth]{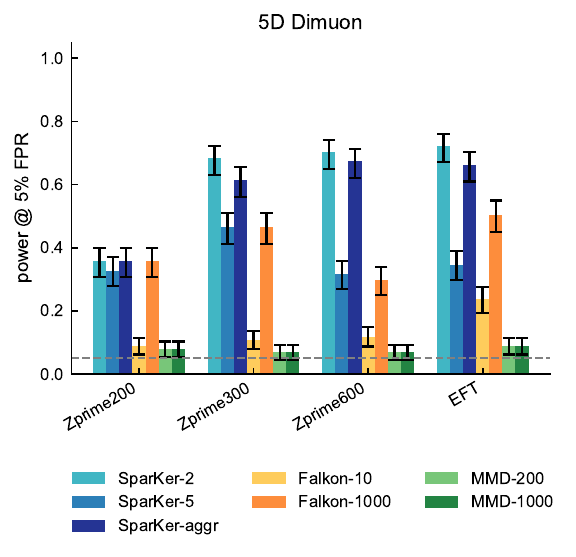}\\
        {\small (a)}
    \end{minipage}
    \hfill
    \begin{minipage}[c]{0.45\linewidth}
        \centering
        \includegraphics[width=\linewidth]{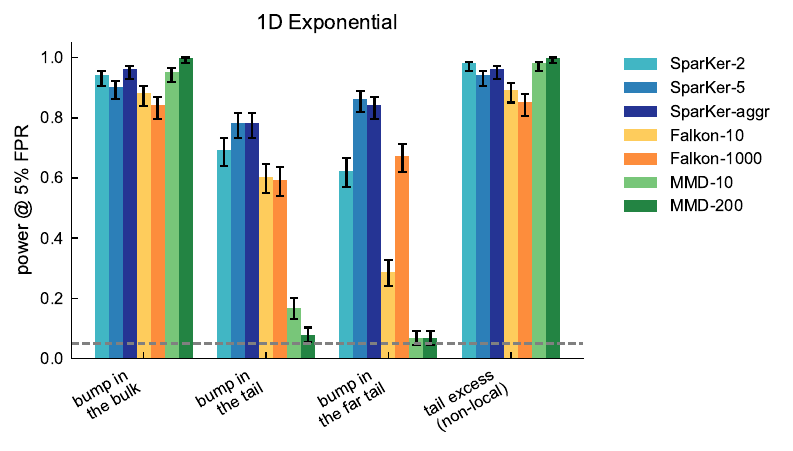}\\
        {\small (b)}\\[1ex]
        \includegraphics[width=\linewidth]{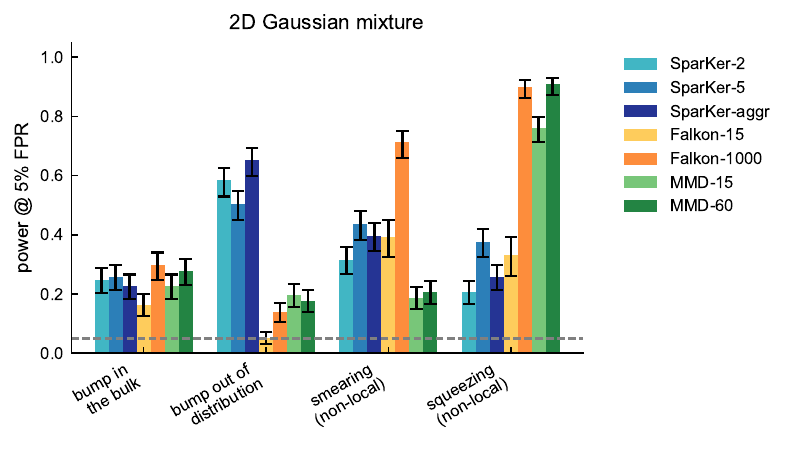}\\
        {\small (c)}
    \end{minipage}
    \caption{\textbf{Test power at a $5\%$ false-positive rate across anomaly regimes.} (a) The 5D di-muon benchmark, for three resonant $Z'$ signals ($m_{Z'}=200,300,600$~GeV) and a non-resonant EFT contact-interaction signal; \textsc{SparKer} attains the highest power among all methods on both signal geometries with as few as $M=2$ kernels. (b, c) Controlled one- and two-dimensional toy models, spanning localized signals in the bulk and tails and non-localized shape distortions (tail excess in 1D; smearing and squeezing in 2D). For rare anomalies \textsc{SparKer} matches or exceeds the Falkon and MMD baselines while using far fewer kernels; the number of kernels of each model follows the dash in the legends. The dashed line marks the $0.05$ level expected in the absence of sensitivity.}
    \label{fig:sensitivity}
\end{figure}
\subsection{Comparison with baseline methods: two-sample tests}\label{subapp:regimes}
We present the results of the collective two-sample comparison at two granularities, and organize the discussion by anomaly regime.
Figure~\ref{fig:sensitivity} tracks \textsc{SparKer} against the two kernel-based baselines across the full set of benchmarks: the realistic five-dimensional di-muon problem (panel a), the one-dimensional exponential toy (panel b), and the two-dimensional mixture of four Gaussians (panel c); the number of kernels defining each model follows the dash in the legends.
Tables~\ref{tab:baselines1d} and~\ref{tab:baselines2d} extend the comparison on the 1D and 2D benchmarks to the full field of baselines of Section~\ref{app:baselines}; the point-wise scorers are ``lifted'' to a batch statistic by aggregating their score over $\mathcal{D}$ (the mean per-point score), and every test passes through the shared empirical calibration.
Table~\ref{tab:baselines5d} extends the comparison to the 5D di-muon benchmark for a representative subset of the baselines: for the hyperparameter-dependent families ($k$-NN two-sample, $k$-NN distance, SOM quantization error and binned $\chi^2$) we carry forward the aggregated variant, which combines the per-hyperparameter tests into a single calibrated result, together with the single-configuration methods that showed sensitivity better than or similar to \textsc{SparKer} in at least one regime of the low-dimensional benchmarks; the methods sitting at the false-positive floor throughout the 1D and 2D studies (RuLSIF, KLIEP, Friedman--Rafsky) are not recomputed. The same selection logic is applied to the high-dimensional applications discussed at the end of this section. 
In all cases the metric is the test power at a $5\%$ type~I error, estimated over $100$ signal realizations, with the null distribution and the $5\%$ threshold obtained by calibrating the test empirically on $300$ anomaly-free experiments.
Throughout the tables, wherever a method has no sensitivity its power sits around the nominal $0.05$, confirming that the empirical calibration controls the false-positive rate correctly and that the comparison is fair.
\paragraph{\textsc{SparKer} configurations.}
The ensemble size $M$ bounds the number of distinct anomalous regions the ensemble can resolve, while capacity beyond this multiplicity degrades power by inflating the null distribution of the test statistic (see the ablation of Section~\ref{subsec:role_m}). On the other hand, the choice of $M$ also poses constraints to exploration: under time constraint a larger ensemble explores the space more efficiently than a smaller one. Thus, the optimal choice of $M$ depends on the problem and specific signal. 
We report $M{=}2$ and $M{=}5$ as the extremes of a reasonable range, $M{=}5$ being the default adopted, with no per-application tuning, in all the applications of the main text.
For completeness we also report an $M$-aggregated test, which removes the choice of ensemble size altogether by combining the two configurations through the same $p$-value combination used for the kernel width (Eq.~\ref{eq:aggregation}), paying the associated trial factor explicitly.
\begin{table}[h!]
\centering
\begin{tabular}{l cccc}
\toprule
Method & \texttt{bulk} & \texttt{tail} & \texttt{extreme tail} & \texttt{tail excess} \\
\midrule
\multicolumn{5}{l}{\textit{Distance-based two-sample tests}}\\
Energy distance        & $0.67\,[0.62,0.72]$ & $0.08\,[0.05,0.12]$ & $0.10\,[0.07,0.14]$ & $0.90\,[0.86,0.93]$ \\
Nystr\"om-MMD ($M{=}10$)& $\mathbf{0.95\,[0.92,0.96]}$ & $0.17\,[0.13,0.20]$ & $0.07\,[0.05,0.09]$ & $\mathbf{0.98\,[0.95,0.99]}$ \\
Nystr\"om-MMD ($M{=}200$)& $\mathbf{1.00\,[0.98,1.00]}$ & $0.08\,[0.05,0.10]$ & $0.07\,[0.05,0.09]$ & $\mathbf{1.00\,[0.98,1.00]}$ \\
$k$-NN two-sample balanced (aggregated)& $0.41\,[0.36,0.46]$ & $0.22\,[0.18,0.27]$ & $0.18\,[0.14,0.23]$ & $0.16\,[0.12,0.21]$ \\
$k$-NN distance  (aggregated) & $0.09\,[0.06,0.13]$ & $0.55\,[0.50,0.60]$ & $\mathbf{0.82\,[0.77,0.86]}$ & $0.69\,[0.64,0.74]$ \\
Friedman--Rafsky       & $0.00\,[0.00,0.02]$ & $0.04\,[0.02,0.07]$ & $0.06\,[0.04,0.09]$ & $0.00\,[0.00,0.02]$ \\
\midrule
\multicolumn{5}{l}{\textit{Density-ratio-based methods}}\\
RuLSIF                 & $0.27\,[0.22,0.32]$ & $0.02\,[0.01,0.05]$ & $0.01\,[0.00,0.03]$ & $0.08\,[0.05,0.12]$ \\
KLIEP                  & $0.03\,[0.01,0.06]$ & $0.03\,[0.01,0.06]$ & $0.03\,[0.01,0.06]$ & $0.01\,[0.00,0.03]$ \\
\midrule
\multicolumn{5}{l}{\textit{Density-based methods (unsup.)}}\\
Typicality             & $0.03\,[0.01,0.06]$ & $0.06\,[0.04,0.09]$ & $0.16\,[0.12,0.21]$ & $0.64\,[0.59,0.69]$ \\
likelihood novelty score (KDE) & $0.26\,[0.21,0.31]$ & $0.29\,[0.24,0.34]$ & $0.28\,[0.23,0.33]$ & $0.95\,[0.92,0.97]$ \\
\midrule
\multicolumn{5}{l}{\textit{Self-Organizing-Map based tests (unsup.)}}\\
SOM quant.\ error (aggregated)  & $0.10\,[0.07,0.14]$ & $0.64\,[0.59,0.69]$ & $0.63\,[0.58,0.68]$ & $0.75\,[0.70,0.80]$ \\
SOM binned $\chi^2$ (aggregated) & $\mathbf{0.95\,[0.92,0.97]}$ & $0.07\,[0.04,0.11]$ & $0.05\,[0.03,0.08]$ & $0.65\,[0.60,0.70]$ \\
\midrule
\multicolumn{5}{l}{\textit{Other point-wise statistics (unsup.)}}\\
Isolation Forest       & $0.09\,[0.06,0.13]$ & $0.26\,[0.21,0.31]$ & $0.18\,[0.14,0.23]$ & $0.85\,[0.81,0.89]$ \\
ECOD                   & $0.01\,[0.00,0.03]$ & $0.24\,[0.20,0.29]$ & $0.18\,[0.14,0.23]$ & $0.72\,[0.67,0.77]$ \\
\midrule
\multicolumn{5}{l}{\textit{Neyman--Pearson two-sample tests}}\\
Falkon ($M{=}10$)& $0.88\,[0.84,0.91]$ & $0.60\,[0.55,0.65]$ & $0.29\,[0.24,0.33]$ & $0.89\,[0.85,0.91]$ \\
Falkon ($M{=}1000$)& $0.84\,[0.80,0.87]$ & $0.59\,[0.54,0.64]$ & $0.67\,[0.62,0.71]$ & $0.85\,[0.81,0.88]$ \\
\textbf{SparKer ($M{=}2$) ---ours}  & $\mathbf{0.92\,[0.88,0.94]}$ & $0.65\,[0.60,0.69]$ & $0.61\,[0.56,0.66]$ & $0.97\,[0.94,0.98]$ \\
\textbf{SparKer ($M{=}5$) ---ours}  & $\mathbf{0.92\,[0.88,0.94]}$ & $\mathbf{0.80\,[0.75,0.83]}$ & $\mathbf{0.86\,[0.82,0.89]}$ & $0.94\,[0.91,0.96]$\\
\textbf{SparKer ($M$-aggregated) ---ours}  & $\mathbf{0.95\,[0.92,0.96]}$ & $\mathbf{0.78\,[0.73,0.82]}$ & $\mathbf{0.85\,[0.81,0.88]}$ & $0.95\,[0.92,0.96]$ \\
\bottomrule
\end{tabular}
\caption{\textbf{Two-sample test (1D).} Detection power at $5\%$ false-positive rate on the 1D benchmark, for each signal type. Values are power with $68\%$ Clopper--Pearson intervals, estimated over $n=100$ independent realizations; all methods share the identical signal definitions (Table~\ref{tab:1d-dataset}) and empirical calibration (Algorithm~1 of the main text). Best entries per column in \textbf{bold}.}\label{tab:baselines1d}
\end{table}

\begin{table}[h!]
\centering
\begin{tabular}{l cccc}
\toprule
Method & \texttt{bulk} & \texttt{ood} & \texttt{smearing} & \texttt{squeezing} \\
\midrule
\multicolumn{5}{l}{\textit{Distance-based two-sample tests}}\\
Energy distance                              & $0.08\,[0.05,0.12]$ & $0.06\,[0.04,0.09]$ & $0.04\,[0.02,0.07]$ & $0.10\,[0.07,0.14]$ \\
Nystr\"om-MMD ($M{=}15$)                     & $0.23\,[0.19,0.27]$ & $0.20\,[0.16,0.24]$ & $0.19\,[0.15,0.22]$ & $0.76\,[0.71,0.80]$ \\
Nystr\"om-MMD ($M{=}60$)                     & $\mathbf{0.28\,[0.23,0.32]}$ & $0.18\,[0.14,0.21]$ & $0.21\,[0.17,0.25]$ & $\mathbf{0.91\,[0.87,0.93]}$ \\
$k$-NN two-sample balanced (aggregated)     & $0.11\,[0.08,0.15]$ & $0.20\,[0.16,0.25]$ & $0.60\,[0.55,0.65]$ & $0.26\,[0.21,0.31]$ \\
$k$-NN distance  (aggregated)               & $0.08\,[0.05,0.12]$ & $\mathbf{0.76\,[0.71,0.80]}$ & $\mathbf{0.96\,[0.93,0.98]}$ & $0.00\,[0.00,0.02]$ \\
Friedman--Rafsky                   & $0.02\,[0.01,0.05]$ & $0.05\,[0.03,0.08]$ & $0.08\,[0.05,0.12]$ & $0.02\,[0.01,0.05]$ \\
\midrule
\multicolumn{5}{l}{\textit{Density-ratio-based methods}}\\
RuLSIF                                      & $0.27\,[0.22,0.32]$ & $0.05\,[0.03,0.08]$ & $0.01\,[0.00,0.03]$ & $0.35\,[0.30,0.40]$ \\
KLIEP                              & $0.03\,[0.01,0.06]$ & $0.02\,[0.01,0.05]$ & $0.02\,[0.01,0.05]$ & $0.19\,[0.15,0.24]$ \\
\midrule
\multicolumn{5}{l}{\textit{Density-based methods (unsup.)}}\\
Typicality                                  & $0.04\,[0.02,0.07]$ & $0.30\,[0.25,0.35]$ & $0.77\,[0.72,0.81]$ & $0.73\,[0.68,0.78]$ \\
likelihood novelty score (KDE)     & $0.06\,[0.04,0.09]$ & $0.49\,[0.44,0.54]$ & $\mathbf{0.97\,[0.94,0.99]}$ & $0.00\,[0.00,0.02]$ \\
\midrule
\multicolumn{5}{l}{\textit{Self-Organizing-Map based tests (unsup.)}}\\
SOM quant.\ error (aggregated)   & $0.10\,[0.07,0.14]$ & $0.13\,[0.10,0.17]$ & $0.48\,[0.43,0.53]$ & $0.01\,[0.00,0.03]$ \\
SOM binned $\chi^2$ (aggregated) & $0.16\,[0.12,0.21]$ & $0.39\,[0.34,0.44]$ & $0.30\,[0.25,0.35]$ & $0.28\,[0.23,0.33]$ \\
\midrule
\multicolumn{5}{l}{\textit{Other point-wise statistics (unsup.)}}\\
Isolation Forest                   & $0.09\,[0.06,0.13]$ & $0.11\,[0.08,0.15]$ & $0.45\,[0.40,0.50]$ & $0.01\,[0.00,0.03]$ \\
ECOD                               & $0.00\,[0.00,0.02]$ & $0.08\,[0.05,0.12]$ & $0.37\,[0.32,0.42]$ & $0.01\,[0.00,0.03]$ \\
\midrule
\multicolumn{5}{l}{\textit{Neyman--Pearson two-sample tests}}\\
Falkon ($M{=}15$)                   & $0.17\,[0.13,0.20]$ & $0.05\,[0.03,0.07]$ & $0.39\,[0.33,0.45]$ & $0.33\,[0.26,0.39]$ \\
Falkon ($M{=}1000$)                 & $\mathbf{0.30\,[0.25,0.34]}$ & $0.14\,[0.11,0.17]$ & $0.71\,[0.66,0.75]$ & $\mathbf{0.90\,[0.86,0.92]}$ \\
\textbf{SparKer ($M{=}2$) ---ours}& $0.25\,[0.20,0.29]$ & $0.58\,[0.53,0.63]$ & $0.32\,[0.27,0.36]$ & $0.21\,[0.17,0.25]$ \\
\textbf{SparKer ($M{=}5$) ---ours}& $\mathbf{0.27\,[0.22,0.31]}$ & $0.50\,[0.44,0.54]$ & $0.44\,[0.38,0.48]$ & $0.37\,[0.32,0.41]$\\
\textbf{SparKer ($M$-aggregated) ---ours}& $0.23\,[0.19,0.27]$ & $0.65\,[0.60,0.69]$ & $0.40\,[0.34,0.44]$ & $0.26\,[0.21,0.30]$ \\
\bottomrule
\end{tabular}
\caption{\textbf{Two-sample test (2D).} Detection power at $5\%$ false-positive rate on the 2D benchmark, for each signal type. Values are power with $68\%$ Clopper--Pearson intervals, estimated over $n=100$ independent realizations; all methods share the identical signal definitions (Table~\ref{tab:2d-dataset}) and empirical calibration (Algorithm~1 of the main text). Best entries per column in \textbf{bold}.}\label{tab:baselines2d}
\end{table}

\begin{table}[h!]
\centering
\begin{tabular}{l cccc}
\toprule
Method & \texttt{Z'200 (bulk)} & \texttt{Z'300 (tail)} & \texttt{Z'600 (extreme tail)} & \texttt{EFT (non-resonant)} \\
\midrule
\multicolumn{5}{l}{\textit{Distance-based two-sample tests}}\\
Nystr\"om-MMD ($M{=}200$) & $0.08\,[0.05,0.10]$ & $0.07\,[0.05,0.09]$ & $0.07\,[0.05,0.09]$ & $0.09\,[0.06,0.11]$ \\
Nystr\"om-MMD ($M{=}1000$)& $0.08\,[0.05,0.10]$ & $0.07\,[0.05,0.09]$ & $0.07\,[0.05,0.09]$ & $0.09\,[0.06,0.11]$ \\
$k$-NN two-sample balanced (aggregated) & $0.12\,[0.09,0.16]$ & $0.20\,[0.16,0.25]$ & $0.13\,[0.10,0.17]$ & $0.12\,[0.09,0.16]$ \\
$k$-NN distance (aggregated)            & $0.36\,[0.31,0.41]$ & $0.35\,[0.30,0.40]$ & $0.40\,[0.35,0.45]$ & $0.85\,[0.81,0.89]$ \\
\midrule
\multicolumn{5}{l}{\textit{Density-based methods (unsup.)}}\\
Typicality                                   & $0.17\,[0.12,0.23]$ & $0.06\,[0.04,0.09]$ & $0.12\,[0.09,0.16]$ & $0.82\,[0.77,0.86]$ \\
likelihood novelty score (KDE)               & $\mathbf{0.37\,[0.28,0.46]}$ & $0.31\,[0.26,0.37]$ & $0.41\,[0.36,0.47]$ & $\mathbf{0.92\,[0.88,0.95]}$ \\
\midrule
\multicolumn{5}{l}{\textit{Self-Organizing-Map based tests (unsup.)}}\\
SOM quant.\ error (aggregated)          & $0.16\,[0.12,0.21]$ & $0.21\,[0.17,0.26]$ & $0.08\,[0.05,0.12]$ & $0.36\,[0.31,0.42]$ \\
SOM binned $\chi^2$ (aggregated) & $0.09\,[0.06,0.13]$ & $0.02\,[0.01,0.05]$ & $0.05\,[0.03,0.08]$ & $0.04\,[0.02,0.07]$ \\
\midrule
\multicolumn{5}{l}{\textit{Neyman--Pearson two-sample tests}}\\
Falkon ($M{=}10$)         & $0.09\,[0.06,0.11]$ & $0.11\,[0.08,0.14]$ & $0.12\,[0.09,0.15]$ & $0.24\,[0.19,0.28]$ \\
Falkon ($M{=}1000$)       & $0.36\,[0.30,0.40]$ & $0.47\,[0.41,0.51]$ & $0.30\,[0.25,0.34]$ & $0.51\,[0.45,0.55]$ \\
\textbf{SparKer ($M{=}2$)---ours}  & $\mathbf{0.38\,[0.33,0.42]}$ & $\mathbf{0.70\,[0.66,0.74]}$ & $\mathbf{0.70\,[0.66,0.74]}$ & $0.71\,[0.68,0.75]$ \\
\textbf{SparKer ($M{=}5$)---ours}  & $0.33\,[0.28,0.37]$ & $0.47\,[0.41,0.51]$ & $0.32\,[0.27,0.36]$ & $0.35\,[0.30,0.39]$ \\
\textbf{SparKer ($M$-aggregated)---ours}& $\mathbf{0.37\,[0.32,0.41]}$ & $0.64\,[0.59,0.68]$ & $\mathbf{0.66\,[0.61,0.70]}$ & $0.64\,[0.59,0.68]$ \\
\bottomrule
\end{tabular}
\caption{\textbf{Two-sample test (5D).} Detection power at $5\%$ false-positive rate on the 5D benchmark, for each signal type. Values are power with $68\%$ Clopper--Pearson intervals, estimated over $n=100$ independent realizations; all methods share the identical signal definitions (Table~\ref{tab:5d-dataset}) and empirical calibration (Algorithm~1 of the main text). Best entries per column in \textbf{bold}.}\label{tab:baselines5d}
\end{table}

\paragraph{Localized signals in the bulk.} When a compact excess sits in a high-density region of the reference (first column in panels (b) and (c) of Figure~\ref{fig:sensitivity}, and \texttt{bulk} columns of the tables), the log-density-ratio is small and the anomaly is intrinsically hard to detect.
This is also the regime that most sharply separates the paradigms: the point-wise outlier scorers (ECOD, Isolation Forest, $k$-NN distance, SOM quantization error) collapse to the false-positive floor, since a point added inside the bulk is indistinguishable from its neighbors, while genuine collective two-sample statistics retain power.
In 1D the most sensitive baselines are the Nystr\"om-MMD ($\sim\!0.95$ at $M{=}10$, $\sim\!1$ at $M{=}200$), the aggregated SOM $\chi^2$ test ($\sim\!0.95$), and the energy distance ($\sim\!0.67$); \textsc{SparKer} follows closely ($0.92$ at the default $M{=}5$, $0.92$ at $M{=}2$, $0.95$ aggregated) using only a handful of kernels---and, unlike these bulk-sensitive statistics, it retains its power as the anomaly leaves the bulk (next paragraph).
In 2D the \texttt{bulk} over-density is the hardest signal for every method: the specialized detectors that excel elsewhere fall to the $5\%$ floor, and the only methods retaining appreciable power are the collective Neyman--Pearson and kernel tests, with \textsc{SparKer} ($0.27$ at $M{=}5$, $0.25$ at $M{=}2$, $0.23$ aggregated) 
among the best alongside Falkon ($0.30$) and Nystr\"om-MMD ($0.28$).
\paragraph{Localized signals in the tails.} As the signal moves into the tail and the extreme tail (second and third columns in panels (b) and (c), and in Table~\ref{tab:baselines1d}), the ordering among baselines reverses: the collective bulk-sensitive tests lose power, while distance-based scores take over---the $k$-NN distance reaches $\sim\!0.82$ in the far tail 
and the quantization error becomes the strongest of the two SOM variants---consistent with these methods effectively testing membership in the support of $H_0$.
The kernel Neyman--Pearson baselines also degrade markedly, as the anomalous pattern must be distinguished from statistical fluctuations in sparsely populated regions.
\textsc{SparKer} retains the highest power among the collective tests throughout ($0.80$--$0.86$ at the default $M{=}5$), and does so with a small number of trainable parameters, indicating that sparsity, locality and competition are a well-matched set of inductive biases for rare, in-distribution anomalies.
The 5D benchmark confirms this picture at higher dimensionality (Table~\ref{tab:baselines5d}): on the $Z'$ resonances at $300$ and $600$~GeV, compact excesses in the sparsely populated high-mass region of the spectrum, \textsc{SparKer} with only two kernels reaches a power of $0.70$ on both, significantly higher than that of the strongest alternative

(no baseline exceeds $\sim\!0.47$), while the collective Nystr\"om-MMD sits at the false-positive floor.
\paragraph{Rare, delocalized signals.} The tail excess in panel~(b) and the EFT signal in panel~(a) contain \emph{no} compact excess---the deviation is spread smoothly over the tail or the whole spectrum, an explicitly non-resonant signature of the type increasingly relevant for new-physics searches---yet both are \emph{rare}, affecting only a small fraction of the data. These signals fall squarely within \textsc{SparKer}'s target regime: the sparsity prior is satisfied, and the multi-scale test captures the deviation at the appropriate width, yielding high detection power in both cases.
The 1D \texttt{tail\_excess}, injected at a larger fraction, is comparatively easy and is recovered by most unsupervised and global two-sample methods as well (Table~\ref{tab:baselines1d}).
The 5D EFT signal probes the same regime at higher dimensionality, and reproduces the 1D tail-excess pattern (Table~\ref{tab:baselines5d}): because the non-resonant excess populates the sparse high-mass tail, the unsupervised tail and outlier detectors respond strongly (KDE likelihood score $0.92$, $k$-NN distance $0.85$, typicality $0.82$). 
Among the collective two-sample tests, however, only the Neyman--Pearson family retains power, with \textsc{SparKer} at $0.71$ using two kernels, Falkon at $0.51$ using a thousand, and the Nystr\"om-MMD at the floor. Crucially, the tail detectors that lead on the EFT column are the same methods that sit near the floor on the resonant signals: no baseline is sensitive to both signal geometries at once.
\paragraph{Global, non-rare distortions.} The smearing and squeezing scenarios in panel~(c) and Table~\ref{tab:baselines2d} are of a different nature: rescaling the width of a background component deforms the distribution of \emph{every} event in that component, rather than a bounded subset, so the deviation is not rare by construction. Here the sparsity prior is mismatched, and \textsc{SparKer}, while not blind, is less performant than the dense alternatives.
The two distortions also expose an instructive asymmetry among the baselines: the broadening (\texttt{smearing}) is recovered by density- and distance-based methods (likelihood novelty $0.97$, $k$-NN distance $0.96$, typicality $0.77$), 
whereas the concentrating distortion (\texttt{squeezing}) is detected by the collective kernel tests (Nystr\"om-MMD $0.91$, Falkon $0.90$, typicality $0.73$) while the likelihood- and distance-based scores collapse to zero on it---a signal that \emph{increases} the local density raises, rather than lowers, the typical likelihood and shrinks nearest-neighbor distances, so detectors that equate ``anomalous'' with ``low-density'' or ``far'' are blind to it.
For \textsc{SparKer} this regime is the expected limitation rather than a failure mode: for such global shifts, the Neyman--Pearson test is better served by dense kernel methods without sparsity and competition, as in Falkon, which recover the expected sensitivity. Stating the conditions under which a detector's assumptions are and are not suitable is, in our view, part of its specification.
\paragraph{One architecture across regimes.} The di-muon benchmark of panel~(a) is particularly telling, as it probes opposite signal geometries within a single, physically-motivated dataset: the same \textsc{SparKer} architecture, with no signal-specific modification, attains the highest power among all methods on both the compact resonances and the delocalized EFT deviation, matching or exceeding baselines that employ up to three orders of magnitude more kernels.
Consistent with the unimodal nature of these signals and the capacity argument of Section~\ref{subsec:role_m}, the smallest ensemble ($M{=}2$) is the most powerful configuration, and the $M$-aggregated test, which removes the choice of ensemble size, performs comparably (Table~\ref{tab:baselines5d}). 
The extended comparison of Table~\ref{tab:baselines5d} reinforces the point: \textsc{SparKer} is the only method with substantial power on all four columns, while each baseline responds to at most one of the two signal geometries.
\paragraph{Summary.} Three cross-cutting findings emerge from the comparison.
First, \emph{no single baseline is sensitive across regimes}: each paradigm succeeds only in the corner of the anomaly space it was designed for and collapses elsewhere---the distance and tail detectors dominate on out-of-support and extreme-tail signals but fall to the false-positive floor on the rare in-distribution bulk anomaly, while the collective bulk-sensitive tests (MMD, energy distance, binned $\chi^2$) show the mirror-image pattern.
The two SOM adaptations exemplify this within a single model family, behaving precisely as their construction predicts (see Section~1.5 of the main text): the quantization-error variant acts as a distance detector (null in the bulk, strong in the tails), the binned-$\chi^2$ variant as a collective test (strong in the bulk, null in the tails), and no choice of the number of prototypes $M$ makes either sensitive everywhere (see Section~\ref{subapp:som-vis} and Table~\ref{tab:som} for the full scan over $M$).
Second, the two density-ratio estimators (RuLSIF, KLIEP)---the baselines closest in spirit to \textsc{SparKer}, since they target the same quantity---remain weak throughout, indicating that estimating a density ratio is not by itself sufficient in this regime: it is the combination of locality, competition and scale annealing with the Neyman--Pearson objective that delivers the observed sensitivity.
Third, in contrast to every individual baseline, \textsc{SparKer} retains sensitivity across all benchmarks simultaneously, using only a handful of kernels; in 2D it does not dominate every column---the distance and kernel baselines win on the signals designed for them---but it is the only method with non-negligible power on all four signals at once, and it remains among the best on the in-distribution \texttt{bulk} anomaly that is the target of this work.
This is the central point of the comparison: the value of the method is not that it beats each specialized detector in that detector's own regime, but that a single, interpretable, self-organizing ensemble remains sensitive across the whole diagram of Fig.~1 of the main text---and in particular in the rare, low-separability, in-distribution corner where the specialized baselines that succeed elsewhere fail.
The dependence of these results on the ensemble size $M$ is studied as a dedicated ablation in Section~\ref{subsec:role_m}: capacity beyond the anomaly multiplicity systematically costs power, motivating the small ensembles adopted here, while the $M$-aggregated test offers a choice-free alternative at the cost of an explicitly paid trial factor.
This overall behavior is consistent with the multi-scale nature of the test: broad kernels early in the annealing capture smooth, extended deviations such as the EFT distortion, while narrower kernels resolve compact structures such as the $Z'$ resonances, so that scanning the kernel width grants sensitivity across scales.
Consistent with the interpretability analysis of the main text, for the EFT signal the \textsc{SparKer} components most correlated with the anomaly score localize at high di-muon invariant mass $m_{\ell\ell}$, even though $m_{\ell\ell}$ is never supplied to the model; ``locality'' in \textsc{SparKer} thus refers to structure in the density-ratio, not to a peak in a raw observable.
\paragraph{High-dimensional applications.}
From each category of tests we select the variants with overall higher performance in the 1D and 2D experiments and compute them for the high-dimensional applications considered in this work: scientific discovery in gravitational waves, particle jets and butterfly images; intrusion detection in text and out-of-domain detection in images.
The results are reported in Figure~8 of the main text: \textsc{SparKer}, run at the fixed default $M{=}5$, exceeds or equals all the baselines across experiments, and the picture that emerged in the 1D and 2D benchmarks is reproduced---none of the alternatives retains high sensitivity across applications, except \textsc{SparKer} and its companion Neyman--Pearson method \textsc{Falkon}.
\subsection{Comparison with baseline methods: point-wise detection}\label{subapp:pointwise-comparison}
The comparison of Section~\ref{subapp:regimes} places every method on the collective, two-sample footing of this paper: each test returns a single verdict on the whole sample $\mathcal{D}$, and the $5\%$ false-positive rate is defined at the level of the \emph{sample}.
Many of the baselines, however, are natively \emph{point-wise} detectors: they assign an anomaly score to each individual point and are normally used to \emph{tag} anomalous points rather than to test a whole batch.
To evaluate these methods in their intended mode---and to compare \textsc{SparKer} against them on that ground---we add a second rung of experiments that operates at the level of individual points.
We stress that the figure of merit here is a \emph{per-point} false-positive rate, a distinct quantity from the collective power@$5\%$ of Section~\ref{subapp:regimes}, and the two sets of tables should not be read against each other cell by cell.
\paragraph{Protocol.}
We use the ground-truth point labels, available for the \emph{injected-signal} benchmarks, to measure how well each detector separates anomalous from normal points at a controlled false-positive rate.
Every method produces a per-point score $s(x)$ (larger $=$ more anomalous), and in every case the decision threshold $t^\star$ is set on anomaly-free data to target a $5\%$ per-point false-positive rate, with the labels of the anomalous points used only to evaluate the true-positive rate, never to set $t^\star$.
The unsupervised detectors (Isolation Forest, ECOD, $k$-NN distance, SOM quantization error) are fit on the reference $\mathcal{R}$ and then score each point of the inspected sample $\mathcal{D}$ directly; since these methods never see $\mathcal{D}$ during fitting, their scores are automatically out-of-sample and the threshold is read off the (frozen) score distribution on $\mathcal{R}$.
Conversely, \textsc{SparKer} is semi-supervised, fit on $(\mathcal{R},\mathcal{D})$, so its per-point score $c(x)=\mathrm{Sigmoid}(f^{\sigma}_{M}(x))$ (Eq.~13 of the main text) cannot be thresholded on the same data without incurring an in-sample (post-selection) bias.
We therefore evaluate \textsc{SparKer} by data splitting: the model is trained on one half of $\mathcal{R}$ and $\mathcal{D}$, the learned score is then frozen, its threshold $t^\star$ is calibrated on the held-out half of the reference, and the held-out half of the data is tagged against it; the two halves are then swapped and the held-out taggings pooled (cross-fitting), so that the full inspected sample is evaluated out-of-sample.
This places the semi-supervised score on the same out-of-sample footing that the unsupervised detectors enjoy by construction, and yields a per-point false-positive rate controlled at $5\%$ without post-selection bias.
Because \textsc{SparKer} scans the kernel width $\sigma$ during training, its per-point score aggregates the frozen per-scale scores through the same multiple-resolution combination used for the collective test (Algorithm~1 of the main text); the unsupervised detectors have a single native scale.
We report, on the inspected sample, the true-positive rate $\mathrm{TPR}@5\%=\Pr[s(x)\ge t^\star \mid x\ \text{anomalous}]$ together with the realized per-point false-positive rate $\Pr[s(x)\ge t^\star \mid x\ \text{normal}]$, the latter serving as a check that the operating point is genuinely $\approx 5\%$ on the test population.
Both are averaged over the same independent realizations used in Section~\ref{subapp:regimes}, with $68\%$ Clopper--Pearson confidence intervals.
\paragraph{Scope.}
This evaluation is meaningful only where individual points carry an anomaly label, i.e.\ for the additive injected signals, in which the injected points are individually anomalous.
It is \emph{not} defined for the non-local distortions (the 2D \texttt{smearing} and \texttt{squeezing} benchmarks), where the anomaly is a change of the whole distribution and every point is drawn from a valid density: there is no ``anomalous point'' to tag, which is itself a reminder that some anomalies are intrinsically collective and can only be detected by a two-sample test.
We therefore restrict the point-wise comparison to the injected-signal benchmarks and include only the methods that are genuinely point-wise scorers, namely the unsupervised detectors above and \textsc{SparKer}; the collective statistics of Section~\ref{subapp:regimes} (energy distance, MMD, the two-sample $k$-NN and Friedman--Rafsky tests, the density-ratio estimators, and the SOM binned test) are excluded here by construction.
We run \textsc{SparKer} at the default configuration $M{=}5$.
\paragraph{A note on fairness.}
The threshold is set identically for all methods, at a $5\%$ per-point false-positive rate on anomaly-free data, so no method is advantaged at the tagging stage.
There remains a definitional asymmetry that we state explicitly rather than correct for: the unsupervised detectors are fit on $\mathcal{R}$ alone, whereas \textsc{SparKer} is semi-supervised and its score $c(x)$ is obtained from a model trained on $(\mathcal{R},\mathcal{D})$, i.e.\ in the presence of the anomalies.
This is the intended difference between the unsupervised and semi-supervised settings, and part of what the comparison is meant to expose, not an artifact of the protocol.
\begin{table}[h!]
\centering
\scalebox{0.99}{
\begin{tabular}{l cccc}
\toprule
Method & \texttt{bulk} & \texttt{tail} & \texttt{extreme\_tail} & \texttt{tail\_excess} \\
\midrule
\multicolumn{5}{l}{\textit{Unsupervised point-wise detectors}}\\
ECOD             & $0.00\,[0.00,0.00]$ & $\mathbf{1.00\,[1.00,1.00]}$ & $\mathbf{1.00\,[1.00,1.00]}$ & $0.27\,[0.27,0.28]$ \\
Isolation Forest & $0.00\,[0.00,0.00]$ & $\mathbf{1.00\,[1.00,1.00]}$ & $\mathbf{1.00\,[1.00,1.00]}$ & $0.42\,[0.41,0.43]$ \\
$k$-NN distance  & $0.00\,[0.00,0.00]$ & $\mathbf{1.00\,[1.00,1.00]}$ & $\mathbf{1.00\,[1.00,1.00]}$ & $\mathbf{0.45\,[0.45,0.46]}$ \\
SOM quant.\ error ($M{=}200$)& $0.00\,[0.00,0.00]$ & $\mathbf{1.00\,[1.00,1.00]}$ & $\mathbf{1.00\,[1.00,1.00]}$ & $0.25\,[0.24,0.26]$ \\
\midrule
\multicolumn{5}{l}{\textit{Semi-supervised}}\\
\textbf{SparKer ($M{=}5$) ---ours} & $\mathbf{0.22\,[0.21,0.23]}$ & $0.94\,[0.93,0.95]$ & $0.93\,[0.91,0.95]$ & $0.24\,[0.23,0.25]$ \\
\bottomrule
\end{tabular}}
\caption{\textbf{Point-wise detection (1D).} True-positive rate at a $5\%$ per-point false-positive rate ($\mathrm{TPR}@5\%$), on the 1D benchmarks where point-level labels exist. The decision threshold targets a $5\%$ per-point false-positive rate and is set on anomaly-free data: on the frozen score distribution over the reference $\mathcal{R}$ for the unsupervised detectors, and on held-out anomaly-free data via cross-fitting for \textsc{SparKer} (see text). Values are $\mathrm{TPR}$ with $68\%$ Clopper--Pearson intervals over $n=100$ realizations. Note that the per-point false-positive rate is distinct from the collective power@$5\%$ of Tables~\ref{tab:baselines1d}--\ref{tab:baselines2d}, and the two should not be compared directly.}
\label{tab:pointwise-1d}
\end{table}
\begin{table}[h!]
\centering
\scalebox{0.99}{
\begin{tabular}{l cc}
\toprule
Method & \texttt{bulk} & \texttt{ood} \\
\midrule
\multicolumn{3}{l}{\textit{Unsupervised point-wise detectors}}\\
ECOD             & $0.00\,[0.00,0.00]$ & $0.00\,[0.00,0.00]$ \\
Isolation Forest & $0.00\,[0.00,0.00]$ & $\mathbf{1.00\,[1.00,1.00]}$ \\
$k$-NN distance  & $0.05\,[0.05,0.05]$ & $\mathbf{1.00\,[1.00,1.00]}$ \\
SOM quant.\ error ($M{=}200$)& $0.06\,[0.05,0.06]$ & $\mathbf{1.00\,[1.00,1.00]}$ \\
\midrule
\multicolumn{3}{l}{\textit{Semi-supervised}}\\
\textbf{SparKer ($M{=}5$) ---ours} & $\mathbf{0.20\,[0.19,0.21]}$ & $0.67\,[0.64,0.70]$ \\
\bottomrule
\end{tabular}}
\caption{\textbf{Point-wise detection (2D).} True-positive rate at a $5\%$ per-point false-positive rate ($\mathrm{TPR}@5\%$), on the 2D benchmarks where point-level labels exist. The decision threshold targets a $5\%$ per-point false-positive rate and is set on anomaly-free data: on the frozen score distribution over the reference $\mathcal{R}$ for the unsupervised detectors, and on held-out anomaly-free data via cross-fitting for \textsc{SparKer} (see text). Values are $\mathrm{TPR}$ with $68\%$ Clopper--Pearson intervals over $n=100$ realizations. Note that the per-point false-positive rate is distinct from the collective power@$5\%$ of Tables~\ref{tab:baselines1d}--\ref{tab:baselines2d}, and the two should not be compared directly.}
\label{tab:pointwise-2d}
\end{table}

\paragraph{One-dimensional results.}
Table~\ref{tab:pointwise-1d} reports the point-wise true-positive rate at a $5\%$ per-point false-positive rate on the 1D benchmarks that carry point-level labels; the result is the point-level counterpart of the collective picture, and it is complementary.
The unsupervised tail/sparsity detectors (ECOD, Isolation Forest, $k$-NN distance, SOM quantization error) tag \emph{every} injected point in the \texttt{tail} and \texttt{extreme\_tail} regimes ($\mathrm{TPR}=1.00$): an individual point lying far out in the sparse tail is trivially flagged by a distance- or tail-based score.
The same detectors, however, tag \emph{none} of the injected \texttt{bulk} points ($\mathrm{TPR}=0.00$): a point added inside the high-density bulk is surrounded by reference points and is point-wise indistinguishable from them.
\textsc{SparKer} is the \emph{only} method that tags in-distribution \texttt{bulk} anomalies ($\mathrm{TPR}\approx0.22$), because its per-point score $c(x)$ reflects the local \emph{over-density} of data relative to the reference rather than distance to the support; it also flags tail points, with $\mathrm{TPR}\approx0.93$, slightly lower than the unsupervised methods.
This is the point-level expression of the paper's central message: tagging a rare anomaly that lives \emph{inside} the bulk is precisely what the sparsity/distance detectors cannot do and what \textsc{SparKer} is built for.
On the non-local \texttt{tail\_excess} signal the injected points are spread through the tail and are partially tagged by the tail detectors ($0.25$--$0.45$), with \textsc{SparKer} reaching a comparable performance ($\mathrm{TPR}\approx0.24$).
\paragraph{Two-dimensional results.}
Table~\ref{tab:pointwise-2d} repeats the point-wise evaluation on the 2D benchmarks with point-level labels (\texttt{bulk} and the central-void cluster \texttt{ood}); the non-local distortions are excluded as they carry no anomalous point.
The pattern of 1D is reproduced.
On the out-of-support \texttt{ood} cluster the distance/sparsity detectors saturate ($k$-NN distance, Isolation Forest and SOM quantization error all reach $\mathrm{TPR}=1.00$), since those points sit in an empty region of the space.
The exception is ECOD ($0.00$): the \texttt{ood} cluster lies at the center of the mixture, so although it occupies a two-dimensional void it is \emph{not} in the tail of either one-dimensional marginal, and a per-dimension tail detector is therefore blind to it---an illustration of the limits of coordinate-wise scoring.
On the in-distribution \texttt{bulk} anomaly the unsupervised detectors again tag essentially no points ($\mathrm{TPR}\le0.06$), whereas \textsc{SparKer} is once more the only method with appreciable point-level sensitivity there ($\mathrm{TPR}\approx0.20$).
\textsc{SparKer} also retains moderate power on \texttt{ood} ($\approx0.67$), though here the specialized distance detectors are stronger, as expected for a genuinely out-of-support signal.
In both dimensionalities, the realized per-point false-positive rate is controlled at $\approx5\%$ for every method.
\subsection{SOM-based detectors: capacity allocation and sensitivity across regimes}\label{subapp:som-vis}
To illustrate at a glance the conceptual difference between \textsc{SparKer} and its closest established relative discussed in Section~1.5 of the main text, Figure~\ref{fig:som_vs_sparker_combined} contrasts, on the 1D exponential benchmark, the prototype locations learned by a SOM trained on the reference with the kernel locations learned by \textsc{SparKer}, for capacities $M=5$ and $M=50$.
The SOM allocates all of its prototypes to the high-density bulk at small $x$, for both values of $M$, with essentially no coverage of the tail: this is the vector-quantization objective faithfully representing the reference distribution.
\textsc{SparKer} shows the complementary behavior: on background-only data no kernel is drawn to any particular location, whereas for each injected signal a subset of kernels acquires large amplitudes $|a_i|$ and moves into the injected-signal region, placing model capacity precisely where the data depart from the reference, while the remaining kernels stay in the bulk to represent the region of agreement (where the log-density ratio is $\approx 0$).
The contrast is sharpest in the tail and extreme-tail cases, where the SOM has no prototypes at all yet \textsc{SparKer} positions active kernels directly on the anomaly.
Increasing the capacity from $M=5$ to $M=50$ does not change this picture---it only refines it---confirming that the difference is one of objective, not of expressive power: the same self-organizing machinery is aimed at representing a distribution in one case, and at testing for local departures from it in the other.
\paragraph{Sensitivity of the SOM-based tests across regimes.}
Table~\ref{tab:som} complements this qualitative picture with the detection power of the two SOM-based detectors of Section~\ref{app:baselines} on the 1D, 2D and 5D benchmarks, across four choices of the number of prototypes $M$ and for the $M$-aggregated variant entering the comparisons of Section~\ref{subapp:regimes}.
The two detectors behave as their construction predicts, and in opposite regimes.
The quantization-error test is sensitive to anomalies injected into the tails: because the SOM allocates its prototypes to the high-density bulk, points in the sparse tails are reconstructed poorly and incur a large quantization error, whereas an over-density sitting in the well-covered bulk is reconstructed as accurately as ordinary reference points and leaves the score essentially unchanged (power at or below $0.16$ on every \texttt{bulk} signal, in any dimension).
The binned $\chi^2$/Poisson test shows the complementary pattern: it is effective on the \texttt{bulk} signal in 1D, since the SOM's Voronoi cells are fine-grained where the data are dense and thus resolve a local excess there, but it loses power in the tails, where the same construction leaves only a few coarse cells and the anomalous excess is diluted within them.
The dependence on the number of prototypes is also instructive: the quantization-error test improves with larger $M$, as additional prototypes extend coverage into the tails, whereas the binned test is most powerful with smaller $M$ (its power on the 1D \texttt{bulk} falls from $0.95$ at $M{=}5$ to $0.30$ at $M{=}200$) and does not recover the localized tail signals at any of the values of $M$ we tried.
Notably, in 1D the binned $\chi^2$ at small $M$ is competitive with the best collective tests on the \texttt{bulk} excess ($0.95$, comparable to \textsc{SparKer}); this sensitivity, however, does not survive the move to higher dimensions, where the Voronoi partition becomes too coarse relative to the data density (power $\le 0.16$ on the 2D \texttt{bulk} and $\le 0.17$ on the 5D Z'200), while \textsc{SparKer} retains its power there (Tables~\ref{tab:baselines2d} and~\ref{tab:baselines5d}).
Neither variant---at no choice of $M$, nor in the aggregation---is simultaneously sensitive across regimes: each SOM test succeeds only in the corner of the anomaly space its objective favors, the empirical counterpart of the conceptual difference between representing a distribution and testing for local departures from it.
\begin{table}[h]
\centering
\scalebox{0.99}{
\begin{tabular}{l cccc}
\toprule
\multicolumn{5}{c}{\textbf{1D benchmarks}}\\
Method & \texttt{bulk} & \texttt{tail} & \texttt{extreme tail} & \texttt{tail excess} \\
\midrule
SOM quant.\ error ($M{=}5$)     & $0.05\,[0.03,0.08]$ & $0.36\,[0.31,0.41]$ & $0.32\,[0.27,0.37]$ & $0.66\,[0.61,0.71]$ \\
SOM quant.\ error ($M{=}20$)    & $0.08\,[0.05,0.12]$ & $0.51\,[0.46,0.56]$ & $0.58\,[0.53,0.63]$ & $0.69\,[0.64,0.74]$ \\
SOM quant.\ error ($M{=}100$)   & $0.06\,[0.04,0.09]$ & $0.51\,[0.46,0.56]$ & $0.59\,[0.54,0.64]$ & $0.55\,[0.50,0.60]$ \\
SOM quant.\ error ($M{=}200$)   & $0.08\,[0.05,0.12]$ & $0.58\,[0.53,0.63]$ & $0.56\,[0.51,0.61]$ & $0.52\,[0.47,0.57]$ \\
SOM quant.\ error ($M$-aggregated)  & $0.10\,[0.07,0.14]$ & $0.64\,[0.59,0.69]$ & $0.63\,[0.58,0.68]$ & $0.75\,[0.70,0.80]$ \\
\midrule
SOM binned $\chi^2$ ($M{=}5$)   & $0.95\,[0.92,0.97]$ & $0.10\,[0.07,0.14]$ & $0.08\,[0.05,0.12]$ & $0.76\,[0.71,0.80]$ \\
SOM binned $\chi^2$ ($M{=}20$)  & $0.88\,[0.84,0.91]$ & $0.04\,[0.02,0.07]$ & $0.10\,[0.07,0.14]$ & $0.52\,[0.47,0.57]$ \\
SOM binned $\chi^2$ ($M{=}100$) & $0.48\,[0.43,0.53]$ & $0.05\,[0.03,0.08]$ & $0.03\,[0.01,0.06]$ & $0.18\,[0.14,0.23]$ \\
SOM binned $\chi^2$ ($M{=}200$) & $0.30\,[0.25,0.35]$ & $0.03\,[0.01,0.06]$ & $0.01\,[0.00,0.03]$ & $0.11\,[0.08,0.15]$ \\
SOM binned $\chi^2$ ($M$-aggregated) & $0.95\,[0.92,0.97]$ & $0.07\,[0.04,0.11]$ & $0.05\,[0.03,0.08]$ & $0.65\,[0.60,0.70]$ \\
\midrule
\midrule
\multicolumn{5}{c}{\textbf{2D benchmarks}}\\
Method & \texttt{bulk} & \texttt{ood} & \texttt{smearing} & \texttt{squeezing} \\
\midrule
SOM quant.\ error ($M{=}5$)                 & $0.03\,[0.01,0.06]$ & $0.04\,[0.02,0.07]$ & $0.04\,[0.02,0.07]$ & $0.03\,[0.01,0.06]$ \\
SOM quant.\ error ($M{=}20$)                & $0.10\,[0.07,0.14]$ & $0.17\,[0.13,0.22]$ & $0.55\,[0.50,0.60]$ & $0.00\,[0.00,0.02]$ \\
SOM quant.\ error ($M{=}100$)               & $0.11\,[0.08,0.15]$ & $0.20\,[0.16,0.25]$ & $0.47\,[0.42,0.52]$ & $0.00\,[0.00,0.02]$ \\
SOM quant.\ error ($M{=}200$)               & $0.09\,[0.06,0.13]$ & $0.12\,[0.09,0.16]$ & $0.38\,[0.33,0.43]$ & $0.00\,[0.00,0.02]$ \\
SOM quant.\ error ($M$-aggregated)          & $0.10\,[0.07,0.14]$ & $0.13\,[0.10,0.17]$ & $0.48\,[0.43,0.53]$ & $0.01\,[0.00,0.03]$ \\
\midrule
SOM binned $\chi^2$ ($M{=}5$)               & $0.09\,[0.06,0.13]$ & $0.03\,[0.01,0.06]$ & $0.12\,[0.09,0.16]$ & $0.16\,[0.12,0.21]$ \\
SOM binned $\chi^2$ ($M{=}20$)              & $0.10\,[0.07,0.14]$ & $0.14\,[0.10,0.18]$ & $0.38\,[0.33,0.43]$ & $0.40\,[0.35,0.45]$ \\
SOM binned $\chi^2$ ($M{=}100$)             & $0.10\,[0.07,0.14]$ & $0.20\,[0.16,0.25]$ & $0.15\,[0.11,0.19]$ & $0.13\,[0.10,0.17]$ \\
SOM binned $\chi^2$ ($M{=}200$)             & $0.12\,[0.09,0.16]$ & $0.52\,[0.47,0.57]$ & $0.15\,[0.11,0.19]$ & $0.11\,[0.08,0.15]$ \\
SOM binned $\chi^2$ ($M$-aggregated)        & $0.16\,[0.12,0.21]$ & $0.39\,[0.34,0.44]$ & $0.30\,[0.25,0.35]$ & $0.28\,[0.23,0.33]$ \\
\midrule
\midrule
\multicolumn{5}{c}{\textbf{5D benchmarks}}\\
Method & \texttt{Z'200 (bulk)} & \texttt{Z'300 (tail)} & \texttt{Z'600 (extreme tail)} & \texttt{EFT (non-resonant)} \\
\midrule
SOM quant.\ error ($M{=}5$)                  & $0.11\,[0.08,0.15]$ & $0.16\,[0.12,0.21]$ & $0.16\,[0.12,0.21]$ & $0.22\,[0.18,0.27]$ \\
SOM quant.\ error ($M{=}20$)                 & $0.13\,[0.10,0.17]$ & $0.23\,[0.19,0.28]$ & $0.11\,[0.08,0.15]$ & $0.31\,[0.26,0.36]$ \\
SOM quant.\ error ($M{=}100$)                & $0.10\,[0.07,0.14]$ & $0.19\,[0.15,0.24]$ & $0.09\,[0.06,0.13]$ & $0.31\,[0.26,0.36]$ \\
SOM quant.\ error ($M{=}200$)                & $0.14\,[0.10,0.18]$ & $0.14\,[0.10,0.18]$ & $0.15\,[0.11,0.19]$ & $0.33\,[0.28,0.39]$ \\
SOM quant.\ error ($M$-aggregated)           & $0.16\,[0.12,0.21]$ & $0.21\,[0.17,0.26]$ & $0.08\,[0.05,0.12]$ & $0.36\,[0.31,0.42]$ \\
\midrule
SOM binned $\chi^2$ ($M{=}5$)                & $0.05\,[0.03,0.08]$ & $0.06\,[0.04,0.09]$ & $0.06\,[0.04,0.09]$ & $0.11\,[0.08,0.15]$ \\
SOM binned $\chi^2$ ($M{=}20$)               & $0.04\,[0.02,0.07]$ & $0.02\,[0.01,0.05]$ & $0.06\,[0.04,0.09]$ & $0.05\,[0.03,0.08]$ \\
SOM binned $\chi^2$ ($M{=}100$)              & $0.11\,[0.08,0.15]$ & $0.05\,[0.03,0.08]$ & $0.12\,[0.09,0.16]$ & $0.08\,[0.05,0.12]$ \\
SOM binned $\chi^2$ ($M{=}200$)              & $0.17\,[0.13,0.21]$ & $0.08\,[0.06,0.12]$ & $0.10\,[0.07,0.14]$ & $0.07\,[0.04,0.11]$ \\
SOM binned $\chi^2$ ($M$-aggregated) & $0.09\,[0.06,0.13]$ & $0.02\,[0.01,0.05]$ & $0.05\,[0.03,0.08]$ & $0.04\,[0.02,0.07]$ \\

\bottomrule
\end{tabular}}
\caption{\textbf{SOM-based methods: summary of the results.} Detection power at $5\%$ false-positive rate on the 1D, 2D, and 5D benchmarks, for each signal type, for the two SOM-based detectors across four choices of the number of prototypes $M$ used to train the SOM. The $M$-aggregated rows combine the per-$M$ $p$-values via the same combination rule used for the kernel bandwidths (Eq.~\ref{eq:aggregation}), paying the associated trial factor explicitly. Values are power with $68\%$ Clopper--Pearson intervals, estimated over $n=100$.}
\label{tab:som}
\end{table}

\begin{figure}[h]
\centering
\includegraphics[width=0.87\linewidth]{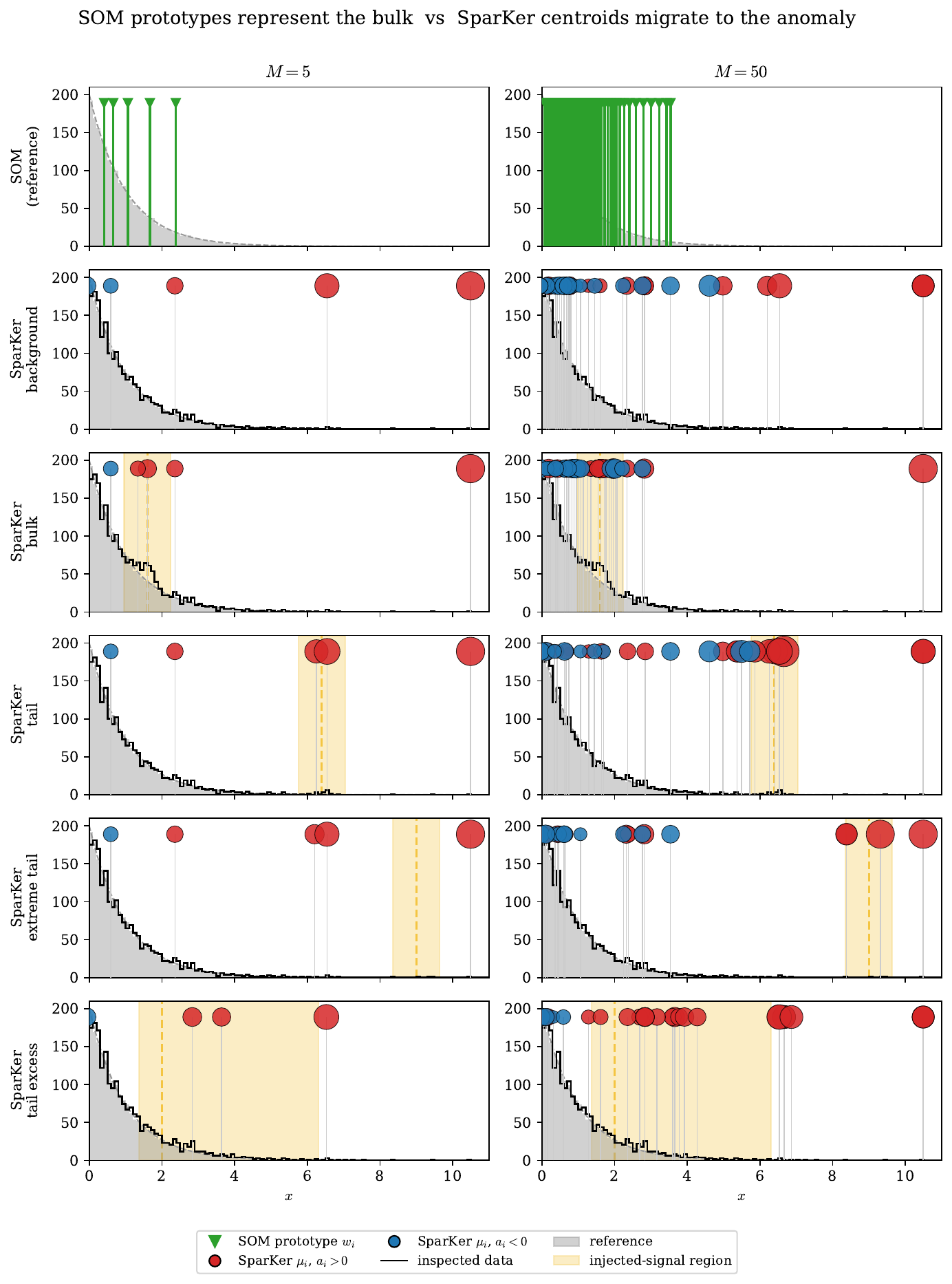}
\caption{\textbf{SOM prototypes represent the bulk, whereas SparKer kernels
populate the anomalous region.} Learned locations on the 1D exponential benchmark
for capacities $M=5$ (left) and $M=50$ (right), from a single representative seed.
\emph{Top row:} prototypes $w_i$ (green) of a SOM trained on the reference (gray density). \emph{Other rows:} SparKer centroids $\mu_i$ for background-only data and the four injected signals (Table~S1); the inspected data are the black histogram, the reference gray, and the injected-signal region shaded. Circle area is proportional to the amplitude $|a_i|$ and color encodes its sign ($a_i>0$ red, $a_i<0$ blue); marker height is not meaningful. Regardless of $M$, the SOM places all prototypes in the high-density bulk with no coverage of the tails, while \textsc{SparKer} puts high-amplitude kernels inside the signal region---most strikingly in
the (extreme) tail, where the SOM has none---and leaves the rest in the bulk: the
same self-organizing machinery aimed at representing a distribution versus testing
for local departures from it.}
\label{fig:som_vs_sparker_combined}
\end{figure}
\clearpage
\section{Ablations}\label{app:abl}
\subsection{The role of the NP loss}
We compare the gradients of three loss functions commonly considered for likelihood-ratio or two-sample learning: the Neyman–Pearson loss (NP), binary cross-entropy loss (BCE), and mean-squared error loss (MSE). Writing $f(x)$ for the model logit and $\sigma(f) = 1 / (1 + e^{-f})$ for its logistic transformation, the losses and their gradients with respect to the model parameters are listed in Table~\ref{tab:losses}.

\begin{table}[h]
\centering
\begin{tabular}{p{1cm}p{12cm}}
\toprule
\multirow{2}{*}{\textbf{NP}} 
& $\displaystyle \mathcal{L}_{\mathrm{NP}}[f] = 
w_{\reference}\sum_{x \in \reference} \big(e^{f(x)} - 1\big)
- \sum_{x \in \mathcal{D}} f(x)$ \\[8pt]
& $\displaystyle \partial_\theta \mathcal{L}_{\mathrm{NP}}[f] =
w_{\reference}\sum_{x \in \reference} e^{f(x)}\,\partial_\theta f(x)
- \sum_{x \in \data} \partial_\theta f(x)$ \\
\midrule

\multirow{2}{*}{\textbf{BCE}} 
& $\displaystyle \mathcal{L}_{\mathrm{BCE}}[f] =
-w_{\reference}\sum_{x \in \reference} \log\!\big[1-\sigma(f(x))\big]
- \sum_{x \in \data} \log\!\big[\sigma(f(x))\big]$ \\[8pt]
& $\displaystyle \partial_\theta \mathcal{L}_{\mathrm{BCE}}[f] =
w_{\reference}\sum_{x \in \reference} \sigma(f(x))\,\partial_\theta f(x)
- \sum_{x \in \data} [1-\sigma(f(x))]\,\partial_\theta f(x)$ \\
\midrule

\multirow{2}{*}{\textbf{MSE}} 
& $\displaystyle \mathcal{L}_{\mathrm{MSE}}[f] =
w_{\reference}\sum_{x \in \reference} [\sigma(f(x))]^2
+ \sum_{x \in \data} [1-\sigma(f(x))]^2$ \\[8pt]
& $\displaystyle \partial_\theta \mathcal{L}_{\mathrm{MSE}}[f] =
w_{\reference}\sum_{x \in \reference} 2\,\sigma(f(x))^2[1-\sigma(f(x))]\,\partial_\theta f(x)
- \sum_{x \in \data} 2\,\sigma(f(x))[1-\sigma(f(x))]^2\,\partial_\theta f(x)$ \\
\bottomrule
\end{tabular}
\caption{\textbf{Comparison of loss functions and their gradients.} We study the Neyman-Pearson (NP) loss, the Binary Cross Entropy (BCE) loss and the Mean Squared Error (MSE) loss. Each loss defines a different optimization landscape and gradient behavior.}
\label{tab:losses}
\end{table}

The different loss definitions imply distinct gradient behaviors. 
The BCE and MSE losses depend on the function $f$ through its logistic transformation. Thus, by the chain rule, their gradients depend on the sigmoid as well.
In the MSE case, all the terms contributing to the gradient depend on the product $\sigma(f)\cdot[1 - \sigma(f)]$, causing gradient decay as $\sigma(f)$ approaches either 0 or 1, independently of the point $x$ category. 
As for the BCE loss, contributions coming from the reference sample $\reference$ directly depend on the logistic factor $\sigma(f)$, whereas contributions coming from the data sample $\data$ depend on the term $[1 - \sigma(f)]$. This structure causes partial gradient saturation when $\sigma(f)$ approaches $1$ for points in $\data$ or $0$ for points in $\reference$. The result is a slower learning dynamic for confidently classified samples, though typically less severely than for MSE.

The NP loss does not depend on $f$ through the logistic transformation, 
thus avoiding vanishing gradients typical of sigmoid-based losses. However, the loss is unbounded from below and requires implicit or explicit regularization, such as the amplitude clipping used in this work, to prevent divergence.

Due to its non-vanishing gradients, the NP loss seems the most suitable candidate in settings that require accurate learning of the density-ratio. 
We validate these theoretical observations in a one-dimensional controlled experiment. We consider the scenario of a peak in the extreme tail of an exponential distribution, described in Section~\ref{subapp:1d}.
Table~\ref{tab:loss} reports the detection power at 5\% type I error for \textsc{SparKer} optimized using the MSE, BCE or NP loss. As expected, the MSE loss is the least powerful choice and NP the most effective one.
\begin{table}[h]
    \centering
    \begin{tabular}{lccc}
    \toprule
    Loss function   & MSE & BCE & NP \\
    Power@5\%       & 0.61 & 0.68 & \textbf{0.86}\\
    \bottomrule
    \end{tabular}
    \caption{Detection power at 5\% false positive rate for various choices of the loss function: mean squared error (MSE), binary cross-entropy (BCE), and Neyman-Pearson (NP). The NP loss reaches the highest performance.}
    \label{tab:loss}
\end{table}

\subsection{The role of ${\rm SoftMax}$ activation}
In this section we further elaborate on the impact of the ${\rm SoftMax}$ activation layer in the design of \textsc{SparKer}. Figure~\ref{fig:2M} illustrates the impact of ${\rm SoftMax}$ for a one-dimensional model composed of two kernels, located at $\mu_1$ and $\mu_2$ respectively. The left panel shows the model design with (dashed orange line) and without (solid blue line) ${\rm SoftMax}$ activation layer. The middle and right panels display the partial derivatives of the model with respect to the kernels' locations $\mu_1$ and $\mu_2$, highlighting the implications of ${\rm SoftMax}$ on the learning dynamics of the kernels' locations. The presence of ${\rm SoftMax}$ activation causes a change of sign in the derivative with respect to one kernel for regions of the input space where the other kernel ``wins'' the competition, altering the direction of the kernel's motion during the learning dynamics. 
\begin{figure}[h]
\centering
\includegraphics[width=0.32\linewidth]{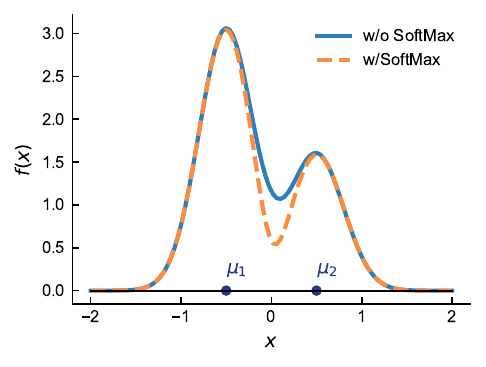}
\includegraphics[width=0.32\linewidth]{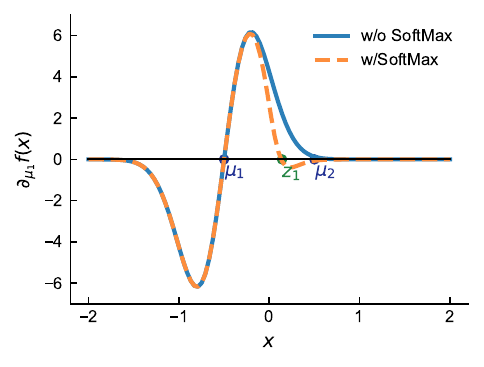}
\includegraphics[width=0.32\linewidth]{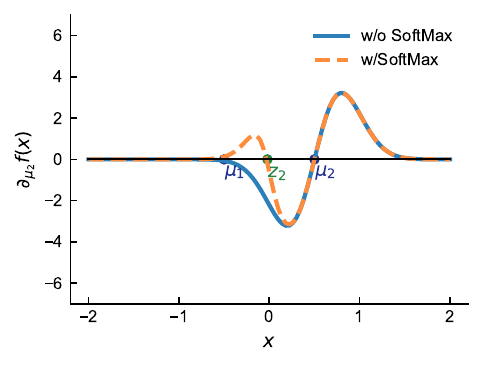}
\caption{\textbf{The effect of ${\rm SoftMax}$ on \textsc{SparKer}.} We compare \textsc{SparKer} in the presence (dashed orange line) and in the absence (solid blue line) of ${\rm SoftMax}$ activation. The left panel shows an example of a 2-kernel model in 1D; the middle and right panels show the partial derivatives of the model with respect to the kernels' locations $\mu_1$ and $\mu_2$. The kernels' interaction induced by ${\rm SoftMax}$ manifests as a change of sign in the derivative governing the direction of the kernels' motion during the learning dynamics.}\label{fig:2M}
\end{figure}
The implication of this phenomenon is an increased model sensitivity to rare anomalies. Table~\ref{tab:softmax} shows numerical results obtained comparing \textsc{SparKer} with and without ${\rm SoftMax}$ activation for the one-dimensional problem of detecting rare anomalies on the extreme tails of the nominal distribution, introduced in Section~\ref{subapp:1d}. ${\rm SoftMax}$ leads to a higher power of the Neyman–Pearson test statistic. This improvement is accompanied by a more uniform spatial allocation of kernels, as evidenced by the Moran spacing~\cite{cheng1989goodness} in Figure~\ref{fig:moran}. Together, these results indicate that including ${\rm SoftMax}$ activation in the model design promotes a more efficient and homogeneous exploration of the input space, enhancing chances of discovery.
\begin{table}[h!]
    \centering
    \begin{tabular}{lcc}
    \toprule
    Model       & w/o ${\rm SoftMax}$ & w/ ${\rm SoftMax}$\\
    Power@5\%   & 0.72 & \textbf{0.86}\\
    \bottomrule
    \end{tabular}
    \caption{Detection power at 5\% false positive rate for \textsc{SparKer} model with and without ${\rm SoftMax}$ activation layer. ${\rm SoftMax}$ activation leads to higher performance.}
    \label{tab:softmax}
\end{table}

\begin{figure}[h]
    \centering
    \includegraphics[width=0.5\linewidth]{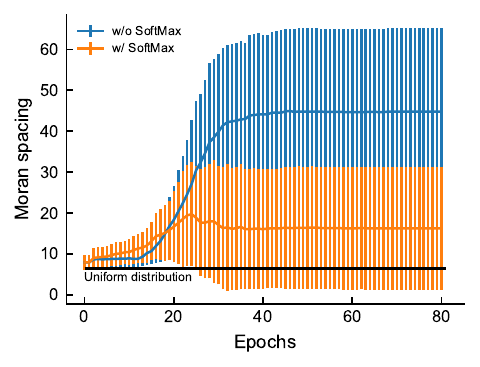}
    \caption{\textbf{${\rm SoftMax}$ improves specialization.} We measure the distance of the kernels distribution from a uniform distribution by computing the Moran spacing statistic as a function of the training epoch. The  Moran evolution for \textsc{SparKer} in the absence of ${\rm SoftMax}$ activation is reported in blue, whereas the trend in the presence of ${\rm SoftMax}$ is reported in orange. The solid lines are obtained averaging over 100 anomaly-free realizations; error bars represent the standard deviation. We focus on the 1-dimensional problem of detecting a peak in the extreme tail, introduced in Section~\ref{subsec:toys}. The average value of the test statistic under the uniform assumption is highlighted in black. Generally, the higher the value of Moran spacing, the farther from the uniformity assumption. On average including ${\rm SoftMax}$ in the model design leads to a more uniform distribution of kernels over the space.}
    \label{fig:moran}
\end{figure}

\subsection{The role of the ensemble size $M$}\label{subsec:role_m}
The number of kernels $M$ is the model's capacity budget. It bounds the number of distinct anomalous regions the ensemble can resolve simultaneously.
Table~\ref{tab:M_ablation} collects the detection power of \textsc{SparKer} across the controlled benchmarks for different ensemble sizes.
Increasing $M$ does not help and typically costs power: $M{=}5$ is comparable to $M{=}50$ in 1D, $M{=}5$ outperforms $M{=}20$ on three out of four benchmarks in 2D, and in 5D the two-kernel ensemble outperforms the five- and fifty-kernel ones on every signal.
This behavior is the empirical counterpart of the capacity argument of Section~1.3.1 of the main text: each additional kernel offers more ways to accommodate fluctuations of anomaly-free data, inflating the null distribution of the test statistic and raising the bar a genuine signal must clear.
The lower bound on $M$ is set by the anomaly itself: an ensemble with fewer kernels than anomalous modes cannot resolve them separately, degrading both power and the geometric interpretation of Section~4 of the main text.
The resulting guidance is signal-agnostic in the same sense as the rest of the pipeline: choose the smallest $M$ compatible with the expected multiplicity of anomalous regions---a handful when no prior expectation exists.
Following the 1D and 2D ablations, we adopted small ensembles ($M{=}2$, $M{=}5$) for the 5D benchmark and fixed the default to $M{=}5$ for all the applications of Section~3 of the main text, with no per-application tuning.

\begin{table}[h!]
\centering
\begin{tabular}{ll cccc}
\toprule
Benchmark & Signal & $M{=}2$ & $M{=}5$ & $M{=}20$ & $M{=}50$ \\
\midrule
1D exponential & \texttt{bulk}          & \textbf{0.94 [0.91, 0.96]}   & 0.92 [0.88, 0.94] & 0.83\,[0.79,0.86]   & 0.85 [0.81, 0.88] \\
               & \texttt{tail}          & 0.67 [0.62, 0.71]  & \textbf{0.80 [0.75, 0.83]} & 0.74 [0.69,0.78]   & 0.76 [0.71, 0.80] \\
               & \texttt{extreme tail} & 0.62 [0.57, 0.66]  & \textbf{0.86 [0.82, 0.89]} & 0.82 [0.78,0.85]   & \textbf{0.86 [0.82, 0.89]} \\
               & \texttt{tail excess}  & \textbf{0.98 [0.95, 0.99]}   & 0.94 [0.91, 0.96] & 0.92 [0.88,0.94]   & 0.95 [0.92, 0.96] \\
\midrule
2D mixture     & \texttt{bulk}          & 0.25 [0.20, 0.29]  & 0.27 [0.22, 0.31] & \textbf{0.31 [0.18, 0.40]} & -- \\
               & \texttt{ood}           & \textbf{0.58 [0.53, 0.63]}  & 0.50 [0.44, 0.54] & 0.22 [0.18, 0.26] & -- \\
               & \texttt{smearing}      & 0.32 [0.27, 0.36]  & \textbf{0.44 [0.38, 0.48]} & 0.21 [0.17, 0.25] & -- \\
               & \texttt{squeezing}     & 0.21 [0.17, 0.25]  & \textbf{0.37 [0.32, 0.41]} & 0.19 [0.15, 0.22] & -- \\
\midrule
5D di-muon     & \texttt{Z'200 (bulk)}                  & \textbf{0.38 [0.33, 0.42]} & 0.33 [0.28, 0.37] & -- & 0.12 [0.09,0.15] \\
               & \texttt{Z'300 (tail)}                  & \textbf{0.70 [0.66, 0.74]} & 0.47 [0.41, 0.51] & -- & 0.21 [0.16,0.26] \\
               & \texttt{Z'600 (extreme tail)}                  & \textbf{0.70 [0.66, 0.74]} & 0.32 [0.27, 0.36] & -- & 0.13 [0.10,0.16] \\
               & \texttt{EFT (non-resonant)}                    & \textbf{0.71 [0.68, 0.75]} & 0.35 [0.30, 0.39] & -- & 0.23 [0.19,0.27] \\
\bottomrule
\end{tabular}
\caption{\textbf{Ablation of the ensemble size $M$.} Detection power at $5\%$ false-positive rate of \textsc{SparKer} across the controlled benchmarks, for the ensemble sizes run in this work (dashes: configuration not run). Values are power with $68\%$ Clopper--Pearson intervals over $n=100$ signal realizations, calibrated on $300$ anomaly-free realizations, sharing the signal definitions and calibration pipeline of Section~3. We highlight in bold the best power in each row.  Power does not increase with capacity: beyond the anomaly multiplicity, additional kernels inflate the null distribution and cost sensitivity (Section~1.3.1 of the main text).}
\label{tab:M_ablation}
\end{table}

When no prior expectation on the anomaly multiplicity is available, the choice of $M$ can be removed from the procedure altogether: one can test at several ensemble sizes and aggregate the resulting $p$-values with the same combination rule used for the kernel width (Eq.~\ref{eq:aggregation}), at the cost of one calibration per value of $M$.
The same machinery keeps the test valid when $M$ is enlarged after a first attempt: selecting the more significant of two ensemble sizes would incur an unaccounted post-selection bias, whereas combining the two tests into a single aggregated statistic, calibrated on the same anomaly-free experiments, pays the trial factor explicitly and preserves the false-positive rate.
\section{Applications details}\label{app:appl} 
\subsection{Gravitational waves detection} 
Gravitational waves are ripples in spacetime generated by massive accelerating objects, such as merging black holes or neutron stars, and propagate at the speed of light. Detectors like LIGO~\cite{LIGOScientific:2014pky} directly measure these signals on Earth, revealing the properties of their sources and testing the nature of gravity. Many predicted sources remain unobserved, and the diversity of possible signals—often lacking precise theoretical models—motivates anomaly detection methods. In this work, we analyze data from the beginning of LIGO’s O3 run~\cite{O3a,O3b}, categorized into Gaussian noise, instrumental glitches, and six astrophysical or hypothetical classes of signal sources. Following the NSF HDR ML Challenge~\cite{campolongo2025building}, we include white-noise bursts (WNB) as anomalies. Our study uses the feature-extraction pipeline of~\cite{phlab} to provide inputs to the \textsc{SparKer} algorithm. A one-dimensional ResNet encoder set up as in~\cite{Marx:2024wjt} is fine-tuned using supervised contrastive learning on anomaly-free data to provide a four-dimensional compressed representation. The embedded data are publicly available on Zenodo~\cite{bright_thonney_2026_lhcjets}. For the presented experiments, we set the size of the reference sample at $20\,000$, and the size of the inspected data at $4\,000$. Signal injections range from 1 to 5\%.

\subsection{New particle jets discovery} 
In high energy physics, hadronic jets are sprays of particles originating from quarks or gluons produced in hadron collider experiments. Their substructure provides hints of the physics process underlying their generation, making them a valuable tool to search for new phenomena deviating from the Standard Model of particle physics.
We apply \textsc{SparKer} to simulations of particle physics jets produced in proton-proton collisions at the Large Hadron Collider (LHC), available in the \textsc{JetClass} dataset~\cite{qu_2022_6619768,Qu:2022mxj}. 
Jets originating from quarks and gluons (QCD), top quark decays ($t \to bqq^\prime$), and vector boson decays ($V \to qq^\prime$) are considered as background classes, representing the nominal generation model. Following the numerical experiments proposed in~\cite{phlab}, we consider as a signal jets produced by the Higgs boson decaying to bottom-antibottom quark pairs. The encoder, introduced in~\cite{phlab}, has the Particle Transformer (ParT) model as its backbone~\cite{Qu:2022mxj}, and an additional fully connected final layer to fine tune the embedding with supervised contrastive learning. The embedded data are publicly available on Zenodo~\cite{bright_thonney_2026_lhcjets}. For the experiments presented in the main paper, we set the size of the latent space at 128, the reference sample is $100\,000$ events, and the size of the inspected data Poisson fluctuates around $20\,000$. Signal events in this case are added to background events and range from 0.5 to 2\% of the background. Also in this case, the total number fluctuates according to a Poisson distribution around the chosen working point value\footnote{Poissonian fluctuations in the number of events in each class simulate real experimental conditions at colliders}.

\subsection{Hybrid butterfly discovery} Detecting hybrid species and assessing their statistical significance at the population level is a crucial task to recognize the emergence of new species in evolutionary biology. We apply \textsc{SparKer} to the NSF HDR ML Challenge of identifying new hybrid samples in collections of butterfly images~\cite{campolongo2025building}. The data are a subset of the Heliconius Collection (Cambridge Butterfly) consisting of butterfly images and corresponding labels identifying their nature (either hybrids or not)~\cite{gabriela_montejo_kovacevich_2020_4289223, patricio_a_salazar_2020_4288311, montejo_kovacevich_2019_2677821, jiggins_2019_2682458, montejo_kovacevich_2019_2684906, warren_2019_2552371, warren_2019_2553977,montejo_kovacevich_2019_2686762, jiggins_2019_2549524, jiggins_2019_2550097, joana_i_meier_2020_4153502, montejo_kovacevich_2019_3082688, montejo_kovacevich_2019_2813153, salazar_2018_1748277, montejo_kovacevich_2019_2702457, salazar_2019_2548678, pinheiro_de_castro_2022_5561246, montejo_kovacevich_2019_2707828, montejo_kovacevich_2019_2714333, gabriela_montejo_kovacevich_2020_4291095, montejo_kovacevich_2021_5526257, warren_2019_2553501, salazar_2019_2735056, mattila_2019_2555086}.
For the experiments presented in the main paper, we embed the data using the 1000-dimensional logits representation of a ResNet50 pretrained on ImageNet. Due to the limited amount of available data, the reference sample only contains $1\,000$ points, and the size of the inspected data $100$. Signal injections range from 5 to 20\% of the background. 

\subsection{CIFAR-10 vs. CIFAR-5m}
In the experiments involving generative model validation using CIFAR-10~\cite{krizhevsky2009learning} and CIFAR-5m~\cite{nakkiran2020deep} datasets, we consider the CIFAR-10 test set, composed of 10000 data points, as data $\data$ and a 100000 size sample from CIFAR-5m as reference $\reference$. Both samples are processed using ResNet50 pretrained on ImageNet to extract a 1000D representation to run \textsc{SparKer} on.
When performing bootstrap over anomaly-free realizations to calibrate the test we randomly sample both the reference and the data from a set of 300000 CIFAR-5m points.

\subsection{OOD in ImageNet}
For the experiments involving out-of-domain detection in ImageNet~\cite{ILSVRC15}, we consider data samples $\data$ of 1000 points and compare them with a reference sample $\reference$ of size 10000. The data are a mixture of ImageNet data points and novelties, taken from the Striped class of the Textures dataset~\cite{cimpoi2014describing}. The reference set is fully drawn from the ImageNet dataset and represents the closed-world assumption. Both samples are processed using ResNet50 pretrained on ImageNet to extract a 1000D representation to run \textsc{SparKer} on. 

\subsection{Intrusion detection in Dante's lines}
We prompted Meta-Llama-3-70B-Instruct~\cite{llama3modelcard} to randomly generate lines from Dante's Divine Comedy, repeatedly querying the model with the following prompt:
\begin{quote}
    ``Write five stanzas randomly extracted from the full Dante's Divina Commedia (italian). Each stanza has equal probability of being chosen.''
\end{quote}
At each query, the model builds its output one word (token) at a time, choosing from a probability distribution over candidate words rather than always picking the single most likely one. Three parameters control how this choice is made. \texttt{temperature} rescales the distribution before sampling: low values make the model nearly deterministic and repetitive, while high values flatten the distribution and make word choice closer to uniform, increasing variety at the cost of coherence. \texttt{top\_p} (nucleus sampling) restricts the choice at each step to the smallest set of candidate words whose cumulative probability reaches $p$, discarding the unlikely long tail; a low \texttt{top\_p} keeps sampling focused even when temperature is high. \texttt{presence\_penalty} discourages the model from repeating words it has already used earlier in the same generation, pushing it toward more diverse continuations. We set an intermediate \texttt{temperature=1.6} together with a low \texttt{top\_p=0.1} — so that the added randomness from the medium temperature is confined to a small set of plausible words — and \texttt{presence\_penalty=1.15} to further reduce repetition across the five generated stanzas. 
We nonetheless observe recurrent patterns in generation, so we filter the output keeping only unique outcomes. This makes the detection harder as it cannot rely on exact duplicates that are a clear mark of LLM generation.
We embedded the resulting lines into 384-dimensional representations using multilingual Sentence-BERT~\cite{reimers-2019-sentence-bert}. We tested samples comprising 1000 lines, and, to simulate a problem of intrusion detection, injected 2 to 20\% AI-generated lines into the dataset.
The dataset is publicly available on Zenodo~\cite{grosso_2026_dante}.

\subsection{Implementation details}
In all applications considered in this paper we run a 5-kernel \textsc{SparKer}. The model initialization follows the indications provided in Section~\ref{app:hyperp}. The number of training epochs, the L2 regularization coefficient and the amplitude clip are independently chosen for each application to ensure good statistical behavior of the test in anomaly-free conditions (i.e. calibration settings). We summarize the hyperparameter choice in Table~\ref{tab:hyperpars}.
\begin{table*}[h!]
\centering
\begin{tabular}{lccc}
\toprule
\textbf{Application} & $\boldsymbol{n_{\rm steps}}$ & \textbf{L2 coeff.} & \textbf{Amplitude clip} \\
\midrule
Gravitational waves detection&100k& 1& 100\\
New particle jets discovery& 100k & $10^{-3}$ &100\\
Hybrid butterfly discovery&100k & $10^{-4}$ & 200\\
intrusion detection in Dante's text& 100k & $10^{-3}$ & 500\\
OOD in ImageNet& 40k & 0 & 10k\\
CIFAR-10 vs. CIFAR-5m& 40k & 0 & 10k\\
\bottomrule
\end{tabular}
\caption{\textbf{\textsc{SparKer} training details.} Number of training epochs, L2 regularization coefficient and amplitude clipping for various applications considered in the paper.}
\label{tab:hyperpars}
\end{table*}
\section{Impact of data dimensionality}
We study the impact of the data dimensionality on the detection performance, focusing on the problem of particle jets discovery. As detailed in Section~\ref{app:appl}, the feature extraction employed in this application results from fine-tuning a transformer-based model tailored to analyze particle jets. We prepare different data representations scanning the number of dimensions of such representation while keeping all other aspects of the training the same. Figure~\ref{fig:dimscan_lhcjets} shows the detection power at 5\% false positive rate as a function of the number of dimensions of the latent space for \textsc{SparKer} and the same selection of baselines used in Figure~8 of the main paper.
Our experiments show stable performance for all tests. Falkon and \textsc{SparKer} obtain the highest performance. 
\begin{figure}[h]
    \centering
    \includegraphics[width=0.5\linewidth]{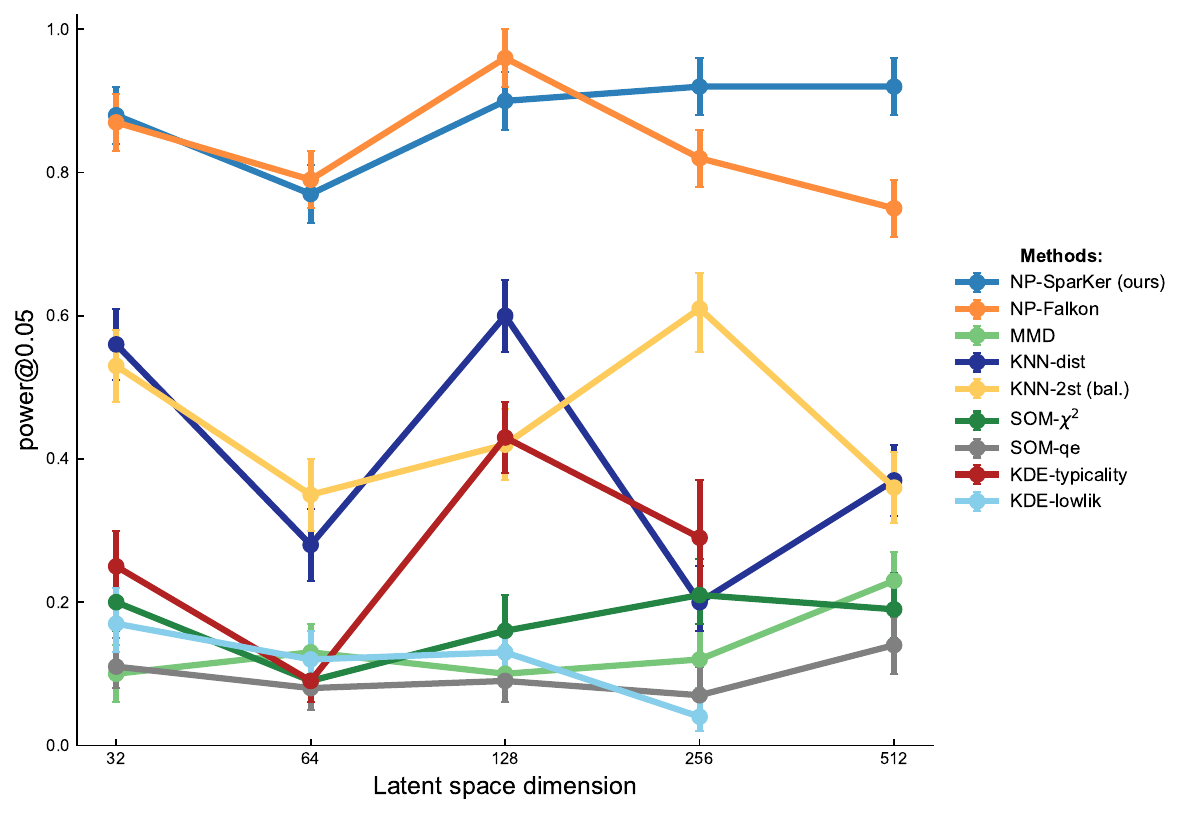}
    \caption{\textbf{Impact of latent space size in particle jets discovery.} Detection power at 5\% type I error as a function of latent space size. We focus on a 1\% anomaly fraction. Each color represents a different detection method. The power and type I error are estimated using 100 realizations in anomaly conditions and 300 realizations anomaly-free. Error bars represent the Clopper--Person confidence interval. All methods have stable performance across the dimension scan. The NP test, in both \textsc{SparKer} and Falkon implementations, consistently provides the highest power.}
    \label{fig:dimscan_lhcjets}
\end{figure}

\section{Additional plots on interpretability}\label{app:interp}
In this Section we report additional interpretability studies for the applications presented in the main paper.
\paragraph{Interpreting discovered species of butterflies.}
The left side of Figure~\ref{fig:app_interp} showcases the result of a hybrid butterfly detection task. Panel (a) shows the distribution of the anomaly score across the inspected data in black compared to the average distribution in the absence of signal in light blue (error bars represent the standard deviation over 100 anomaly-free realizations). \textsc{SparKer} identifies three kernel components ($f_1$, $f_2$, and $f_4$) that are strongly activated by the most anomalous samples, two of which are positively correlated with the anomaly score (see panel (b)). By projecting data onto the subspace spanned by these two components, distinct regions of anomalous butterflies emerge—corresponding to specific phenotypic variations. For instance, butterflies in the upper-left region display two bright white spots on the upper wings, whereas those in the lower-right exhibit only one and an overall softer coloration. 
\paragraph{Interpreting statistical deviations in CIFAR AI generation.}
The outcome of the CIFAR generative model validation task is summarized on the right side of Figure~\ref{fig:app_interp}. \textsc{SparKer} reveals a consistent excess in the right tail of the anomaly score distribution reported in panel (d), indicating that the reference model under-represents certain regions of data space. The analysis of component activations highlights that all kernels are relevant, with $f_1$, $f_2$, and $f_3$ positively correlated with the anomaly score. Visualizing the three-dimensional space spanned by these components highlights distinct and anomalous locations. The centroids of the components identifying these clusters are classified by ImageNet as ``rule'', ``Pekinese'', and ``crossword puzzl'' respectively. 

\begin{figure}[h]
    \centering
    \includegraphics[width=\linewidth]{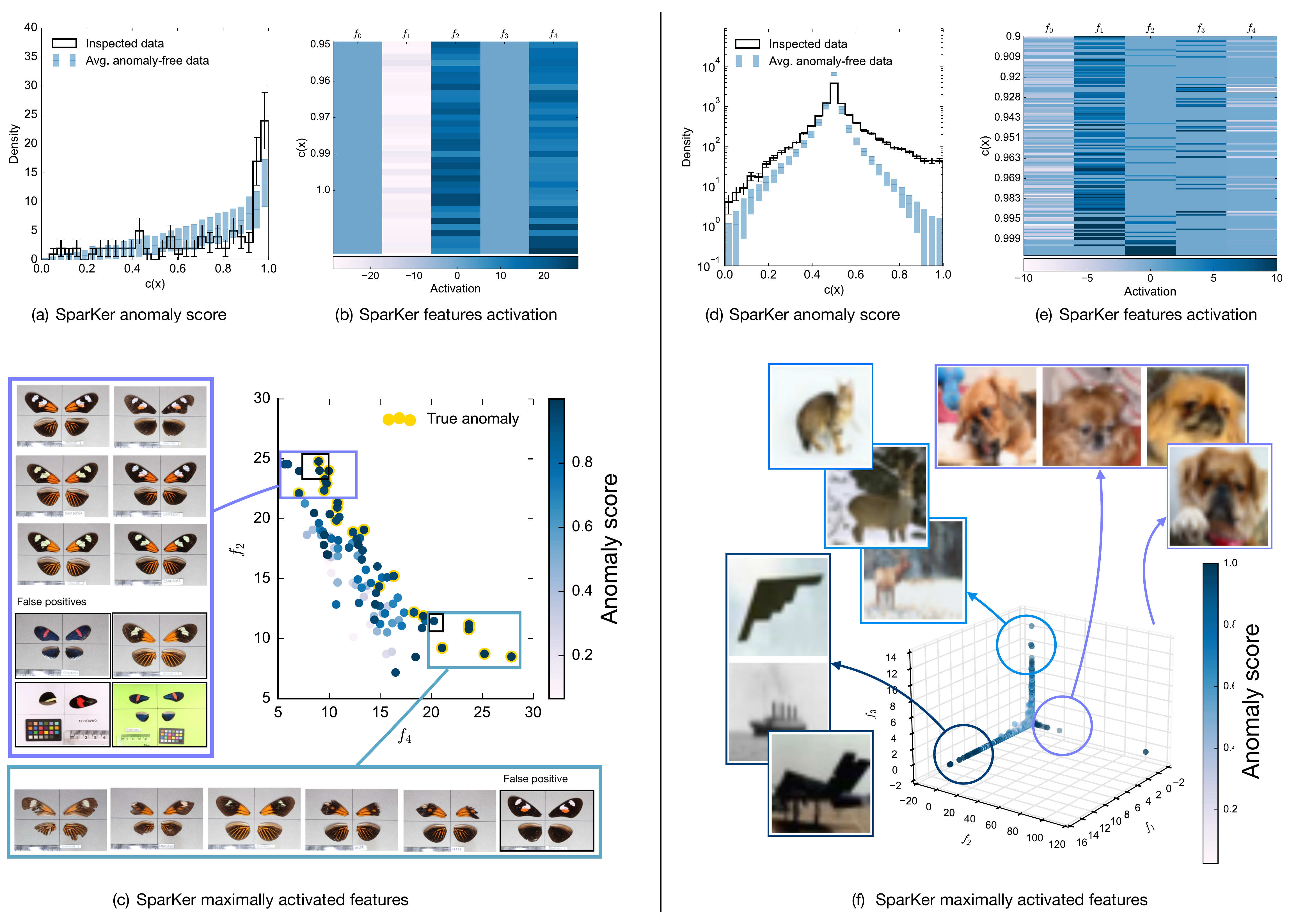}
    \caption{\textbf{Extracting geometric interpretation from \textsc{SparKer}}. The left side of the figure shows an example of data analysis in the hybrid butterfly detection problem. From panel (b) we identify three \textsc{SparKer} components ($f_1$, $f_2$ and $f_4$) highly activated by the most anomalous points in the dataset. In panel (c) we report the correlation plot between the positively activated components, color coded by the anomaly score, and we distinguish groups of correctly identified hybrids according to their location in this plane. The right side of the figure shows the results of our experiments on the CIFAR generative model validation problem. Panel (d) shows a consistent excess in the right tail of the anomaly score distribution, indicating that the generative model (used as the reference) fails to adequately represent a subspace of the data. More precisely, it under-represents it. Panel (e) informs us about the main regions of model misspecification; 
    three components ($f_1$, $f_2$ and $f_3$) highly activate to the anomalous points. Panel (f) shows a 3D scatter plot over these components, color coded by anomaly score, and identifies three orthogonal subgroups of anomalies, aligning with the ImageNet classes of ``rule'' ($f_1$), ``Pekinese'' ($f_2$), and ``crossword puzzle'' ($f_3$).}
    \label{fig:app_interp}
\end{figure}

\newpage
\bibliographystyle{ieeetr}
\bibliography{bibliography}
\end{document}